\documentclass[twoside,11pt, abbrvbib]{article}

\usepackage{tikz}
\usetikzlibrary{matrix}
\usepackage{jmlr2e}
\usepackage[utf8]{inputenc} 
\usepackage[T1]{fontenc}    
\usepackage{natbib}
\usepackage{multirow}
\usepackage{hyperref}
\hypersetup{colorlinks=false, allcolors=blue, breaklinks=true}
\usepackage{url}            
\usepackage{booktabs}       
\usepackage{amsfonts}       
\usepackage{nicefrac}       
\usepackage{microtype}      
\usepackage{lipsum}
\usepackage{amsmath}
\usepackage{amssymb}

\makeatletter
\def\th@plain{%
  \thm@notefont{}
  \itshape 
}
\def\th@definition{%
  \thm@notefont{}
  \normalfont 
}
\makeatother
\usepackage{dsfont}
\usepackage{graphicx}
\usepackage{xcolor,colortbl}
\definecolor{mygray}{gray}{0.6}
\usepackage{subfiles}
\usepackage{subcaption}
\usepackage{algorithm2e}
\usepackage[labelformat=simple,labelsep=quad,skip=3pt]{caption}
\usepackage{booktabs} 
\usepackage{amsfonts}
\usepackage{color}
\usepackage{bm}
\usepackage{cleveref}
\usepackage{mathrsfs}
\usepackage{thmtools}
\usepackage{thm-restate}
\usepackage{gordon}
\usepackage[normalem]{ulem}

\numberwithin{equation}{section}

\newcommand{\zero}{\mathbf{0}}
\newcommand{\RR}{\mathds{R}}

\newcommand{\Sp}{\mathrm{Sp}}

\newcommand{\blue}[1]{{#1}}
\newcommand{\D}{\mathsf{D}}
\newcommand{\Af}{\mathsf{A}}
\newcommand{\Ds}{\mathsf{D}^+}
\newcommand{\Di}{\mathsf{D}_{+}}
\newcommand{\Hi}{\mathsf{H}_{+}}
\newcommand{\He}{\mathsf{H}}

\newcommand{\K}[1][]{\mathsf{K}_{\mathsf{#1}}}
\newcommand{\h}{h}
\newcommand{\inner}[2]{\left\langle#1,#2\right\rangle}
\newcommand{\textcite}{\citet}
\def \glp {LRP}
\def \a {\alpha}
\def \b {\beta}
\def \R {\mathbb{R}}

\def \s {\sigma}
\def \d {\delta}
\def \bmat {\begin{matrix}}
\def \emat {\end{matrix}}
\def \l {\lambda}

\def \tr {\nonumber \\}
\def \rr {\mathcal{R}}

\def \n {\nabla}

\def \e {\epsilon}
\def \ve {\varepsilon}
\def \nn {\textsf{null}}

\newtheorem{thm}{Theorem}[section]
\newtheorem{cor}[thm]{Corollary}
\newtheorem{lem}[thm]{Lemma}
\newtheorem{prop}[thm]{Proposition}
\newtheorem{eg}[thm]{Example}
\newtheorem{defn}[thm]{Definition}

\newtheorem{assmp}[thm]{Assumption}

\newtheorem{rem}[thm]{Remark}
\def \R {\mathds{R}}
\def \d {\delta}
\def \e {\epsilon}
\def \g {\gamma}
\def \bmat {\begin{matrix}}
\def \emat {\end{matrix}}
\def \l {\lambda}
\def \be {\begin{eqnarray}}
\def \en {\end{eqnarray}}

\def \tr {\nonumber \\}
\def \B {\bf B}
\def \A {\bf A}

\def \I {\bf I}
\def \rr {\mathcal{R}}

\def \n {\partial}
\def \e {\epsilon}
\def \sp {\textrm{Sp}}
\def \zero {\mathbf{0}}

\def \C {\mathcal{C}}
\def \cC {\bf C}

\def \X {\mathcal{X}}
\def \Y {\mathcal{Y}}

\def \haa {{\bf H}_{\a_1, \a_2}}

\newcommand{\p}[1]{{\bf P}_{#1}^\perp}
\newcommand{\of}{\bar{f}}

\newcommand{\oq}{\bar{q}}

\usepackage{accents}
\usepackage{xspace}
\usepackage{amsmath,graphicx}
\usepackage{enumitem}
\usepackage{floatpag}
\DeclareRobustCommand{\mf}{\text{\reflectbox{$f$}}}

\newcommand{\ubar}[1]{\underaccent{\bar}{#1}}
\newcommand{\uf}{\ubar{f}}
\newcommand{\Bc}{\mathds{B}}
\newcommand{\tcs}[1][1]{\textcircled{\small #1}\xspace}

\makeatletter
\renewcommand*{\@opargbegintheorem}[3]{\trivlist
      \item[\hskip \labelsep{\bfseries #1\ #2}] \textbf{(#3)}\ \itshape}
\makeatother

\title{Optimality and Stability in Non-Convex Smooth Games}

\author{%
\name Guojun Zhang \email guojun.zhang@uwaterloo.ca
\AND
\name Pascal Poupart \email ppoupart@uwaterloo.ca
\AND
\name Yaoliang Yu \email yaoliang.yu@uwaterloo.ca\\
\addr School of Computer Science\\
University of Waterloo\\
Vector Institute
}

\editor{Simon Lacoste-Julien}

\usepackage{lastpage}
\jmlrheading{23}{2022}{1-\pageref{LastPage}}{8/20; Revised
5/21}{1/22}{20-918}{Guojun Zhang, Pascal Poupart and Yaoliang Yu}
\ShortHeadings{Optimality and Stability in Non-Convex Smooth Games}{Zhang, Poupart, and Yu}
\firstpageno{1}

\begin{document}
\maketitle

\begin{abstract}%
\blue{
Convergence to a saddle point for convex-concave functions has been studied for decades, while recent years has seen a surge of interest in \emph{non-convex} (zero-sum) smooth games, motivated by their recent wide applications. It remains an intriguing research challenge how local optimal points are defined and which algorithm can converge to such points. An interesting concept is known as the local minimax point \citep{jin2019minmax}, which strongly correlates with the widely-known gradient descent ascent algorithm. 
This paper aims to provide a comprehensive analysis of local minimax points, such as their relation with other solution concepts and their optimality conditions. We find that local saddle points can be regarded as a special type of local minimax points, called \emph{uniformly local minimax points}, under mild continuity assumptions. In (non-convex) quadratic games, we show that local minimax points are (in some sense) equivalent to global minimax points. Finally, we study the stability of gradient algorithms near local minimax points. Although gradient algorithms can converge to local/global minimax points in the non-degenerate case, they would often fail in general cases. This implies the necessity of either novel algorithms or concepts beyond saddle points and minimax points in non-convex smooth games.
}

\end{abstract}


\maketitle

\blue{
\tableofcontents
}

\section{Introduction}
The existence of a saddle point in convex-concave minimax optimization follows from the celebrated minimax theorem \citep[e.g.][]{neumann1928theorie, sion1958general} and numerical algorithms for finding it have a long history in optimization \citep[e.g.][]{DemyanovMalozemov72,nemirovsky1983problem, zhang2019lower, lin2020nearoptimal}. Recent success in generative adversarial networks (GANs) \citep{goodfellow2014generative, heusel2017gans}, adversarial training \citep{madry2017towards} and reinforcement learning \citep{sutton1998introduction} has lead to new challenges \citep{razaviyayn2020non} for \emph{non-convex non-concave} (NCNC) minimax optimization, a.k.a.~NCNC zero-sum games. \blue{In such a formulation, we are given a non-convex non-concave bi-variate function $f(\xv, \yv)$. One player chooses $\xv$ to minimize $f(\xv, \yv)$, and another player chooses $\yv$ to maximize $f(\xv, \yv)$ (see detailed settings in \Cref{sec:global}). 
Since non-convex minimax optimization include non-convex  minimization as a special case, one cannot hope to find a global optimal solution efficiently. Therefore, we need to look for local optimal solutions as surrogates.
The fundamental gap between the theory for \emph{convex-concave games} and applications using \emph{non-convex non-concave games} raises an important question:
\begin{quote}
\emph{What is a \emph{reasonable definition}, in terms of both computational and theoretical convenience, of a local optimal point in non-convex (two-player, zero-sum) games?}
\end{quote} 
Unlike conventional minimization problems where local optimal solutions are well-defined, for non-convex games a satisfying definition is still under debate. }\citet{daskalakis2018limit} used a local version of saddle points to define local optimality. They studied the local convergence behavior of gradient descent ascent (GDA) \citep{arrow1958studies} and optimistic gradient descent (OGD) \citep{popov1980modification, daskalakis2018training}. \blue{Following this work, an important step was made by \citet{jin2019minmax}, who proposed a new definition of local optimality called local minimax points, compared them with local saddle points, and showed that they are equivalent to the stable solutions of GDA (in some sense). As GDA is widely used in practice, such as for adversarial training \citep{madry2017towards} and for GANs, an enhanced understanding of local minimax points is needed from both theory and application perspectives.} 

\blue{
Our work is based on \citet{jin2019minmax} and we aim to discuss the consequences and implications of their local minimax points to a greater extent. We believe this somewhat pedagogical study can help readers better understand local optimality in non-convex zero-sum games. 
Specifically, we aim to address the following questions: 
\begin{itemize}[parsep=0pt, partopsep=0pt]
\item What is the relation between local saddle and local minimax points?  \citet{jin2019minmax} showed that every local saddle point is local minimax, but is there a deeper connection? In Prop.~\ref{prop:ulmm}, we show that local saddle points are a special category of local minimax points called \emph{uniformly} local minimax points, under mild continuity assumptions.
\item How can we interpret local minimax points? We give a simplified and unified approach that recovers and extends existing notions of ``local mini-maximality,'' from the perspective of
\emph{infinitesimal robustness} \citep{hampel1974influence}.
Local minimax points are understood as the min-player doing infinitesimal robust optimization and the max-player following the strategy of the min-player (\Cref{sec:existing}). 
\item One of the benefits of local minimax points is that they are stationary points. Based on the interpretation using infinitesimal robustness, we go one step further and propose a new type of local optimal solutions, called \emph{local robust points} (Def.~\ref{def:glp}), which are still stationary points, but strictly include local minimax points as a special case. This new solution concept opens up the possibility to explore solutions in games that are not sequential, in contrast to the sequential Stackelberg games studied in \citet{jin2019minmax}.  
\item How do we identify local optimal solutions based on derivatives of the function? We analyze natural properties of local minimax points, including first- and second-order optimality conditions. These conditions extend the optimality conditions in \citet{jin2019minmax} to cases where the domains are constrained and where the Hessian for the max-player is not invertible. 
\item What is the connection between local and global optimal solutions? We analyze convex-concave games (\Cref{thm:cc}) and non-convex quadratic games (see below), and point out their difference from general non-convex games. 
\item Is a gradient algorithm stable at a certain local optimal solution? Under suitable conditions, \citet{jin2019minmax} showed the equivalence between the stable solutions of GDA and local minimax points when the Hessian for the max-player is invertible. We extend this study by analyzing the stability of several other popular gradient algorithms for min-max games and study if they converge to local optimal solutions (see below), even when the Hessian for the max-player is not invertible. Such study provides us with new insights for designing algorithms for minimax points.  
\end{itemize}
}
\noindent\blue{As a case study, we thoroughly characterize unconstrained quadratic games, which are potentially non-convex \citep{daskalakis2018limit, jin2019minmax, IbrahimAGM19, wang2019solving}. On the one hand, quadratic games could help us understand \emph{local convergence} of various gradient algorithms even on NCNC games. On the other hand, w.r.t.~the existence and equivalence of global and local versions of minimax points and saddle points, properties for quadratic games are not usually true for general NCNC games. For quadratic games:
\begin{itemize}[topsep=5pt, itemsep=0pt]
\item whenever both global (local) minimax and maximin points exist, global (local) saddle points must exist (\Cref{thm:quadr_2}; \Cref{eg:nc}, \Cref{eg:local_minimax_maximin_no_saddle});
\item global minimax points exist iff local minimax points exist (\Cref{thm:qc}; \Cref{eg:local_non_global});
\item being stationary and global minimax is equivalent to being local minimax (\Cref{thm:qc}; \Cref{eg:stationary_and_global_no_local}).
\end{itemize}  
The exact statements formalized as theorems and the corresponding NCNC counterexamples are listed in the parentheses above. 
Hence, we should be careful when using unconstrained quadratic games as a typical representative in the NCNC setting, especially w.r.t.~the optimality properties.
}

Since our unified definitions of local optimal points are all stationary points, a natural followup question is whether there exist gradient algorithms that can converge to them. In \Cref{sec:local_stab} we discuss extra-gradient algorithms \citep{korpelevich1976extragradient, popov1980modification, hsieh2019convergence}. By analyzing the spectrum of the Jacobian, we characterize the stable sets of hyperparameters, which yields insights on how to find local optimal points: \begin{itemize}[topsep=5pt, itemsep=0pt]
\item EG/OGD always locally converge to any non-degenerate local saddle points, and having larger extra-gradient steps increases the local stability;
\item for convergence to local minimax points, it is necessary to use two different step sizes and one step size cannot be arbitrarily small;
\item for convergence to local robust points, it is more appropriate to use OGD than EG as there are cases where OGD converges, but EG does not. 
\end{itemize}
For one-dimensional quadratic games, we establish the equivalence between local robust points and the stable solution of OGD, extending \citet{jin2019minmax} for local minimax points. 

We delay most proofs to the appendices to keep the main text concise. \blue{
To help readers navigate the results, we add a title for each definition, theorem, proposition, corollary, remark and example. We also provide a table for easier navigation on the next page. 
}
\begin{table*}
\vspace{-3.5em}
\thisfloatpagestyle{empty}%
\renewcommand{\arraystretch}{1.3}
\centering
\begin{tabular}{|c|c|c|} 
\hline
&\textbf{Statement} & \textbf{Reference} \\ \hline
\hline
\multirow{4}{*}{Definitions} & 
{\cellcolor[gray]{.9}} global/local saddle point & Definitions \ref{def:sadl}, \ref{def:loc_sadl} \\ \cline{2-3}
& \cellcolor{lightgray} global/local envelope function  & Definitions \ref{def:envelope}, \ref{def:envelope_loc}  \\
\cline{2-3} &  {\cellcolor[gray]{.9}} global/local minimax (maximin) point  & Definitions \ref{def:globalmm}, \ref{def:jin} \\ 
\cline{2-3} & \cellcolor{lightgray} local robust point (LRP) & \Cref{def:glp}
\\ 
\hline \hline
 & {\cellcolor[gray]{.9}} global saddle = global minimax + global maximin & \Cref{thm:sss} \\
\cline{2-3} &  \cellcolor{lightgray} both global minimax and maximin points exist, & \multirow{2}{*}{\Cref{eg:nc}}  \\
\multirow{1}{*}{Global} & \cellcolor{lightgray} but there is no global saddle point &  \\
\cline{2-3}  \multirow{1}{*}{results} &  {\cellcolor[gray]{.9}} instability of GDA & \Cref{exm:bl} \\
\cline{2-3} & \cellcolor{lightgray} global minimax points exist; & \multirow{2}{*}{\Cref{exm:onedq}}  \\
 & \cellcolor{lightgray} no global maximin or global saddle points& \\
\hline \hline 
 & {\cellcolor[gray]{.9}} optimality condition when $\yyv f$ is invertible & \Cref{thm:invertible_local_mm} \\
\cline{2-3} &  \cellcolor{lightgray} equivalence with \citet{jin2019minmax}  & Props.~\ref{prop:eq2}, \ref{prop:eq1} \\
\cline{2-3} & {\cellcolor[gray]{.9}} local saddle $\approx$ uniformly local minimax & Prop.~\ref{prop:ulmm}, \Cref{eg:uniform_not_local_saddle} \\
\cline{2-3}  &  \cellcolor{lightgray} stationary and/or global minimax $\neq$ local minimax & Examples~\ref{eg:glbstatl}, \ref{eg:stationary_and_global_no_local} \\
\cline{2-3} \multirow{2}{*}{Local} & {\cellcolor[gray]{.9} first-order sufficient condition and examples} &  Thm~\ref{suff:1st_suff_localmm}, Example~\ref{eg:app_first_suff_local_minimax} \\
\cline{2-3} \multirow{2}{*}{minimax} &\cellcolor{lightgray} & Thms \ref{cor:suff_2nd_localmm}, \ref{thm:2nd_suff_seeger}, Cor~\ref{cor:second_suff_interior}, \\
 & \multirow{-2}{*}{ \cellcolor{lightgray} second-order sufficient condition and examples} & \ref{cor:suf_jin}, Examples~\ref{eg:stronger_suff_cond}, \ref{eg:kawa-suff} \\ 
\cline{2-3} & \cellcolor[gray]{.9} & Thm~\ref{thm:nec_localmm}, Cor.~\ref{def:nec_2} \\
&  \multirow{-2}{*}{\cellcolor[gray]{.9} necessary conditions and related examples } & Examples~\ref{rem:critical}, \ref{eg:counter_jin}, \ref{rem:higher_order} \\
\cline{2-3} &\cellcolor{lightgray} local minimax exists, no global minimax & \Cref{eg:local_non_global} \\
\cline{2-3} & \cellcolor[gray]{.9} local minimax $\&$ maximin exist, no local saddle & \Cref{eg:local_minimax_maximin_no_saddle} \\
\hline \hline
{Convex-} & \cellcolor{lightgray} local minimax $=$ stationary $\Longrightarrow$ global minimax & \Cref{thm:cc} \\
\cline{2-3} concave &  \cellcolor[gray]{.9} local minimax $=$ local saddle $=$ LRP & \Cref{thm:cc_2} \\
\hline \hline
& \cellcolor{lightgray} optimality conditions & Thm \ref{thm:local_global_q}, Remark \ref{eg:quadratic} \\
\cline{2-3}  & \cellcolor[gray]{0.9} quadratic games can be non-convex & \Cref{eg:no_local_saddle} \\
\cline{2-3} &  \cellcolor{lightgray} stationary + global minimax = local minimax & \Cref{thm:qc} 
\\
\cline{2-3} Quadratics &\cellcolor[gray]{0.9} bilinear games & \Cref{cor:bilinear} \\
\cline{2-3} &  \cellcolor{lightgray} minimax + maximin = saddle & \Cref{thm:quadr_2} \\
\cline{2-3} & \cellcolor[gray]{0.9} non-uniformly minimax in quadratic games & \Cref{rem:non_uniform_minimax}  \\
\hline \hline
&\cellcolor{lightgray} equivalence between past-extra gradient and OGD &\Cref{lem:peg} \\
\cline{2-3} & \cellcolor[gray]{0.9}  stability criteria of EG/OGD & \Cref{thm:eg_stable} \\
\cline{2-3} & \cellcolor{lightgray} more aggressive extra-gradient steps, more stable & \Cref{thm:egratio} \\
\cline{2-3} & \cellcolor[gray]{0.9} EG/OGD are more stable than GDA & \Cref{thm:gd_eg} \\ 
\cline{2-3} Stability & \cellcolor{lightgray} local stability at local saddle points & \Cref{lem:loc_sadl}, \Cref{cor:eg_sadl} \\
\cline{2-3} & \cellcolor[gray]{0.9} local stability at strict local minimax points & \Cref{lem:loc_min_max}, \Cref{thm:neg_y} \\
\cline{2-3} & \cellcolor{lightgray}  local stability of gradient algorithms &  \multirow{2}{*}{\Cref{thm:all}} \\
&  \cellcolor{lightgray} at general local minimax points &  \\
\hline
\end{tabular}\label{tab:results}
\end{table*}


\noindent \blue{\textbf{Notation:}  In this paper we will use several conventions to denote optimality. To distinguish the concepts clearly, we use $\zv_\star = (\xv_\star, \yv_\star)$ for global/local saddle points; $\zv^* = (\xv^*, \yv^*)$ for global/local minimax points; $\zv_* = (\xv_*, \yv_*)$ for global/local maximin points and $\zv^\star = (\xv^\star, \yv^\star)$ for local robust points (\Cref{app:local_robust_point}). In Section \ref{sec:local_stab} we also use $\zv^* = (\xv^*, \yv^*)$ for general stationary points. 
When two different notions of optimality appear (such as in the proof of Prop.~\ref{prop:ulmm}), we choose the notation based on which notion comes first. 
}

\section{Global optimal points}\label{sec:global}

We focus on a \emph{two-player zero-sum smooth game} with a payoff function $f: \X\times \Y \to \R$ that is sufficiently many times differentiable depending on the context. \blue{We consider $\Xc\subset \R^n$ and $\Yc \subset \R^m$ to be non-empty subsets of Euclidean spaces and will add additional assumptions (convexity, closedness) when necessary. }
The min-player selects a strategy $\xv \in \X$ while the max-player selects a strategy $\yv \in \Y$, after which the min-player receives utility $-f(\xv,\yv)$ and the max-player receives $f(\xv,\yv)$. 
In our setting the min-player aims to minimize $f(\cdot, \yv)$ given (an estimate of) the max-player's strategy $\yv$ and conversely the max-player tries to maximize $f(\xv,\cdot)$ given  (an estimate of) the min-player's strategy $\xv$. 
In general, $f$ is not convex in $\xv$ and not concave in $\yv$ (NCNC), which has become extremely popular in machine learning (ML) recently, due to the rise of deep models. For instance, in generative adversarial networks \citep{goodfellow2014generative}, $\xv$ models the parameter of a generator while $\yv$ models that of a discriminator. In adversarial training \citep{madry2017towards}, $\xv$ is the robust model that we aim to train while $\yv$ represents possible adversarial attacks. 
In those examples (and many others), the function $f$ of interest is NCNC. A major challenge is to define proper notions of optimality (stationarity) and to understand the limiting behaviour of popular algorithms that are currently used by practitioners. 

In the convex setting, the following solution concept is well-known:

\begin{restatable}[global saddle]{defn}{}\label{def:sadl}
We call $(\xv_\star, \yv_\star) \in \X \times \Y$ global saddle if for all  $\xv \in \X$ and $\yv \in \Y$:
\be\label{eq:saddle_def}
f(\xv_\star, \yv) \leq f(\xv_\star, \yv_\star)\leq f(\xv, \yv_\star).
\en
In other words, we have \emph{simultaneously}:
\begin{align}
\xv_\star \in \argmin_{\xv\in\X} ~ f(\xv, \yv_\star), ~~ \yv_\star \in \argmax_{\yv\in\Y} ~ f(\xv_\star, \yv).
\end{align}
\end{restatable}

\blue{Global saddle points correspond to Nash equilibria \citep{nash1950equilibrium}, where each player knows the opponent's strategy exactly and aims to maximize the gain, but has no incentive to deviate from his/her current strategy. }

We may also encounter a scenario where the players move in sequence, and we need the following definitions:

\begin{restatable}[global envelope function]{defn}{}\label{def:envelope}
The upper and lower envelope functions are defined respectively as:
\begin{align}
\of(\xv) &:= \sup_{\yv\in\Y} ~ f(\xv,\yv), ~~
\uf(\yv) := \inf_{\xv\in\X} ~ f(\xv,\yv).
\end{align}
\end{restatable}
\blue{For envelope functions, we allow $\of$ to take value $+\infty$ and $\uf$ to take value $-\infty$. In \Cref{def:envelope}, the min-player for $\xv$ moves first and knows nothing about the max-player for $\yv$. A natural strategy is to minimize the worst-case payoff, i.e., the upper envelope function $\of(\xv)$, which is typically non-convex and non-smooth (even when $f$ is itself smooth):}
\begin{align}
\min_{\xv \in \X} ~~ \of(\xv).
\end{align}
\blue{On the other hand, the max-player simply maximizes $f(\xv, \cdot)$ given any $\xv$. This leads immediately to the following solution concept:}
\begin{restatable}[global minimax and maximin]{defn}{}\label{def:globalmm}
$(\xv^*, \yv^*) \in \X\times \Y$ is global minimax if 
\begin{align}
\textrm{\tcs}~\xv^* \in \argmin_{\xv\in\X} \of(\xv), ~  \textrm{\tcs[2]}~\yv^*\blue{= \yv^*(\xv^*)} \in \argmax_{\yv \in \Y} ~ f(\xv^*, \yv).    
\end{align}
In other words, for all
$\xv \in \X$ and $\yv \in \Y$:
\be\label{eq:minimax}
f(\xv^*, \yv) \leq f(\xv^*, \yv^*) =  \of(\xv^*) \leq \of(\xv).
\en
Similarly, we call $(\xv_*, \yv_*) \in \X\times \Y$ global maximin if 
\begin{align}
\textrm{\tcs}~\yv_* \in \argmax_{\yv\in\Y} \uf(\yv), ~  \textrm{\tcs[2]}~\xv_* \blue{= \xv_*(\yv_*)} \in \argmin_{\xv \in \X} ~ f(\xv, \yv_*).
\end{align}
In other words, for all
$\yv \in \Y$ and $\xv \in \X$:
\be\label{eq:maximin}
\uf(\yv) \leq \uf(\yv_*) = f(\xv_*, \yv_*) \leq f(\xv, \yv_*).
\en
\end{restatable}

\blue{
The concept of global minimax points is used widely in machine learning. For example, in the formulation of GAN \citep{goodfellow2014generative}, we first find the optimal parameters of the discriminator, $\thetav_D$, based on the parameters of the generator $\thetav_G$, and then optimize over $\thetav_G$. In other words, the optimal solution $(\thetav_G^*, \thetav_D^*)$ is a global minimax point (see the definition of $V$ in \citet{goodfellow2014generative}):
\be
V(\thetav_G^*, \thetav_D) \leq V(\thetav_G^*, \thetav_D^*), \, \max_{\thetav_D} V(\thetav_G, \thetav_D) \geq \max_{\thetav_D} V(\thetav_G^*, \thetav_D), \, \forall \thetav_G, \thetav_D.
\en
In the distributional robustness formulation \citep{sinha2018certifying}, we find the global minimax point $(\thetav^*, P^*)$, where $\thetav^*$ is the best model parameter and $P^*$ is the worst adversarial distribution, such that:
\be
\mathbb{E}_P[\ell(\thetav^*; Z)] \leq  \mathbb{E}_{P^*}[\ell(\thetav^*; Z)], \, \sup_{P\in \Pc}\mathbb{E}_{P}[\ell(\thetav; Z)] \geq \sup_{P\in \Pc}\mathbb{E}_{P}[\ell(\thetav^*; Z)], \, \forall \thetav \in \Theta, \, P\in \Pc. 
\en
Since we use neural networks in these applications, the payoff function is non-convex non-concave, and thus a saddle point may not always exist. 
}

\begin{rem}[\blue{\textbf{difficulty of finding global minimax}}]\label{rem:difficulty}
Although the notion of global minimax is well-defined, it suffers from some major issues once we enter the NCNC world: \begin{itemize}
\item \blue{We are not aware of an efficient algorithm \citep{murty1987some} for finding} a global minimizer $\xv^*$ of the non-convex function $\of$. This can be mitigated by contending with a local minimizer or even stationary point. 
\item Given $\xv^*$, \blue{it is NP-hard to find} a global maximizer $\yv^*$ of the non-concave function $f(\xv^*, \yv)$. While it is tempting to relax again to a local solution, this will unfortunately affect our notion of optimality for $\xv^*$ in the first place. We will return to this issue in the next section.
\item The envelope function $\of$ is not smooth even when $f$ is. Although we can turn to non-smooth optimization techniques, it will be inevitably slow to optimize $\of$.
\end{itemize}
\end{rem}

If we define the ``mirror'' function $\mf(\yv, \xv) = f(\xv,\yv)$, then $(\xv_*,\yv_*)$ is global maximin for $f$ iff $(\yv_*, \xv_*)$ is global minimax for $-\mf$. For this reason, we will limit our discussion mainly to minimax. \Cref{def:globalmm} arises in the optimization literature as well since it can be treated as a global solution to the minimax optimization problem: $$\min_{\xv\in \X}\max_{\yv \in \Y} f(\xv, \yv).$$ 

We note that the ordering of $\xv$ and $\yv$, i.e. which player moves first, matters: for instance, to get a global minimax pair $(\xv^*,\yv^*)$, we must first find $\xv^*$ and then conditioned on $\xv^*$ we find the ``certificate'' $\yv^*$. In game-theoretic terms, this is also known as a Stackelberg game \citep{von2010market}, where $\xv$ is the leader while $\yv$ is the follower.

It is well-known that weak duality, namely the inequality 
\begin{align}
\label{eq:wd}
\max_{\yv \in \Y} \uf(\yv) \leq \min_{\xv\in\X} \of(\xv)    
\end{align}
always holds. Strong duality, namely when equality is attained in \eqref{eq:wd}, holds only under stringent conditions. The following theorem easily follows from the definitions:

\begin{restatable}[e.g.~{\citealt[][Theorem 1.4.1]{facchinei2007finite}}]{thm}{mMinter}
\label{thm:sss}
For \emph{any} function $f$, the pair $(\xv_\star, \yv_\star)\in \X\times \Y$ is global saddle iff it is both global minimax and global maximin iff strong duality holds and 
\begin{align}\label{eq:argmin}
\xv_\star \in \argmin_{\xv\in\X} \of(\xv), ~ \yv_\star \in\argmax_{\yv\in\Y} \uf(\yv).
\end{align}
\end{restatable}

Let us give some examples to digest the definitions. In general, it is possible to find a game where both global maximin and minimax points exist, but there is no saddle point:

\begin{restatable}[both global minimax and maximin points exist; no saddle point]{eg}{}
\label{eg:nc}
Consider the bivariate function
\be\label{eq:nc}
f(x, y) = x^4/4 - x^2/2 + xy
\en
defined on $\R\times \R$. Global minimax points are clearly $\{0\}\times \R$ with value $0$. On the other hand, global maximin points are $(\pm 1, 0)$ with value $-1/4$. Indeed, 
\begin{align}
\max_y \min_x ~ x^4/4 - x^2/2 + xy \leq \max_y \min_x ~ x^4/4 - x^2/2 \leq -\tfrac14,
\end{align}
with equality attained at $(\pm 1, 0)$. 
The failure of strong duality proves the non-existence of saddle points (\Cref{thm:sss}). 
\end{restatable}

Note that given a global saddle pair $(\xv_\star, \yv_\star)$, $\yv_\star\in \Yc_\star := \argmax_{\yv\in\Y} f(\xv_\star, \yv)$ but not every certificate $\bar \yv \in \Yc_\star$ forms a global saddle pair with $\xv_\star$. This is known as ``instability,'' which is the reason underlying the non-convergence of the gradient descent ascent (GDA) algorithm \citep{Golshtein72, nemirovsky1983problem}.

\begin{restatable}[\blue{\textbf{instability of GDA}}]{eg}{}
\label{exm:bl}
Consider the bilinear (hence convex-concave) function $$f(x, y) = xy$$ defined on $\R \times \R$. It is easy to verify that global minimax points are precisely the set $\{0\}\times \R$ while global maximin points are $\R\times \{0\}$. 
Taking the intersection we have the unique global saddle point $(0, 0)$. This bilinear function is unstable, since given $x^* = 0$, not every global minimax certificate (namely the entire $\R$) forms a global saddle point with $x^*$. The last iterates of GDA do not converge to the unique global saddle point for this function with any \blue{(constant or not)} step size, provided that it is not initialized at the saddle point \citep[p.~211]{nemirovsky1983problem}.
\end{restatable}

Another interesting example consists of quadratic games, which we completely classify in \Cref{sec:q2}. Below we give a one-dimensional example where there is no global maximin or saddle point, but global minimax points exist.
\begin{restatable}[\blue{\textbf{global minimax points exist; no global maximin or saddle points}}]{eg}{}\label{exm:onedq}
Let $f(x, y) = ax^2 + by^2 + cxy$ with $a < 0, \, b < 0$ and $c^2 \geq ab$. According to the characterization in \Cref{thm:local_global_q}, $f$ only admits global minimax points. Note that for quadratic games, the existence of both global minimax and maximin points implies the existence of a saddle point, in sharp contrast with \Cref{eg:nc}.
\end{restatable}

From the example above, we see that even for simple quadratic games, saddle points may not exist.
In fact, unconstrained quadratic games are often given as typical examples for NCNC minimax optimization \citep{daskalakis2018limit, jin2019minmax, IbrahimAGM19, wang2019solving}. Locally, they can also be regarded as second-order approximations of a smooth function, and thus seem to be good representatives of NCNC games. However, we will show in \Cref{sec:q2} that they are quite special in many aspects.


\section{Local optimal points}\label{sec:local_opt}

In this section, we study definitions of local optimal points based on envelope functions. Compared to global optimal points, for local versions, we assume that we only have access to local information of $f$, i.e., given a point $(\xv, \yv)$, we only know $f$ over a neighborhood $\Nc(\xv)\times \Nc(\yv)$. Therefore, each player can only evaluate its current strategy by comparing with other strategies in the current neighborhood, corresponding to the notion of a local minimum (maximum). This can be achieved with the following local envelope functions.
In the definition below, we denote
\be\label{eq:nbhd}
\Nc(\yv^*, \epsilon) := \{\yv\in \Y: \|\yv-\yv^*\|\leq \e\},
\en
as the intersection of $\Y$ with a ball of radius $\e$ surrounding $\yv^*$ in $\R^m$, and similarly for $\Nc(\xv^*, \ve)$. Of course, the exact form of the ball depends on the norm we choose.
\begin{restatable}[local envelope function]{defn}{}\label{def:envelope_loc}
Fix a reference point $\yv^*\in\Y$ and radius $\epsilon \geq 0$, 
we localize the envelope function:
\begin{align}
\of_\epsilon(\xv) = \of_{\epsilon, \yv^*}(\xv) &:= \max_{\yv\in \Nc(\yv^*, \epsilon)} ~ f(\xv,\yv). 
\end{align}
The definition for $\uf_\epsilon(\yv) = \uf_{\epsilon,\xv^*}(\yv)$ is similar if we fix some $\xv^*\in\X$.
\end{restatable}

In \textsection~\ref{sec:existing} we propose a unified framework for local optimality and then study the differential optimality conditions in \textsection~\ref{sec:loc_family}.

\subsection{Definitions of local optimality}\label{sec:existing}

\blue{In this subsection, we start from the simplest definition of local optimality -- local saddle points, and then relax the constraints on the players to obtain the more general local minimax points \citep{jin2019minmax}. It is also possible to extend local minimax points further to local robust points (LRPs), which we delay to \Cref{app:local_robust_point}. }

In the NCNC setting, it is natural to consider local versions of saddle points (see \Cref{def:sadl}) by localizing around neighborhoods $\Nc(\xv_\star, \epsilon)$ and $\Nc(\yv_\star, \epsilon)$. Below, when we mention the local envelope functions $\of_{\e}(\xv)$ and $\uf_{\ve}(\yv)$ (see \Cref{def:envelope_loc}) the centers and the neighborhoods are often omitted since they are clear from the context.

\begin{restatable}[local saddle]{defn}{}\label{def:loc_sadl}
We call the pair $(\xv_\star, \yv_\star) \in \X\times \Y$ local saddle if there exists $\epsilon > 0$, such that for all $\xv \in \Nc(\xv_\star, \e)$ and $\yv\in \Nc(\yv_\star, \e)$, 
$f(\xv_\star, \yv) \leq f(\xv_\star, \yv_\star) \leq f(\xv, \yv_\star).
$
In other words,
\begin{itemize}
\item  Fixing $\xv_\star$, then $\yv_\star$ is a local maximizer of $\uf_{0, \xv_\star}(\yv) = f(\xv_\star, \yv)$; 
\item Fixing $\yv_\star$, then $\xv_\star$ is a local minimizer of  $\of_{0, \yv_\star}(\xv) = f(\xv, \yv_\star)$.
\end{itemize}
\end{restatable}
In the above definition, each player contends with the local optimality of its strategy by comparing with other strategies in a neighborhood. 
For local saddle points, \blue{we can WLOG choose the Euclidean norm $\|\cdot\|_2$ in the neighborhood definition (see \eqref{eq:nbhd}).}

\blue{We can now generalize the definition above. One player may not be aware of the exact strategy of the opponent, and thus doing robust optimization, given a certain range of the opponent's strategy.}
\blue{If $\xv$ is doing (a sequence of) local robust optimization and $\yv$ is doing usual optimization given the strategy of $\xv$, we have the following definition:}

\begin{restatable}[local minimax]{defn}{}\label{def:jin}
We call $(\xv^*, \yv^*) \in \X\times \Y$ a local minimax point if 
\begin{itemize} 
    \item Fixing $\xv^*$, then $\yv^*$ is a local maximizer of $\uf_{0, \xv^*}(\yv) = f(\xv^*, \yv)$; 
    \item Fixing $\yv^*$, then $\xv^*$ is a local minimizer of  $\of_{\epsilon_n, \yv^*}(\xv)$ for all $\e_n$ in \blue{some sequence $0 < \e_n \to 0$}. 
\end{itemize}
Furthermore, if there is a neighborhood $\Nc$ of $\xv^*$ such that~for all $\e_n$ in the sequence, $\xv^*$ is a local minimizer of $\of_{\e_n}$ on $\Nc$, then we call $(\xv^*, \yv^*)$ \emph{uniformly local minimax}. 
\end{restatable}

In the definition above, we also proposed uniformly local minimax points. By uniformity we mean that the neighborhood $\Nc$ does not depend on the element $\e_n$ in the sequence. We will show a close relation between local saddle points and uniformly local minimax points in \Cref{prop:ulmm}.

\blue{
\Cref{def:jin} reveals the asymmetric position between the two players for $\xv$ and $\yv$: $\yv$ needs only be a local certificate to testify the local optimality of $\xv$, but $\xv$ minimizes the envelope function $\of_\e(\xv)$, the worst-case payoff, simultaneously for a sequence of $\e_n \to 0$. 
By switching the role of $\xv$ and $\yv$ we obtain a similar notion of local maximin. When both players satisfy this stringent condition, we obtain a new optimality notion that we term as local robust points (\Cref{app:local_robust_point}). }

In \Cref{prop:eq2} we will see that \Cref{def:jin} has a seemingly stronger but equivalent form. To digest the somewhat complicated definition, we mention the following interpretation \citep[e.g.][]{wang2019solving}:
\begin{restatable}[\blue{\textbf{sufficient and necessary condition of local minimax when  $\partial^2_{\yv\yv} f$ is invertible}}]{thm}{slmm}\label{thm:invertible_local_mm}
Let $\X = \R^n, \Y=\R^m$ and $f: \R^n \to \R^m$ be twice continuously differentiable.
Suppose $\partial^2_{\yv\yv} f(\xv^*, \yv^*)$ is invertible (i.e.~non-degenerate), then $(\xv^*, \yv^*)$ is local minimax iff 
\begin{itemize}
\item $\partial_{\yv} f(\xv^*,\yv^*) = \zero$,  $\partial^2_{\yv\yv} f(\xv^*, \yv^*) \prec \zero$, and
\item $\xv^*$ is a local minimizer of the total function $f(\xv, \yv(\xv))$ where $\yv$ is defined implicitly near $\xv^*$ through the non-linear equation 
\begin{align}
\label{eq:im-tol}
\partial_{\yv} f(\xv, \yv) = \zero.
\end{align}
\end{itemize}
\end{restatable}

We emphasize that, unlike the definition in \citet{jin2019minmax}, we do not allow $\e_n$ to take 0 in \Cref{def:jin} for two reasons: (a) This allows us to better separate local saddle from local minimax; (b) It is  unnecessary to have $\e_n = 0$, as we will see in \Cref{prop:eq1}.

We now show how to simplify \Cref{def:jin}, starting with the following key lemma:

\begin{restatable}[]{lem}{inclusion}
\label{thm:eqseq}
Suppose $\yv^*$ maximizes $f(\xv^*, \yv)$ over some neighborhood $\Nc(\yv^*, \e_0)$.
If $\xv^*$ is a local minimizer of $\of_{\e, \yv^*}$, for some \blue{$0\leq \e \leq \e_0$}, then it remains a local minimizer (even over the same local neighborhood) of $\of_{\Nc}(\xv) := \max_{\yv \in \Nc} f(\xv, \yv)$  
for any $\Nc(\yv^*, \e) \subseteq \Nc \subseteq \Nc(\yv^*, \e_0)$.
\end{restatable}

\blue{Note that in the lemma above we allow $\e = 0$.} \Cref{thm:eqseq} reveals a key property of the local minimax point in \Cref{def:jin}: the norm in the neighborhood definition (see \eqref{eq:nbhd}) is immaterial (since we can shrink the neighborhood using \Cref{thm:eqseq} without impairing local minimaximality). In other words, the definition of local minimax points is topological and it does not depend on the norm we actually choose.
\blue{Using \Cref{thm:eqseq} we can ``strengthen'' the notion of local minimax even more. In particular, if \Cref{def:jin} holds for one diminishing sequence such that $\e_0 \geq \e_n \downarrow 0$ then it automatically holds for \emph{all} sequences that satisfy this same condition. We can even 
extend the sequence to an interval of $\e$'s:}

\begin{restatable}[\blue{\textbf{equivalent definition of local minimax}}]{prop}{propcon}
\label{prop:eq2}
The pair $(\xv^*, \yv^*)\in 
\X\times \Y$ is a local minimax point iff 
\begin{itemize} 
    \item Fixing $\xv^*$, then $\yv^*$ is a local maximizer of $\uf_{0, \xv^*}(\yv) = f(\xv^*, \yv)$;
    \item Fixing $\yv^*$, then $\xv^*$ is a local minimizer of  $\of_{\epsilon, \yv^*}(\xv)$ for all $\e \in (0, \e_0]$ with some $\e_0 > 0$. 
\end{itemize}
\end{restatable}

From \Cref{def:jin}, every uniformly local minimax point is local minimax. In fact, much more can be said between uniformly local minimax and local saddle:
\begin{restatable}[\blue{\textbf{local saddle and uniformly local minimax}}]{prop}{uniform}
\label{prop:ulmm}
Every local saddle point is uniformly local minimax. If for any $\xv\in \X$, $f(\xv, \cdot)$ is upper semi-continuous, then every uniformly local minimax point is local saddle.
\end{restatable}

Thus, for upper semi-continuous functions (in $\yv$), surprisingly, local saddle points coincide with uniformly local minimax points. We cannot drop the semi-continuity assumption:
\begin{eg}[\blue{\textbf{uniformly local minimax does not imply local saddle without semi-continuity}}]\label{eg:uniform_not_local_saddle}
Fix any $\yv^* \in \Y$ and consider the lower semi-continuous function
\begin{align}
f(x, y) = \begin{cases}
-x^2, & y = y^* \\
x^2, & y \ne y^*
\end{cases}
,
~~
\mbox{with}
~~
\of_{\e, y^*}(x) = 
\begin{cases}
-x^2, & \e = 0\\
x^2, & \e \ne 0
\end{cases}
.
\end{align}
$(0, y^*)$ is uniformly local minimax but not local saddle.
\end{eg}

\begin{figure}
\centering
\includegraphics[width=11cm]{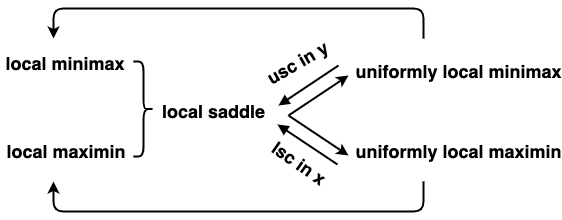}
\caption{The relationship among different notions of local optimality. usc: upper semi-continuity and lsc: lower semi-continuity. The arrow and the bracket signs mean ``to imply.'' For example, a uniformly local minimax point is \emph{bona fide} local minimax, and if a point is both local minimax and local maximin, it is local saddle. }
\label{fig:rel}
\end{figure}

\Cref{fig:rel} shows the relation between local saddle and (uniformly) local minimax (maximin) points. Finally, we prove that our \Cref{def:jin} coincides with the seemingly different one in Definition 14 of \citet[]{jin2019minmax}. Effectively, we manage to remove the continuity assumption in Lemma 16 of \citet{jin2019minmax} (cf.~\Cref{prop:eq2}). 
\begin{restatable}[\blue{\textbf{equivalence with \citet{jin2019minmax}}}]{prop}{equivalence}
\label{prop:eq1}
The pair $(\xv^*, \yv^*)$ is local minimax w.r.t.~function $f$ iff 
there exists $\d_0 > 0$ and a non-negative function $h$ satisfying $h(\d) \to 0 $ as $\d \to 0$, such that for any $\d \in (0, \d_0]$ and any $(\xv,\yv) \in \Nc(\xv^*, \d) \times \Nc(\yv^*, \d)$ we have 
\begin{align}
\label{eq:Jin}
f(\xv^*, \yv) \leq f(\xv^*, \yv^*) \leq \left[\max_{\yv' \in \Nc(\yv^*, h(\d))} ~ f(\xv, \yv')\right] =: \of_{h(\d)}(\xv).
\end{align}
\end{restatable}

From this equivalence, we can also derive that every local saddle point is local minimax  \citep[Proposition~17]{jin2019minmax}. However, our \Cref{prop:ulmm} gives a more detailed depiction of local saddle points. 
For functions that are convex in $\xv$ and concave in $\yv$, we naturally expect that local optimality is somehow equivalent to global optimality:
\begin{restatable}[\blue{\textbf{local and global minimax points in the convex-concave case}}]{thm}{convexconcave}
\label{thm:cc}
Let the function $f(\xv, \yv)$ be convex in $\xv$ and concave in $\yv$. Then, an interior point $(\xv,\yv)$ is local minimax iff it is stationary, i.e., $\partial_{\xv} f(\xv, \yv) = \zero$ and $\partial_{\yv} f(\xv, \yv) = \zero$ iff it is saddle. In particular, local minimax implies global minimax. 
\end{restatable}

However, non-stationary global minimax points cannot be local minimax, see \Cref{exm:bl} and \Cref{prop:stable} (below). Even with stationarity, the convex-concave assumption in \Cref{thm:cc} cannot be appreciably weakened, as illustrated in the following example: 
\begin{eg}[\blue{\textbf{stationary global minimax points are not local minimax in the non-convex case}}]\label{eg:glbstatl}
Let $f(x, y) = x^3 y$ be non-convex in $x$ but linear in $y$. The point $(x^*, y^*) = (0, 1)$ is clearly stationary  and global minimax. We verify that 
\begin{align}
\of_\e(x) = \begin{cases}
(1+\e)x^3, & x \geq 0\\
(1-\e)x^3, & x \leq 0
\end{cases}
,
\end{align}
hence $x^* = 0$ is not a local minimizer of $\of_\e$ (for any $\e < 1$) and $(0,1)$ is not local minimax. This counterexample is constructed by performing the $\Cc^1$ homeomorphic transformation $(x,y) \mapsto (x^3, y)$ of the bilinear game $b(x, y) = xy$. We can verify that (separate) homeomorphisms transform local/global minimax points accordingly. However, $\Cc^1$ homeomorphisms can turn non-stationary points into stationary (which is not possible in presence of convexity since stationarity equates minimality which is preserved under homeomorphisms).
\end{eg}
\noindent Nevertheless, for quadratic games, we can remove the convexity-concavity assumption, as will be shown in \Cref{thm:local_global_q} below.

\subsection{Optimality conditions}
\label{sec:loc_family}
Optimality conditions are an indispensable part of optimization \citep{bertsekas1997nonlinear} since they help us identify local optimal points and design new algorithms. In this section, we provide first- and second-order necessary and sufficient conditions for local minimax (maximin) points. Our results extend existing ones in \citet{jin2019minmax}.
We assume $\X$ and $\Y$ are closed\footnote{Of course they are contained in bigger open sets where derivatives of $f$ are well defined.} and thus $\Nc(\yv^*, \e)$ and $\Nc(\xv^*, \varepsilon)$ are compact. We build on some classical results in non-smooth analysis, for which we provide a self-contained review in \Cref{app:review}, including the definition of the directional derivative $\D \of_\e(\xv; \tv)$ of an envelope function $\of_\e$ at $\xv$ along direction $\tv$:
\begin{align}\label{eq:directional_derivative}
\D \of_\e(\xv; \tv) = \lim_{\alpha \to 0^+} \frac{\of_\e(\xv+\alpha \tv) - \of_\e(\xv)}{\alpha}.
\end{align}
\blue{Specifically, if $f$ and $\partial_\xv f$ are jointly continuous (continuous w.r.t.~$(\xv, \yv)$), then the directional derivative $\D \of_\e(\xv; \tv)$ always exist (\Cref{thm:danskin}). In the following subsections, $f\in \Cc^p$ means that $f$ is $p^{\rm th}$ continuously differentiable. } 

\subsubsection{First-order necessary conditions}\label{sec:1st-necessary}

\begin{restatable}[first-order necessary, local minimax]{thm}{firstlrp}\label{prop:stable}
Let $f\in \Cc^1$. At a local minimax point $(\xv^*, \yv^*)$, we have:
\be\label{eq:first_order_necessary_lrp}
\partial_{\xv} f(\xv^*, \yv^*)^\top \bar{\tv} \geq 0 \geq \partial_{\yv} f(\xv^*, \yv^*)^\top \ubar{\tv},
\en
for any directions $\bar{\tv} \in \K[d](\X, \xv^*)$, $\ubar{\tv} \in \K[d](\Y, \yv^*)$, where the cone
\begin{align}\K[d](\X, \xv) := \liminf_{\a\to 0^+} \frac{\X - \xv}{\a} := \{\tv: \forall \{\a_k\} \to 0^+~ \exists \{\a_{k_i}\} \to 0^+, \{\tv_{k_i}\} \to \tv, \tr
\blue{\mbox{ such that } \xv+\a_{k_i} \tv_{k_i} \in \X \}}\nonumber
\end{align}
and $\K[d](\Y, \yv)$ is defined similarly.
\end{restatable}

\begin{proof}
This result follows from its more general version for local robust points, \Cref{prop:stable_new}.
\end{proof}

\noindent In the theorem above, $\K[d](\X, \xv)$ is known as the derivable cone \cite[p.~198]{rockafellar2009variational}, which may strictly include the feasible tangent cone. \blue{When the set $\Xc$ is closed and convex, the two coincide \citep[p.~65]{hiriart2004fundamentals}: 
\be
\K[d](\X, \xv) = \overline{\rm cone}(\Xc - \xv) := {\rm cl}(\tv \in \R^n: \tv = \alpha(\yv - \xv), \, \yv \in \Xc, \, \alpha \geq 0),
\en
with ${\rm cl}$ denoting the closure of a set. We can derive a similar reduction when $\Yc$ is closed and convex. If both $\Xc$ and $\Yc$ are closed and convex, then \eqref{eq:first_order_necessary_lrp} reduces to:
\be\label{eq:stationary_constrained}
\partial_{\xv} f(\xv^*, \yv^*)^\top (\xv - \xv^*) \geq 0 \geq \partial_{\yv} f(\xv^*, \yv^*)^\top (\yv - \yv^*), \, \mbox{ for any }\xv\in \Xc, \, \yv \in \Yc.
\en
This can be regarded as a bi-variate version of first-order (necessary) optimality condition for a local minimum \citep[Prop.~2.1.2]{bertsekas1997nonlinear}. Solutions that satisfy such condition are often called stationary points. It extends the result in \citet{jin2019minmax} to the constrained case. Specifically, if $(\xv^*, \yv^*)$ is in the interior of $\Xc\times \Yc$, which always holds when $\X = \R^n$ and $\Y = \R^m$, then \Cref{prop:stable} simplifies to  
\be\label{eq:local_robust_point_stationary}
\partial_{\xv} f(\xv^*, \yv^*) = \zero , ~~\partial_{\yv} f(\xv^*, \yv^*)=\zero,
\en 
agreeing with \citet{jin2019minmax}. Moreover, Theorem \ref{prop:stable_new} in Appendix \ref{app:local_robust_point} shows that there is an even broader class of local optimal points named local robust points (LRPs) that has the same necessary conditions, \eqref{eq:first_order_necessary_lrp}, \eqref{eq:stationary_constrained} and \eqref{eq:local_robust_point_stationary}, as local saddle points \citep[e.g.][Definition 2]{baraz2020solving} and local minimax points. It also implies that in the convex-concave case, all local notions of optimality agree:
}

\begin{restatable}[\blue{\textbf{local optimal solutions in the convex-concave case}}]{cor}{cclocal}
\label{thm:cc_2}
Let $\X$ and $\Y$ be convex and the function $f(\xv, \yv)$ be convex in $\xv$ and concave in $\yv$. A point is local (global) saddle iff it is local minimax (maximin) iff it is an {\glp}.
\end{restatable}

\noindent This corollary does not hold in the  non-convex setting, see Examples~\ref{eg:no_local_saddle} and \ref{eg:glp}.

\subsubsection{First-order sufficient conditions}

Let us define the \emph{active set} of the \emph{zeroth} order (by ``zeroth'' we mean that only the function values are involved):
\be
&&\label{eq:Y_0_active}\,{\Y}_0(\xv^*;\e) = \{\yv\in\Nc(\yv^*, \e): \of_\e(\xv^*) = f(\xv^*, \yv)\}.
\en
We derive the first-order sufficient conditions for local minimax points (which follow from the sufficient condition in \Cref{thm:sufficient} and Danskin's theorem in  \Cref{thm:danskin}):
\begin{thm}[first-order sufficient condition, local minimax]\label{suff:1st_suff_localmm}
Assume $\n_\xv f(\xv, \yv)$ is continuous. If $f(\xv^*, \cdot)$ is maximized at $\yv^*$ over a neighborhood around $\yv^*$, and there exists $\e_0 > 0$ such that for any $ \e \in(0, \e_0)$, 
\be\label{eq:first_suff_local_minmax}
\zero \neq \tv\in \K[c](\X, \xv^*) \, \Longrightarrow 
\D \of_\e(\xv^*; \tv) = 
\max_{\yv \in {\Y}_0(\xv^*;\e)} ~ \partial_{\xv} f(\blue{\xv^*}, \yv)^\top {\tv} > 0,
\en
where the contingent cone is defined as:
\begin{align}
\K[c](\X, \xv) := \limsup_{\a\to 0^+} \frac{\X - \xv}{\a} := \{\tv: \exists \{\a_k\} \to 0^+, ~\{\tv_{k}\} \to \tv, \mbox{ such that } \xv+\a_{k} \tv_{k} \in \X \},\nonumber
\end{align}
then $(\xv^*, \yv^*)$ is a local minimax point.
\end{thm}

\noindent\blue{In the case when $\Xc$ is a convex set. $\K[c](\Xc, \xv)$ reduces to the usual cone of feasible directions:
\be
\K[c](\X, \xv) = \overline{\rm cone}(\Xc - \xv) := {\rm cl}(\tv \in \R^n: \tv = \alpha(\yv - \xv), \, \yv \in \Xc, \, \alpha \geq 0).
\en
If furthermore ${\rm cone}(\Xc - \xv)$ is closed, \eqref{eq:first_suff_local_minmax} becomes:
\be
\max_{\yv \in {\Y}_0(\xv^*;\e)} ~ \partial_{\xv} f(\blue{\xv^*}, \yv)^\top (\xv - \xv^*) > 0, \, \forall \xv^* \neq \xv\in \Xc.
\en
Let us demonstrate the first order condition with the following example:
\begin{eg}[\textbf{application of the first-order sufficient condition of local minimax points}]\label{eg:app_first_suff_local_minimax}
Suppose $f(x, y) = x y$ is bilinear. At $(x^*, y^*) = (0, 0)$, we have:
\be
\of_\e(x^*) = f(x^*, y) = 0, \, \forall y \in \R.
\en
Therefore, according to \eqref{eq:Y_0_active}, $\Yc_0(\xv^*; \e) = \Nc(y^*, \e)$. Also, $\partial_x f(x^*, y) = y$ and 
\be
\D \of_\e(x^*; x - x^*) = \max_{\Nc(y^*, \e)} y (x - x^*) = \e |x| > 0, \forall x\neq x^*.
\en
According to \Cref{suff:1st_suff_localmm}, $(x^*, y^*)$ is a local minimax point. 
\end{eg}
}

\subsubsection{Second-order necessary conditions}\label{sec:second_nec}

We now turn to the second-order necessary condition of local minimax points. We sometimes use $\n_{\xv \xv}^2 f$ as a shorthand for the second-order derivative $\n_{\xv \xv}^2 f(\xv^*, \yv^*)$, and similarly for other second-order partial derivatives. 
For a local minimax point $(\xv^*, \yv^*)$, $\yv^*$ maximizes $f(\xv^*, \cdot)$ locally, and thus we have the property that $\of_\e(\xv^*) = f(\xv^*, \yv^*)$ for any small $\e$, from which we can make significant simplifications. The following technical lemma, when combined with the necessity condition in \Cref{thm:nec_1st}, allows us to classify the directions:
\begin{restatable}[\blue{\textbf{directional derivatives for different $\of_\e$}}]{lem}{nested}\label{lem:dec_f_p}
\blue{Suppose $f$ and $\n_\xv f$ are jointly continuous and thus the directional derivative \eqref{eq:directional_derivative} exists.}
If $\yv^*$ is a local maximizer of $f(\xv^*, \cdot)$ over a neighborhood $\Nc(\yv^*, \e_0)$, then for any $0 \leq \e_1 \leq \e_2 \leq \e_0$, ${\Y}_0(\xv^*;\e_1)\subseteq {\Y}_0(\xv^*;\e_2)$ and for each ${\tv}\in \K[d](\X, \xv^*)$, $\D\of_{\e_2}(\xv^*;\tv) \geq \D\of_{\e_1}(\xv^*;\tv)$.
\end{restatable}

Indeed, for a local minimax point $(\xv^*, \yv^*)$ and any direction ${\tv}\in \K[d](\X, \xv^*)$, we know from the necessity condition in \Cref{thm:nec_1st} that $\D \of_{\e}(\xv^*;{\tv}) \geq 0$ for all small $\e$, which, combined with \Cref{lem:dec_f_p} above, leaves us with two possibilities: 
\begin{enumerate}
\item $\D \of_{\e}(\xv^*;{\tv}) > 0$ for all $\e >0$ smaller than some $\e_0(\tv)$;
\item $\D \of_{\e}(\xv^*;{\tv}) = 0$ for all $\e >0$ smaller than some $\e_0(\tv)$.
\end{enumerate}  We 
call the direction $\tv$ a \emph{critical direction} in the second case above. 
With this distinction among directions, we derive the second-order necessary \blue{condition} for local minimax points:
\begin{restatable}
[second-order necessary condition, local minimax]{thm}{secondminmax}\label{thm:nec_localmm}
Suppose $f, \n_\xv f$ and $\n_{\xv\xv}^2 f$ are all (jointly) continuous. If $(\xv^*, \yv^*)$ is a local minimax point, then for each direction ${\tv} \in \K[d](\X, \xv^*)$, \blue{one of the following holds:}
\begin{enumerate}
\item $\D\of_\e(\xv^*; {\tv}) > 0$ for all $\e >0$ smaller than some $\e_0(\ tv)$;
\item $\D\of_\e(\xv^*; {\tv}) = 0$ for all $\e >0$ smaller than some $\e_0(\tv)$ (i.e. $\tv$ is critical), in which case we further have
\be\label{eq:localmm_min}
{\tv}^\top \n_{\xv \xv}^2 f(\xv^*, \yv^*){\tv} + 
\tfrac{1}{2} \limsup_{\zv\to \yv^*} \left[\max\{\n_\xv f(\xv^*, \zv)^\top \tv,\,0\}^2 (f(\xv^*, \yv^*) - f(\xv^*, \zv))^\dag\right]
\geq 0,\tr
\en
where $t^\dag = 1/t$ if $t\neq 0$ and $0$ otherwise.
\end{enumerate}
\end{restatable}

The important point to take from \Cref{thm:nec_localmm} is that we should test the second-order condition \eqref{eq:localmm_min} only for critical directions, and the second-order derivatives of $f$ may not fully capture the second-order derivatives of the envelope function $\of_\e$, which can be clearly demonstrated from the following examples:
\begin{restatable}[\blue{\textbf{the importance of critical directions}}]{eg}{}\label{rem:critical}
Let $$f(x, y) = -x^2 + x y^3$$ be defined over $\X = \Y = \R$ and consider the local minimax point $(x^*,y^*) = (0,0)$. Indeed, for any $\e > 0$, $x^*$ is a local minimizer of $\of_{\e}(x) = |x|\e^3 - x^2$. However, $\n_{xx}^2 f = -2$ while $f(x^*,y^*) = f(x^*, z) = 0$ for any $z$. Thus, the second-order condition \eqref{eq:localmm_min} fails at the directions $t = \pm 1$. However, there is no contradiction since these directions are not critical: Indeed, using \Cref{thm:danskin} we can verify that $\D \of_\e(x^*; \pm 1) = \e^3 > 0$.
\end{restatable}

\begin{eg}[\blue{\textbf{the importance of critical directions under multiple dimensions}}]\label{eg:counter_jin}
Let $$f(\xv, \yv) = -x_2^2 + x_2 y_2^3 - (y_1+y_2)^2 + 2 x_1 (y_1 + y_2)$$ be defined over $\X = \Y = \R^2$ and consider the local minimax point $(\xv^*, \yv^*) = (\zero, \zero)$: Indeed, $f(\xv^*, \cdot)$ is clearly maximized locally at $\yv^* = \zero$ and upon choosing $y_1 = x_1 - \sgn(x_2)\e/2, y_2= \sgn(x_2)\e/2$ and considering $|x_1| < \e/2$ and $|x_2| < (\e/2)^3$, we have
\begin{align}
\|\yv - \xv\|_\infty  &\leq \e/2 + (\e/2)^3,\, \of_\e(\xv) \geq f(\xv, \yv) = x_1^2 + |x_2|(\e/2)^3 - x_2^2\geq 0 = \of_\e(\xv^*),
\end{align}
where we choose WLOG the $\ell_\infty$ norm in our neighborhood definition \eqref{eq:nbhd}. The second-order derivatives are:
\be
\yxv f = \begin{bmatrix}
2 & 0 \\
2 & 0
\end{bmatrix}, \, \yyv f = \begin{bmatrix}
-2 & -2 \\
-2 & -2
\end{bmatrix}, \, \xxv f = \begin{bmatrix}
0 & 0 \\
0 & -2
\end{bmatrix}.
\en
We have $\Y_0(\xv^*;\e) = \{\yv\in \Nc_\infty(\xv^*,\e): y_1 + y_2 = 0\}$ and for any direction $\tv$, 
\begin{align}
\D \of_\e(\xv^*; \tv) = \max_{\yv \in \Y_0(\xv^*;\e)} \tv^\top \partial_{\xv} f(\xv^*, \yv) = \e^3 |t_2|\geq 0.
\end{align} 
It follows that the critical directions satisfy $t_2 = 0$. Take a non-critical direction $\tv = (1, 3)$, we easily verify that $(\yxv f)\tv = (2,2)$ lies in the range space of $\yyv f$. However, 
\begin{align}
&\limsup_{\zv\to \yv^*} \left[\max\{\n_\xv f(\xv^*, \zv)^\top \tv,\,0\}^2 (f(\xv^*, \yv^*) - f(\xv^*, \zv))^\dag\right] \tr
&= 
\limsup_{\zv \to \zero,z_1+z_2\ne0} \frac{[2(z_1+z_2) + 3z_2^3]_+^2}{(z_1+z_2)^2} = 4,
\end{align}
so that the second-order condition in \eqref{eq:localmm_min}, which in this case coincides with $$\tv^\top(\n_{\xv\xv}^2 f - \n_{\xv\yv}^2 f (\n_{\yv\yv}^2 f)^{\dag} \n_{\yv\xv}^2 f)\tv, $$ does not hold ($-18+2 = -16 \not\geq 0$). Nevertheless, along a critical direction $\tv$ (where $t_2 = 0$):
\be
\tv^\top \n_{\xv \xv}^2 f(\xv^*, \yv^*)\tv = 0, \, f(\xv^*, \zv) = -(z_1 + z_2)^2, \, \n_\xv f(\xv^*, \zv)^\top \tv = 2 t_1 (z_1 + z_2),
\en
and thus the left-hand side of \eqref{eq:localmm_min} simplifies to $2t_1^2 \geq 0$. In other words, the second-order condition indeed holds for critical directions. 
\end{eg}

\begin{eg}[\blue{\textbf{high order derivatives might be involved in Theorem \ref{thm:nec_localmm}}}]\label{rem:higher_order}
The second term in \eqref{eq:localmm_min} may involve higher-order information of $f$\blue{, rather than the standard second-order optimality condition for e.g.~the minimizer of a smooth function}. \blue{The higher-order term comes from the difference of function values.} Let $f(x, y) = -x^2 - y^4 + 4 x y^2$ and consider the local minimax point $(x^*, y^*) = (0, 0)$. We have $\Y_0(x^*;\e) = \{y^*\}$ hence every direction is critical. In the direction $t = 1$, the l.h.s.~of \eqref{eq:localmm_min} becomes $$-2 + \max\{4z^2 t, 0\}^2/(2 z^4) = 6 > 0.$$ 
\end{eg}

\noindent Under the condition that $\n_{\yv\yv}^2 f$ is invertible, we recover the following result from \cite{jin2019minmax}:
\begin{restatable}[\textbf{second-order necessary condition, invertible}]{cor}{thmnec}\label{def:nec_2}
Let $f\in \Cc^2$. At a local minimax point $(\xv^*, \yv^*)$ in the interior of $\X\times \Y$, if $\n_{\yv\yv}^2 fß$ is invertible, then \be\label{eq:nece}\n_{\yv\yv}^2 f \prec {\bf 0}\mbox{ and } \n_{\xv\xv}^2 f - \n_{\xv\yv}^2 f (\n_{\yv\yv}^2 f)^{-1} \n_{\yv\xv}^2 f\succeq {\bf 0}.\en 
\end{restatable}

\begin{proof}
It is easy to prove $\n_{\yv\yv}^2 f \cle \zero$ and since $\n_{\yv\yv}^2 f$ is invertible, we have $\n_{\yv\yv}^2 f\cl \zero$. By expanding $f(\xv^*, \zv)$ to the second order, the second term in \eqref{eq:localmm_min} becomes:
\be
 \limsup_{\zv\to \yv^*} \frac{\max\{(\zv - \yv^*)^\top (\n_{\yv \xv}^2 f) \tv, 0\}^2}{(\zv - \yv^*)^\top (-\n_{\yv \yv}^2 f) (\zv - \yv^*)}.
\en
With a change of variables $\zv - \yv^* = (-\n_{\yv\yv}^{2} f)^{-1/2}(\wv - \yv^*)$ and using Cauchy--Schwarz inequality, we obtain
$ -\tv^\top \n_{\xv\yv}^2 f (\n_{\yv\yv}^2 f)^{-1} (\n_{\yv\xv}^2 f) \tv$. It follows that $\n_{\xv\xv}^2 f - \n_{\xv\yv}^2 f (\n_{\yv\yv}^2 f)^{-1} \n_{\yv\xv}^2 f \cge \zero$.
\end{proof}

Finally, we can compare our second-order necessary condition with Proposition 19 of \citet[][]{jin2019minmax}, which applies to quadratic functions (cf.~\Cref{eg:quadratic}). The difference is that Proposition 19 of \citet[]{jin2019minmax} did not take the critical directions and higher-order derivatives into consideration, as demonstrated by Examples \ref{rem:critical} and \ref{rem:higher_order}.

\subsubsection{Second-order sufficient conditions}\label{sec:ns}

We introduce two second-order sufficient conditions for local minimax points, with the help of results from non-smooth optimization literature \citep{Seeger88, kawasaki1992second}. \blue{Our results extend \citet{jin2019minmax} to the case when $\yyv f$ is not invertible, which may happen in real applications. }

In the following theorem, we define $x_+ = \max\{x, 0\}$ and the first order activation set:
\be
\,{\Y}_1(\xv^*;\e;\tv) = \{\yv\in\Y_0(\xv^*, \e): 
\D \of_\e(\xv^*;\tv) = \n_\xv f(\xv^*, \yv)^\top \tv\}.
\en

\begin{restatable}[\textbf{second-order sufficient condition, local minimax}]{thm}{secondsuff}\label{cor:suff_2nd_localmm}
Assume $\X = \R^n$ and $\Y$ is convex and $f$, $\n_\xv f$, $\n_{\xv\xv}^2 f$ are (jointly) continuous. At a stationary point $(\xv^*, \yv^*)$, if there exists $\e_0 > 0$ such that:
\begin{itemize}
    \item $f(\xv^*, \cdot)$ is maximized at $\yv^*$ on $\Nc(\yv^*, \e_0)$;
    \item along each critical direction $\tv\neq \zero$:
\be\label{eq:sec_suff_local_minmax}
{\tv}^\top \n_{\xv \xv}^2 f(\xv^*, \yv^*){\tv} + 
\frac{1}{2} \limsup_{\zv\to \yv^*} \left(((\n_\xv f(\xv^*, \zv)^\top \tv)_+)^2 (f(\xv^*, \yv^*) - f(\xv^*, \zv))^\dag\right)
> 0,
\en
\blue{and in any direction $\dv\in \R^m$, there exist $\a, \b \neq 0$ and $p, q > 0$ such that~for every $\yv\in \Y_1(\xv^*;\e_0;\tv)$, the following Taylor expansion holds:}
\begin{align}\label{eq:taylor_exp}
f(\xv^*, \yv + \d \dv) = f(\xv^*,  \yv) + \a \d^p + o(\d^p), \, \n_\xv f(\xv^*, \yv + \d \dv)^\top\tv = \b \d^q + o(\d^q),
\end{align}
\end{itemize}
then $(\xv^*, \yv^*)$ is a local minimax point. 
\end{restatable}

\noindent \blue{Note that in the statement above, the variables $\a, \b$ and $p, q$ may depend on the direction $\dv$. If $f \in \Cc^\infty$ is smooth and both $f(\xv^*, \cdot)$ and $\n_\xv f(\xv^*, \cdot)^\top \tv$ have non-zero Taylor expansions, then \eqref{eq:taylor_exp} is always true for every $\yv\in \Y_1(\xv^*;\e_0;\tv)$. }
Here by ``critical direction'' we mean that $\D \of_{\e}(\xv^*;{\tv}) = 0$ for some $\e_0 > 0$ and any $\e \in [0, \e_0]$, as discussed in \Cref{sec:second_nec}. Another second-order sufficient condition for $f\in \Cc^2$ is:
\begin{restatable}[\textbf{second-order sufficient condition, local minimax}]{thm}{secondsuffs}
\label{thm:2nd_suff_seeger}
Assume $f\in \Cc^2$ and let $\X$ be convex. Suppose $\yv^*$ is a local maximizer of $f(\xv^*, \cdot)$ and that $(\xv^*, \yv^*)$ is an interior stationary point. \blue{If there is $\e_0 > 0$ and for any $\e \in (0, \e_0]$, 
there exist $R, r > 0$ such that~for any feasible direction $\|\tv\| = 1$ that satisfies $0\leq \D \of_\e(\xv^*;\tv)\leq r$, we have}
\be\label{eq:2nd_seeger}
\max_{\yv \in \Y_0(\xv^*;\e)} \max_{\substack{\vv\in\Vc(\xv^*,\yv;\tv)\\ \|\vv\|\leq R}}\max_{\substack{\wv\in \K[d](\Omega, \yv; \vv), \\
\|\wv\|\leq R}}~&& \inner{\begin{bmatrix} \n_{\xv\xv}^2 f(\xv^*,\yv) & \n_{\xv\yv}^2 f(\xv^*,\yv) \\ \n_{\yv\xv}^2 f(\xv^*,\yv) & \n_{\yv\yv}^2 f(\xv^*,\yv) \end{bmatrix}{\tv \choose \vv}}{{\tv \choose \vv}} + \tr 
 && + \inner{\n_{\yv} f(\xv^*,\yv)}{\wv}> \zero, 
\en
then this point is local minimax, where $\Vc(\xv, \yv; \tv) := \{\vv\in\K[d](\Omega, \yv): \D\of_\e(\xv;\tv) = {\n_\xv f(\xv,\yv)}^\top{\tv} + {\n_\yv f(\xv,\yv)}^\top {\vv}\}$, $\Omega := \Nc(\yv^*, \e)$ and 
\blue{
\begin{align}
\K[d](\Omega, \yv; \vv) := \liminf_{t\to 0^+} \frac{\Omega - \yv - t \vv}{t^2/2} &:= \{\gv: \forall \{t_k\} \downarrow 0~ \exists \{t_{k_i}\} \downarrow 0, \{\gv_{k_i}\} \to \gv, \tr
&\yv+t_{k_i} \vv + t_{k_i}^2 \gv_{k_i}/2 \in \Omega \}.
\end{align}
}
\end{restatable}

\noindent \blue{
The definition of feasible directions for convex sets can be found in e.g.~\citet[][]{hiriart2013convex}. We used the convention that maximizing over an empty set yields $-\infty$. Specifically, if there exists $\yv\in \Y_0(\xv^*, \e)$ such that it is in the interior of $\Yc$, Theorem~\ref{thm:2nd_suff_seeger} can be simplified as:
\begin{cor}[\textbf{second-order sufficient condition, interior version}]\label{cor:second_suff_interior}
Assume $f\in \Cc^2$ and let $\X$ be convex. Suppose $\yv^*$ is a local maximizer of $f(\xv^*, \cdot)$ and that $(\xv^*, \yv^*)$ is an interior stationary point. If there is $\e_0 > 0$ such that $\Nc(\yv^*, \e_0) \subset \Yc \subset \R^m$, and for any $\e \in (0, \e_0)$, 
there exist $R, r > 0$ such that~for any feasible direction $\|\tv\| = 1$ that satisfies $0\leq \D \of_\e(\xv^*;\tv)\leq r$, we have: 
\begin{align}\label{eq:2nd_seeger_new}
\max_{\yv \in \Y_0(\xv^*;\e)} \max_{\substack{\vv\in\Vc(\xv^*,\yv;\tv)\\ \|\vv\|\leq R}} \max_{\|\wv\|\leq R} \inner{\begin{bmatrix} \n_{\xv\xv}^2 f(\xv^*,\yv) & \n_{\xv\yv}^2 f(\xv^*,\yv) \\ \n_{\yv\xv}^2 f(\xv^*,\yv) & \n_{\yv\yv}^2 f(\xv^*,\yv) \end{bmatrix}{\tv \choose \vv}}{{\tv \choose \vv}} + \inner{\n_{\yv} f(\xv^*,\yv)}{\wv}> \zero, 
\end{align}
then this point is local minimax, where $\Vc(\xv, \yv; \tv) := \{\vv\in \R^m: \D\of_\e(\xv;\tv) = {\n_\xv f(\xv,\yv)}^\top{\tv} + {\n_\yv f(\xv,\yv)}^\top {\vv}\}$.
\end{cor}
\begin{proof}
If $\yv \in \Nc(\yv^*, \e)$, then we have $\K[d](\Omega, \yv) = \K[d](\Omega, \yv; \vv) = \R^m$. 
\end{proof}
}

\noindent In the special case when $\n_{\yv\yv}^2 f(\xv^*, \yv^*) \cl \zero$, we have the following corollary. \blue{This special type of local minimax points that satisfy \eqref{eq:suff} are also known as \emph{strict local minimax points} \citep{jin2019minmax}.}

\begin{restatable}[\textbf{second-order sufficient condition, invertible, \cite{jin2019minmax}}]{cor}{suff}\label{cor:suf_jin}
Let $f$ be twice continuously differentiable. At an interior stationary point $(\xv^*, \yv^*)\in \X \times \Y$, if \be\label{eq:suff}\n_{\yv\yv}^2 f\prec {\bf 0}\mbox{ and } \n_{\xv\xv}^2 f - \n_{\xv\yv}^2 f (\n_{\yv\yv}^2 f)^{-1} \n_{\yv\xv}^2 f\succ {\bf 0},\en
then $(\xv^*, \yv^*)$ is a local minimax point. 
\end{restatable}
\begin{proof}
The active set $\Y_0(\xv^*;\e) = \{\yv^*\}$ is a singleton. From Danskin's theorem (\Cref{thm:danskin}) all directions are critical. The l.h.s.~of \eqref{eq:2nd_seeger} becomes $ \tv^\top (\n_{\xv\xv}^2 f - \n_{\xv\yv}^2 f (\n_{\yv\yv}^2 f)^{-1} \n_{\yv\xv}^2 f)\tv$ if we choose $R = \|(\n_{\yv\yv}^2 f)^{-1}\n_{\yv\xv}^2 f\|$.
\end{proof}

However, \Cref{cor:suf_jin} does not fully cover \Cref{thm:2nd_suff_seeger} when $\n_{\yv\yv}^2$ is not invertible:
\begin{eg}[\blue{\textbf{Theorem \ref{thm:2nd_suff_seeger} strictly includes Corollary \ref{cor:suf_jin}}}]\label{eg:stronger_suff_cond}
Take $$f(x, y) = x y^2 + x^2$$ and a stationary point $(x^*, y^*) = (0, 0)$. $\D\of_\e(x^*;t) = \e^2$ if $t = 1$ and $\D \of_\e(x^*;t) = 0$ if $t = -1$. Take $r = \e^2/2$. Along the critical direction $t = -1$, the l.h.s.~of \eqref{eq:2nd_seeger} becomes $2 > 0$, since $\n_{y} f(x^*, y) = 0$, and $\Vc(x^*,y;t) = \varnothing$ if $y \neq 0$ and $\R$ if $y = 0$.  
So, $(0, 0)$ is local minimax from \Cref{thm:2nd_suff_seeger}. Note that \Cref{cor:suff_2nd_localmm} does not apply since $f(x^*, y)$ does not have a non-zero Taylor expansion. 
\end{eg}

We also give an example when \Cref{thm:2nd_suff_seeger} is not applicable but \Cref{cor:suff_2nd_localmm} is:
\begin{eg}[\blue{\textbf{application of Theorem \ref{cor:suff_2nd_localmm} where Theorem \ref{thm:2nd_suff_seeger} cannot be applied}}]\label{eg:kawa-suff}
 Take $$f(x, y) = x y^3 - y^6$$ and a stationary point $(x^*, y^*) = (0, 0)$. Fixing $x^* = 0$, $f(x^*, \cdot)$ is maximized at $0$, and for any $t\neq 0$, $\D \of_\e(x^*; t) = \max_{y^6 = 0}y^3 t = 0$. Since $\n_x f(x^*, z) = z^3 t$ and $f(x^*, y^*) - f(x^*, z) = z^6$, the l.h.s.~of \eqref{eq:sec_suff_local_minmax} is $t^2/2 > 0$. Moreover, $\Y_1(x^*;\e_0; t) = \{y^*\}$ for any $\e_0 > 0$, and $$f(x^*, y^* + \d d) = -\d^6 d^6, \, \n_x f(x^*, y^* + \d d)^\top t = \d^3 d^3 t.$$
 So, $(0, 0)$ is a local minimax point. Note that \Cref{thm:2nd_suff_seeger} does not apply since $\Y_0(x^*;\e) = \{0\}$ and all second-order derivatives are zero.
\end{eg}

\section{Quadratic games: A case study}
\label{sec:q2}

In this section we study quadratic games with the following form:
\be\label{eq:quadr}
q(\xv, \yv) = \frac{1}{2} \begin{bmatrix}
\xv\\ \yv \\ 1
\end{bmatrix}^\top
\begin{bmatrix}
\A &\cC & \av \\
\cC^\top & \B & \bv \\
\av^\top & \bv^\top & c\\
\end{bmatrix}
\begin{bmatrix}
\xv \\ \yv \\ 1
\end{bmatrix},
\en
where $\xv \in \X = \R^n$ and $\yv \in \Y = \R^m$.
In particular, a game is \emph{bilinear} if $\A, \B$ vanish and \emph{homogeneous} if $\av, \bv$ vanish. \blue{Since quadratic games are continuous, local saddle points are the same as uniformly local minimax points (see \Cref{prop:ulmm}).}

Our first result completely characterizes stationary, global minimax and local minimax points for homogeneous quadratic games:
\begin{restatable}[\blue{\textbf{sufficient and necessary conditions for optimality in quadratic games}}]{thm}{minimaxq}
\label{thm:local_global_q}
For (homogeneous) unconstrained quadratic games, a pair $(\xv, \yv)$ is
\begin{itemize}
\item stationary iff 
\begin{align}
\label{eq:stat}
\begin{bmatrix}
\Av & \Cv \\ \Cv^\top & \Bv
\end{bmatrix} 
\begin{bmatrix}
\xv \\ \yv
\end{bmatrix}
= \zero;
\end{align}
\item global minimax iff $\Bv \preceq \zero$, $\p{\Lv}(\Av - \Cv \Bv^\dag \Cv^\top) \p{\Lv} \cge \zero$ where $\Lv = \Cv \p{\Bv}$, and 
\begin{align}
\label{eq:gLmM}
\begin{bmatrix}
\p{\Lv} &  \\  & \Iv
\end{bmatrix} 
\begin{bmatrix}
\Av & \Cv \\ \Cv^\top & \Bv
\end{bmatrix} 
\begin{bmatrix}
\xv \\ \yv
\end{bmatrix}
 = \zero;
\end{align}
(Recall that $\p{\Lv} = \Iv - \Lv\Lv^\dag$ is the orthogonal projection onto the null space of $\Lv^\top$.)
\item local minimax iff $\Bv \preceq \zero$, $\p{\Lv}(\Av - \Cv \Bv^\dag \Cv^\top) \p{\Lv} \cge \zero$, and stationary (i.e. \eqref{eq:stat} holds).
In particular, local minimax points are always global minimax.
\end{itemize}
\end{restatable}

Comparing \Cref{thm:local_global_q} with \Cref{thm:cc}, we find that in both cases, local minimax points are global minimax, which is not true in general (\Cref{eg:local_non_global}). This shows that there exists some ``hidden convexity'' in quadratic games when local/global minimax points exist: fixing any $\xv$, $q(\xv, \cdot)$ is concave in $\yv$; $\oq(\xv)$ is convex in $\xv$ (see \eqref{eq:barq}).

\begin{rem}[\blue{\textbf{application of \Cref{thm:nec_localmm} in quadratic games}}]\label{eg:quadratic}
We could also use \Cref{thm:nec_localmm} to obtain the necessary condition of local minimax points for quadratic games. First write $$f(\xv^*, \yv^*) - f(\xv^*, \yv) = -\yv^\top\Bv\yv/2\mbox{ and }-\n_\xv f(\xv^*, \yv)^\top \tv = -\yv^\top \Cv^\top \tv$$ and $\D \of_\e(\xv^*; \tv) \geq \d\|\p{\Bv}\Cv^\top \tv\|$ for some $\d > 0$. The critical directions are $\tv \in \Nc(\p{\Bv}\Cv^\top)$. If $\Bv \Cv^\top = \zero$, then $\n_\xv f(\xv^*, \yv)^\top \tv = 0$ for any $\yv$ and thus the second term in \eqref{eq:localmm_min} is zero. So, we have $\p{\Lv}\Av \p{\Lv} \cge 0$ with $\Lv = \Cv \p{\Bv}$. Otherwise, take critical directions $\tv$ such that $\tv \in \Nc(\p{\Bv}\Cv^\top)$. The second term in \eqref{eq:localmm_min} becomes $ -\tv^\top \Cv \Bv^\dag \Cv^\top \tv$ (using Cauchy--Schwarz). Combining with the case $\Bv \Cv^\top = \zero$, we have $\p{\Lv}(\Av - \Cv \Bv^\dag \Cv^\top)\p{\Lv} \cge \zero$. 
\end{rem}

We remark that the last claim of \Cref{thm:local_global_q} does not follow from \Cref{thm:cc}:
\begin{eg}[\blue{\textbf{quadratic games can be non-convex}}]\label{eg:no_local_saddle}
Let $A = -1, C = 1, B = 0, a = b= 0$. Then, from \Cref{thm:local_global_q}  $(x,y) = (0,0)$ is local and global minimax. However, $q(x,y) = -\tfrac12 x^2 + xy$ is clearly non-convex in $x$ (although $\bar{q}$ is convex). Also, $(0, 0)$ is not local saddle since $q(x, 0)\geq q(0, 0)$ does not hold.
\end{eg}

\begin{restatable}[\blue{\textbf{equivalence between global and local minimax in quadratic games}}]{thm}{localglobalq}
\label{thm:qc}
An unconstrained quadratic game admits a global minimax point iff it admits a local minimax point iff 
\begin{align}
\label{eq:qc}
\Bv \preceq \zero, \quad 
\p{\Lv} (\Av - \Cv \Bv^\dag \Cv^\top) \p{\Lv} \succeq \zero, \mbox{ and }
\begin{bmatrix}
\av \\
\bv
\end{bmatrix} \in \rr\left(\begin{bmatrix}
\Av & \Cv \\
\Cv^\top & \Bv
\end{bmatrix}\right).
\end{align}
For such quadratic games, local minimax points are exactly the same as stationary global minimax points.
\end{restatable}

\noindent \blue{In this theorem we used $\Rc(\cdot)$ to denote the range of a matrix.} It is clear that stationary points, global minimax points, and local minimax points are characterized in the same way as in \Cref{thm:local_global_q}: we need only replace $\zero$ on the right-hands of \eqref{eq:stat} and \eqref{eq:gLmM} with the vector $[\av; \bv]$. These points always form an affine subspace for quadratic games.

\Cref{thm:qc} allows us to completely classify (unconstrained) quadratic games:
\begin{itemize}
\item there are no stationary points (hence no local or global minimax points);
\item there exist stationary points but no global or local minimax point;
\item there exist local minimax points which coincide with global minimax points;
\item there exist local minimax points which are strictly contained in global minimax points.%
\end{itemize}
Clearly, for homogeneous (unconstrained) quadratic games, stationary points always exist hence only the last three cases can happen. For (non-trivial) bilinear games, only the last case can happen:
\begin{cor}[\blue{\textbf{blinear games}}]\label{cor:bilinear}
For (homogeneous) unconstrained bilinear games ($\Av = \zero, \Bv = \zero, \Cv \neq \zero, \av = \zero, \bv = \zero$),  global minimax points are $\nn(\Cv^\top) \times \RR^n$ while local minimax points (i.e. stationary points) are 
$\nn(\Cv^\top) \times \nn(\Cv)$.
\end{cor}
\noindent It is thus clear that even in bilinear games, there exist global minimax points that are not local minimax. From \Cref{thm:qc}, we can derive that:

\begin{cor}[\blue{\textbf{saddle points in quadratic games}}]
\label{thm:quadr_2}
For (unconstrained) quadratic games, the following statements are equivalent:
\begin{enumerate}
    \item Local saddle points exist.
    \item Local maximin and minimax points exist.
    \item Global saddle points exist.
    \item Global maximin and minimax points exist.
    \item $\Av\succeq \zero \succeq \Bv,$ and
\be\label{eq:linear}
\begin{bmatrix}
\av \\
\bv
\end{bmatrix} \in \rr\left(\begin{bmatrix}
\Av & \Cv \\
\Cv^\top & \Bv
\end{bmatrix}
\right).
\en
\item stationary points exist and they are all local (global) saddle.
\end{enumerate}
\end{cor}

\noindent \blue{Note that we used $\Rc(\cdot)$ to denote the range of a matrix.} We remark that \Cref{thm:quadr_2} does not follow from typical minimax theorems (such as Sion's) since our domain is unbounded and we do not assume convexity-concavity from the outset. Thus, \Cref{thm:quadr_2} reveals strong duality under weaker assumptions than the usual convexity-concavity. This is in stark contrast with generic NCNC games (see \Cref{eg:nc}). 

\blue{
\begin{rem}[\textbf{non-uniformly local minimax in quadratic games}]\label{rem:non_uniform_minimax}
Since quadratic functions are continuous (and thus upper semi-continuous), from \Cref{prop:ulmm} we know that local saddle points are equivalent to uniformly minimax points. By comparing \Cref{thm:quadr_2} and \Cref{thm:qc}, whenever $\Av\succeq \zero \succeq \Bv$ and \eqref{eq:linear} holds, local saddle points and thus uniformly local minimax points exist. However, if \eqref{eq:qc} holds but $\Av \succeq \zero$ does not hold, local saddle points/uniformly local minimax points do not exist from \Cref{thm:quadr_2}, but local minimax points still exist from \Cref{thm:qc} which are hence non-uniform. We can see it more clearly from \Cref{eg:no_local_saddle}. One can compute $\oq_\e(x) = \e|x| - \tfrac{1}{2} x^2$, and obtain that $\oq_\e(x) \geq \oq_\e(0) = 0$ iff $|x| \leq 2\e$. According to \Cref{def:jin} the point $(0, 0)$ is non-uniformly local minimax. 
\end{rem}
}

\Cref{thm:quadr_2} reveals some fundamental and surprising properties of quadratic games. On the one hand, quadratic games consist of an important theoretical tool for understanding general smooth NCNC games (through local Taylor expansion) \citep[e.g.][]{ daskalakis2018limit,jin2019minmax, IbrahimAGM19, wang2019solving}; see also \Cref{sec:local_stab} below. On the other hand, they are really special and many of their unique properties do not carry over to general smooth NCNC games, as we demonstrate in the following examples:

\begin{eg}[\textbf{stationary/global minimax points exist, no local minimax points}]\label{eg:stationary_and_global_no_local}
For general NCNC games, the existence of a global minimax point may not imply the existence of local minimax points. Indeed, consider 
\be
f(x, y) = -y^4/4 + y^2/2 - xy, ~~ x \in \R, ~~ y \in \R.
\en
We claim $(\pm 1, 0)$ are the only global minimax points. Indeed, 
\be
\of(x) = \max_y -y^4/4 + y^2/2 - xy = \max_{y \geq 0} -y^4/4 + y^2/2 + |x| y \geq \max_{y \geq 0} -y^4/4 + y^2/2 = 1/4.\nonumber
\en
Clearly, the inequality is attained only at $x_* = 0$ and $y_* = \pm 1$. 
Its only stationary point is $(x, y) = (0, 0)$. However, $\partial_{yy}^2 f(0,0) = 1$ hence $y=0$ cannot be a local maximizer of $f(0, \cdot)$.

Note that in this example the global minimax points are not stationary. For an example where a stationary and global minimax point exists with no local minimax point, please refer to \Cref{eg:glbstatl}.
\end{eg}

\begin{eg}[\textbf{local minimax exists, no global minimax}]\label{eg:local_non_global}
This is possible even for separable functions, such as 
$f(x, y) = x^3 - x - y^2$ defined on $\R \times \R$. Clearly, it has a local minimax point at $(1/\sqrt{3}, 0)$ but no global minimax points exist. 
\end{eg}

\begin{eg}[\textbf{local minimax and local maximin points exist; no local saddle}]\label{eg:local_minimax_maximin_no_saddle}
We can also construct an example when both local minimax and local maximin points exist but there is no local saddle point. Take $f_1(x, y) = g(x, y)h(x, y)$, where $$g(x, y) = xy - x^2\mbox{, and }h(x, y) =
\exp\left(-\frac{1}{1-x^2}\right)\one_{|x|< 1}\exp\left(-\frac{1}{1-y^2}\right)\one_{|y|< 1}$$ is a bump function that smoothly interpolates between the unit box and the outside. By numerically computing the stationary points and checking the second-order conditions, we found there is no such a point where $\n_{xx}^2 f_1 \geq 0$ and $\n_{yy}^2 f_1 \leq 0$ in the open box $\Bc_1 = \{(x, y): |x| < 1, \, |y| < 1\}$. In other words, local saddle points do not exist. There is a local minimax point $(0, 0)$ since $$\of_\e(x)\geq (\e|x| - x^2)\exp(-1/(1 -x^2))\exp(-1/(1 -\e^2))\geq 0$$ when $|x|\leq \e$ and $\e^2 < 1$. Similarly we can construct $f_2(x, y) = -g(y-10, x-10)h(x-10, y-10)$ where there is a local maximin point but no local saddle point in the open box $\Bc_2 = \{(x, y): |x - 10| < 1, \, |y - 10| < 1\}$. Therefore, $f(x, y) = f_1(x, y) + f_2(x, y)$ has both local minimax and local maximin points, but there is no local saddle point on $\Bc_1\cup \Bc_2$.
\end{eg}

\begin{figure}
    \centering
    \includegraphics[width=12cm]{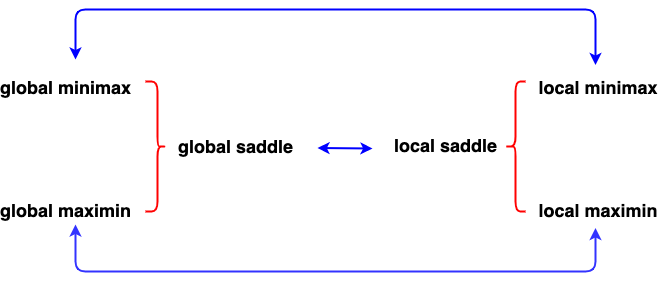}
    \caption{The relation among definitions in quadratic games. $A\longleftrightarrow B$ means $A$ exists iff $B$ exists. The brackets also show the existence relation. For example, global saddle points exist iff both global minimax and maximin points exist.}
    \label{fig:quad}
\end{figure}

\noindent Some special properties for quadratic games in this subsection are illustrated in \Cref{fig:quad}.


\section{Stability of gradient algorithms near local optimal points}\label{sec:local_stab}

In this section, we assume that $\X = \R^n$, $\Y = \R^m$ and that $f$ is twice continuously differentiable ($f\in \Cc^2$). From \eqref{eq:local_robust_point_stationary} we know that local minimax points are stationary points, and thus fixed points of gradient algorithms. We focus on \emph{local linear convergence} around stationary points using spectral analysis. Spectral analysis of a matrix $A$ mainly involves two types of quantities: the spectrum of $A$, $\sp(A) := \{\l : \l \mbox{ is an eigenvalue of }A\},$ as well as the spectral radius, $\rho(A) := \max_{\l \in \sp(A)} |\l|$. An iterative algorithm is \emph{exponentially stable} if the Jacobian matrix of its update function has a spectral radius of less than one, which guarantees local linear convergence \citep{polyak87}. A more rigorous definition uses the Hartman--Grobman theorem \citep[e.g.][]{katok1995introduction}. Below when we refer to convergence, we always mean local linear convergence.

To obtain convergence near local minimax points, we consider two-time-scale (2TS)\footnote{This terminology comes from analogy with the continuous training dynamics. In our paper we simply mean choosing two different step sizes. } gradient algorithms, as applied to GANs by \cite{heusel2017gans}. Also, \citet{jin2019minmax} proved the ``equivalence'' between the stable points of 2TS-GDA and strict local minimax points. The intuition is that 2TS algorithms help the convergence by taking a much larger step w.r.t. the variable $\yv$. We denote $\zv_t = (\xv_t, \yv_t)$ and define the vector field for the gradient update $$\vv({\zv}) = (-\a_1\n_{\xv}f(\zv), \a_2\n_{\yv}f(\zv)).$$ Local stability results can be obtained by analyzing the Jacobian of $\vv(\zv)$ at a stationary point $(\xv^*, \yv^*)$:
\be\label{eq:gen_H}
{\bf H}_{\a_1, \a_2} = {\bf H}_{\a_1, \a_2}(f) := \begin{bmatrix}
-\a_1 \n_{\xv \xv}^2 f & -\a_1 \n_{\xv\yv}^2 f \\
\a_2 \n_{\yv\xv}^2 f & \a_2 \n_{\yv \yv}^2 f
\end{bmatrix}.
\en
Define $\a_2 = \g \a_1$, and ${\bf H}_{\a_1, \a_2} = \a_1 {\bf H}_{1, \g}.$ Note that $\Hv_{\a_1, \a_2}(f)$ may not be symmetric, hence its spectrum lies on the complex plane. We also define $\Hv := \Hv_{\a, \a}/\a$ which is independent of $\a$. To characterize the stable set of an algorithm, we ask the following question:

\begin{quote}
Given hyper-parameters $\{\mu_i\}_{i=0}^k$ (e.g.~step size, momentum coefficient) of an algorithm $\mathsf{A}$, what exactly is the geometric characterization on the spectrum of ${\bf H}_{\a_1, \a_2}$ such that $\mathsf{A}$ is exponentially stable at $\zv^*$? 
\end{quote}
Similar questions have been asked in \citet{niethammer1983analysis} for problems of linear equations, where the Jacobian is a constant matrix. Such geometric characterizations allow us to analyze the convergence near local saddle and local minimax points.

Even with two-time-scale modification, GDA (even with momentum) does not converge near local saddle points for bilinear games \citep{zhang2019convergence}. Therefore, we will focus on extra gradient methods in this work. For completeness, thorough treatment of GDA, heavy ball (HB) and Nesterov's momentum (NAG) is included in \Cref{app:momentum_alg}. Note that second- and zeroth-order algorithms \citep{zhang2020newton, liu2019min} have also been considered very recently for minimax problems but they are beyond the scope of our work.

\blue{Note that in this section we are mostly considering one type of algorithmic modification in sequential games using two-time-scale (except in \Cref{thm:all}). For non-convex sequential smooth games, it is possible to use alternating updates in algorithms as studied in e.g.~\cite{zhang2019convergence} for bilinear games. We leave such systematic study to future work. }

\subsection{Stable sets of Extra-gradient (EG) and Optimistic gradient descent (OGD)}
We consider the generalized extra-gradient method {EG}$(\alpha_1, \a_2, \b)$ \citep{korpelevich1976extragradient} (the original version has $\b = 1$):
\be\label{eq:eg_alg}
\zv_{t+1} = \zv_t + \vv(\zv_{t+1/2})/\b, \, \zv_{t+1/2} = \zv_t + \vv(\zv_t).
\en

\noindent and the generalized optimistic gradient descent \citep{peng2019training} (denoted as {OGD}$(k, \a_1, \a_2)$): 
\be\label{eq:ogd}
\zv_{t+1} = \zv_t + k \vv(\zv_t) - \vv(\zv_{t-1}).\en
\blue{In \eqref{eq:eg_alg}, we call the first equation to be the extra-gradient step and the second equation to be the gradient step.}
EG was recently studied in e.g.~\citet[]{mertikopoulos2019optimistic} for special NCNC games, and in \cite{Azizian2019ATA, azizian2020accelerating} for convex-concave settings using spectral analysis. OGD was originally proposed in \cite{popov1980modification} as the past extra-gradient method, and was recently studied in the GAN literature \citep[e.g.][]{daskalakis2018training}. \citet{hsieh2019convergence, mokhtari2019proximal} showed a close connection between EG and OGD:
\begin{restatable}[\blue{\textbf{equivalence between past extra-gradient and OGD}}]{lem}{peg}\label{lem:peg}
The past extra-gradient method 
\be\label{eq:peg}
\zv_{t+1} = \zv_t + \vv(\zv_{t+1/2})/\b, \, \zv_{t+1/2} = \zv_t + \vv(\zv_{t-1/2})
\en
can be rewritten as $\zv'_{t+1} = \zv'_t + k \vv(\zv'_t) - \vv(\zv'_{t-1})$ with $k = 1 + 1/\beta$ and $\zv'_t = \zv_{t-1/2}$.
\end{restatable}

Due to this correspondence, we will only consider OGD with $k > 1$. We now characterize the stable sets of EG and OGD, or the \emph{necessary and sufficient conditions} for local convergence (see the proof in \Cref{app:proof_stability}):

\begin{restatable}[\blue{\textbf{stability of EG/OGD}}]{thm}{thmeg}\label{thm:eg_stable}
At $(\xv^*, \yv^*)$, EG$(\a_1, \a_2, \b)$ is exponentially stable iff for any $\l \in \sp(\Hv_{\a_1, \a_2})$, 
$|1 + \l/\b + \l^2/\b| < 1.$ ${\rm OGD}(k, \a_1, \a_2)$ is exponentially stable iff for any $\l \in  \sp(\Hv_{\a_1, \a_2})$, $|\l| < 1$ and 
$|\l|^2 (k - 3 + (k + 1)|\l|^2) < 2\Re(\l)(k|\l|^2 - 1).$
\end{restatable}

\blue{In this theorem, $\Re(\cdot)$ represents the real part of a complex number.} From this theorem, we can plot the stable region of EG and OGD with the original parameters \blue{($\beta = 1$ and $k = 2$)}, and find that EG and OGD are indeed similar, as shown on the right of \Cref{fig:eg_ogd_saddle}. For EG, we note that \cite{ azizian2020accelerating} used the spectral shapes of the support of $\Sp(\Hv_{\a_1, \a_2})$ to give upper and lower bounds of the convergence rates of EG, but our results are orthogonal to it since we do not assume a geometric shape of the support of $\sp(\Hv_{\a_1, \a_2})$. 

\begin{figure}
    \centering
    \includegraphics[width=0.9\textwidth]{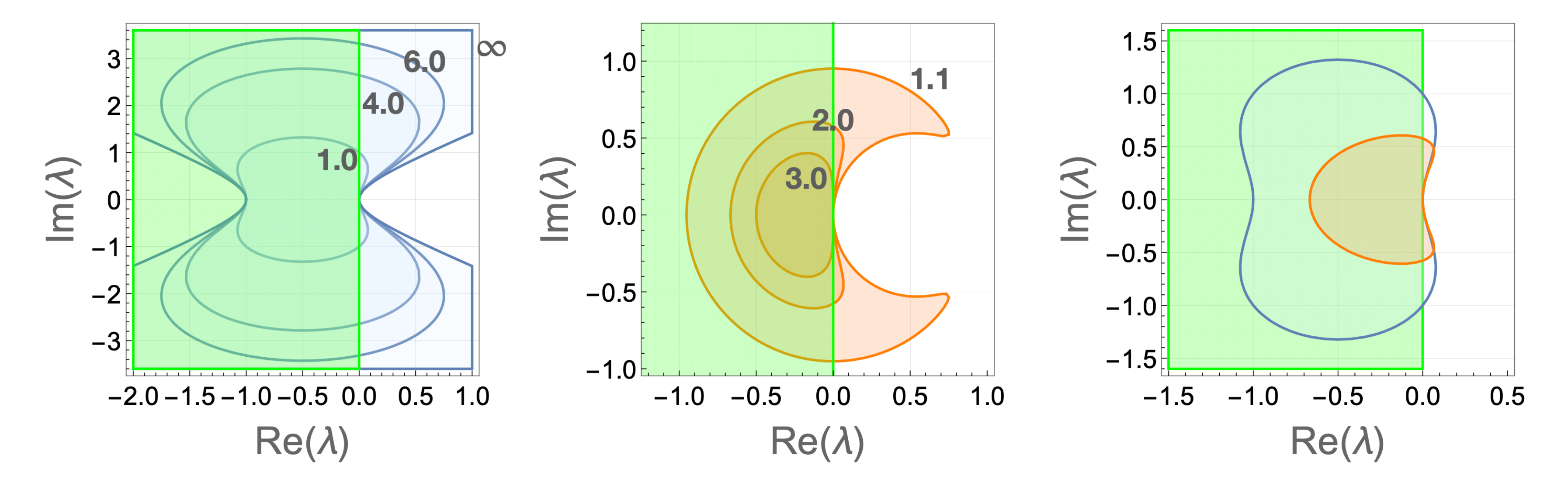}
    \caption{\blue{The blue/orange regions are where EG/OGD are exponentially stable. The green region represents where the eigenvalues of $\sp({\bf H}_{\a_1, \a_2})$ at local saddle points may occur. ({\bf left}) EG$(\a_1, \a_2, \b)$ with $\b \in \{1.0, 4.0, 6.0, \infty\}$; ({\bf middle}) OGD$(k, \a_1, \a_2)$ with $k \in \{1 + {1}/{10}, 1 + {1}/{1}, 1 + {1}/{0.5}\}$. ({\bf right}) Comparison between EG$(\a_1, \a_2, 1)$ ({\bf blue}) and OGD$(2, \a_1, \a_2)$ ({\bf orange}). Best viewed in color. }}
    \label{fig:eg_ogd_saddle}
\end{figure}

\blue{
When $\beta\to \infty$, we have $k\to 1_+$, and the step size of extra-gradient step is much larger than the step size of the gradient step. A similar conclusion can found in Theorem 4.1 of \citet[][]{zhang2019convergence},\footnote{Note that the exact definitions of $\beta$ are different. Suppose the gradient step sizes are $\alpha_1 = \alpha_2 = \alpha$, and the extra-gradient step sizes are $\gamma_1 = \gamma_2 = \gamma$. Our definition gives $\beta = \alpha/\gamma$ while \citet{zhang2019convergence} gives $\beta = \alpha \gamma$. } which states that for bilinear games, taking very small gradient steps and very large extra-gradient steps gives the best convergence rate among all hyper-parameter choices of gradient and extra-gradient steps.
}

\blue{
Moreover, we show that larger $\b$ increases the local stability as well (see also Prop.~1', \citet[][]{hsieh2020explore} for a similar conclusion in saddle point problems, where $\beta$ corresponds to $\gamma_t/\eta_t$). The proof of the following theorem can be found in \Cref{app:proof_stability}:
}
\vspace{-0.2em}
\begin{restatable}[\blue{\textbf{more aggressive extra-gradient steps, more stable}}]{thm}{egratio}\label{thm:egratio}
For $\b_1 > \b_2 > 1$, whenever ${\rm EG}(\a_1, \a_2, \b_2)$ is exponentially stable at $(\xv^*, \yv^*)$, ${\rm EG}(\a_1, \a_2, \b_1)$ is exponentially stable at $(\xv^*, \yv^*)$ as well. For $k_1 > k_2 > 1$, whenever OGD$(k_1, \a_1, \a_2)$ is exponentially stable at $(\xv^*, \yv^*)$, OGD$(k_2, \a_1, \a_2)$ is exponentially stable at $(\xv^*, \yv^*)$ as well.
\end{restatable}
\vspace{-0.2em}
In the limit when $\b\to \infty$, the stable region is $\Re(\l + \l^2) < 0$ whose boundary is a hyperbola.  Similarly, when $k \to 1_+$, OGD has the largest convergence region: $\{\l\in \mathbb{C}: |\l| < 1, \, |\l -1/2| > 1/2\}$. \Cref{fig:eg_ogd_saddle} visualizes the stable sets of EG/OGD. Their convergence regions strictly include that of GDA, and thus these algorithms are more stable:
\vspace{-0.2em}
 \begin{restatable}[\blue{\textbf{EG/OGD are more stable than GDA}}]{cor}{corgdeg}\label{thm:gd_eg}
When the step sizes $\a_1, \a_2$ are small enough, whenever GDA converges, EG and OGD converge as well.
\end{restatable}
\vspace{-0.2em}
\noindent The formal version of \Cref{thm:gd_eg} can be found in \Cref{thm:gd_eg_2}. 

\subsection{Local convergence to local optimal points}\label{sec:local_conv_glp}

After characterizing the stable sets of EG and OGD, we move on to see the spectral behavior of local optimal points. For local saddle points, the spectrum of $\Hv_{\a_1, \a_2}$ is on the left closed half plane. However, the spectrum of local minimax points (and thus LRPs, see \Cref{app:local_robust_point}) can be quite arbitrary. With these results we can study how gradient algorithms (GDA with momentum, EG/OGD) converge to local optimal points.

\subsubsection{Local saddle points}

Even though the matrix $\Hv_{\a_1, \a_2}(f)$ is not symmetric, it is still negative semi-definite near local saddle points.\footnote{\blue{A real $n\times n$ matrix $\Av$ is negative semi-definite if for any $\xv\in \R^n$, $\xv^\top \Av \xv \leq 0$, i.e.~$\Av + \Av^\top$ is symmetric and negative semi-definite.}} Therefore, we can prove that its spectrum lies on the left (closed) complex plane:

\begin{restatable}[\textbf{local saddle}]{lem}{lemlocsadl}\label{lem:loc_sadl}
Suppose $\a_1, \a_2  > 0$ are fixed. For $f\in \Cc^2$, at a local saddle point, \blue{for all $\l \in \sp(\Hv_{\a_1, \a_2}(f))$, we have $\Re(\l) \leq 0$. For all $z\in \mathds{C}$ with $\Re(z)\leq 0$, there exists a quadratic function $q$ and a local saddle point $(\xv^*, \yv^*)$ such that~$z\in \sp(\Hv_{\a_1, \a_2}(q))$.} For bilinear functions, at a local saddle point we have $\Re(\l) = 0$ for all $\l \in \sp(\Hv_{\a_1, \a_2})$. 
\end{restatable}

This result is a slight extension of Lemma 2.4 in \citet{daskalakis2018limit}. \blue{Combined with \Cref{lem:loc_sadl}, we can show that EG converges around any local saddle point where the Jacobian $\Hv(f)$ is non-singular, and a similar result holds for OGD if $k$ is in a certain range:}
\begin{restatable}[\blue{\textbf{stability of EG/OGD at local saddle points}}]{thm}{coregsadl}\label{cor:eg_sadl}
\blue{EG$(\a, \a, 1)$ is exponentially stable at any local saddle point if at such a point, $0 < |\l| < 1/\alpha$ for every $\l \in \sp({\bf H})$.
${\rm OGD}(k, \a, \a)$ is exponentially stable at any local saddle point
if $1 < k \leq 2$ and $0 < |\l| < 1/(k\alpha)$ for every $\l \in \sp({\bf H})$. If $k\geq 3$, OGD$(k, \a_1, \a_2)$ is not exponentially stable for bilinear games. }
\end{restatable}

\noindent \blue{Given a fixed non-singular Jacobian matrix, we can always choose $\alpha$ to be small enough, such that $0 < |\lambda| < 1/\alpha$ (or $0 < |\lambda| < 1/(k\alpha)$) for any $\lambda \in \Sp(\Hv)$. Therefore, EG and OGD always locally converge to any local saddle point as long as $\Hv(f)$ is non-singular.}

\subsubsection{Local minimax points}

Now we study how gradient algorithms converge to local minimax points. \blue{We do not have the results in \Cref{cor:eg_sadl}, since different from local saddle points, the spectrum of the Jacobian $\Hv_{\a_1, \a_2}(f)$ is quite arbitrary:}
\begin{restatable}[\textbf{\blue{spectrum of local minimax can be arbitrary}}]{lem}{lemloc}\label{lem:loc_min_max}
Given $\a_1, \a_2 > 0$, for any $z\in \mathds{C}$, there exists a quadratic function $q$ and a local minimax point $(\xv^*, \yv^*)$ where $z\in \sp(\Hv_{\a_1, \a_2}(q))$.
\end{restatable}

\blue{This result shows that local minimax points are a more general class than the class of local stable stationary points (LSSPs) as studied recently in \citet{Berard2020A}, in terms of zero-sum games, since LSSPs are defined such that $\Re(\lambda) < 0$ for any $\lambda \in \Sp(\Hv_{\a, \a})$ and $\a > 0$ (note the slight change of signs due to the difference of notations).}
Under certain assumptions, 2TS gradient algorithms can converge to local minimax points. The following result slightly extends \citet{jin2019minmax} where only GDA is analyzed:
\begin{restatable}[\blue{\textbf{stability of EG/OGD at strict local minimax points}}]{thm}{thmnegy}\label{thm:neg_y}
Assume at a stationary point $(\xv^*, \yv^*)$, \be
\label{eq:strict_local}\n_{\yv \yv}^2 f \prec {\bf 0}\mbox{ and }\n_{\xv \xv}^2 f - \n_{\xv \yv}^2 f (\n_{\yv \yv}^2 f)^{-1}\n_{\yv \xv}^2 f \succ {\bf 0}.\en Then \blue{there exist} $\gamma_0 > 0$ and $\a_0 > 0$ such that for any $\gamma > \gamma_0 , 0 < \a_2 < \a_0$ and $\a_1 = \a_2/\g$, EG and OGD (with $k> 1$) are exponentially stable.
\end{restatable}

\blue{In fact, the theorem above can be extended to momentum methods as well (see \Cref{app:momentum_alg}).} As we have seen in \Cref{cor:suf_jin}, \eqref{eq:strict_local} is sufficient for being local minimax (see also \cite{fiez2019convergence, wang2019solving, zhang2020newton} for applications in GANs).
However, without assumption \eqref{eq:strict_local} (see also \citet[Theorem 28]{jin2019minmax} for GDA), convergence is more difficult:

\begin{restatable}[\blue{\textbf{stability of gradient algorithms at general local minimax points}}]{prop}{thmall}\label{thm:all}
There exists a quadratic function (e.g., $q(x, y) = -x^2 + xy$) and a global (thus local, from \Cref{thm:qc}) minimax point $\zv^* = (\xv^*, \yv^*)$ where
\begin{itemize}[topsep=2pt, itemsep=0pt]
    \item GDA (with momentum \blue{or alternating updates}) does not converge to $\zv^*$, for any hyper-parameter choice.
    \item \blue{If $\a_1 = \a_2$, or $\a_2 \to 0$, EG/OGD do not converge to $\zv^*$. Otherwise there exist hyper-parameter choices such that EG/OGD converge to $\zv^*$.}
    \blue{\item Alternating OGD does not converge to $\zv^*$ given $\alpha_2 \to 0$. }
\end{itemize} 
\end{restatable}

\blue{
The exact forms of alternating updates can be found in \cite{zhang2019convergence} which we have also included in the proof of \Cref{thm:all}. It basically says that we update $\xv$ and $\yv$ one after the other rather than simultaneously.
\Cref{thm:all} extends \citet{jin2019minmax} by studying the degenerate case of $\yyv f$ and gradient algorithms other than GDA. The implication is two-fold:
\begin{itemize}[topsep=5pt, parsep=0pt]
\item On the algorithmic aspect, we may not always rely on the usual ODE analysis \citep{mescheder2017numerics, MertikopoulosPP18, fiez2019convergence} when trying to find global/local minimax points, as such analysis relies on approximating gradient algorithms with their continuous versions, by taking the step sizes to be arbitrarily small. For EG/OGD, the step size of the follower ($\a_2$) has to be large while the step size of the leader can be arbitrarily small, reflecting the asymmetric position of players in Stackelberg games \citep{jin2019minmax}. 
\item We may also need new solution concepts in addition to global/local minimax points in machine learning applications \citep[e.g.][]{farnia2020gans, schaefer2020implicit}, even though many machine learning applications, including GANs \citep{goodfellow2014generative} and adversarial training \citep{madry2017towards} are essentially based on the notion of global minimax points. This is because when applying standard gradient-based algorithms to do a local search in machine learning applications, we cannot always expect the final solutions found by the algorithms to cover all global/local minimax points. 
\end{itemize}
}

 
\section{Conclusion}

\blue{
The aim of this work is to provide a comprehensive study of the recently proposed local minimax points \citep{jin2019minmax}. We discussed the relations between local saddle and local minimax points, between local and global minimax points, and interpreted local minimax points based on infinitesimal robustness. This new interpretation allows us to further generalize local minimax points such that they are still stationary (\Cref{prop:stable_new}). We presented the first- and second-order optimality conditions of these local optimal solutions, which extend \citet{jin2019minmax} to the constrained and degenerate cases. Specifically, in (potentially non-convex) quadratic games, local minimax points are (in some sense) equivalent to global minimax points. We also studied the stability of popular gradient algorithms near local optimal solutions, which provides insights for the design of algorithms to find minimax points. \\
The implication of this work is two-fold: \textbf{(a)} we may need new algorithms for smooth games, since we have shown in \Cref{thm:all} that our common intuition might fail w.r.t.~the convergence to a local and global minimax point;  \textbf{(b)} we need to think about new solution concepts other than global/local minimax points. As many theoretical works aim to go beyond the definition of Nash equilibria (a.k.a.~saddle points) such as \cite{jin2019minmax, farnia2020gans, Berard2020A}, to name a few, we may need to take one step further, beyond the definition of Stackelberg equilibria (a.k.a~minimax points), as also pointed out in \citet{schaefer2020implicit}. Our new definition of local robust points sheds some light on going beyond Stackelberg games (\Cref{app:local_robust_point}). 
}

\acks{We thank NSERC, the Canada CIFAR AI Chairs Program, Borealis AI and the Waterloo-Huawei Joint Innovation Lab for financial support. GZ is also supported by David R. Cheriton scholarship and Vector research grant. We thank Chi Jin and Oliver Schulte for useful discussion.}

\appendix



\section{Non-smooth analysis: A short detour}\label{app:review}

We give a short detour on some classical optimality conditions in non-smooth optimization. These results will be used in \Cref{sec:local_opt} to yield necessary and sufficient conditions for local optimality in zero-sum two-player games, since the optimality conditions for local optimal points can be reduced to those for the envelope functions, which are in general non-smooth. A more thorough version of this appendix can be found in \cite{zhang2020optimality}.


Let $\h$ be a function
defined on some set $\Xc\subseteq \RR^m$. Its upper and lower (Dini) directional derivatives are defined as:
\begin{align}
\Ds\h(\xv; \dv) &:= \limsup_{t \to 0^+} \frac{\h(\xv+ t\dv) - \h(\xv)}{t}
,\,
\Di\h(\xv; \dv) := \liminf_{t \to 0^+} \frac{\h(\xv+ t\dv) - \h(\xv)}{t}
.
\end{align}
When the two limits coincide, we use the notation $\D\h(\xv; \dv)$ and call the function $\h$ directionally differentiable (at $\xv$ along direction $\dv$). We can similarly define the upper and lower second-order directional derivatives\footnote{A popular directional derivative in non-smooth analysis, due to \citet{Clarke90}, is to replace $\h(\xv+ t\dv)$ with $\h(\yv+t\dv)$ for some sequence $\yv \to \xv$. The second-order counterpart appeared in \citet{CominettiCorrea90}. For our purpose here, the classical Dini definitions suffice.} according to  \citet{ben1982necessary}:
\begin{align}
\Hs\h(\xv; \dv, \gv) &= \limsup_{t \to 0^+} \frac{\h(\xv+ t \dv + t^2 \gv/2) - \h(\xv)- t \cdot \D\h(\xv; \dv)}{t^2/2}, \\
\Hi\h(\xv; \dv, \gv) &= \liminf_{t \to 0^+} \frac{\h(\xv+ t \dv + t^2 \gv/2) - \h(\xv)- t \cdot \D\h(\xv; \dv)}{t^2/2}.
\end{align}
Similarly, when the two limits coincide we use the simplified notation $\He\h(\xv; \dv, \gv)$ and call $\h$ twice directionally differentiable (at $\xv$ along parabolic $(\dv, \gv)$). Note that, when $\dv = \zero$, we recover the directional derivative: 
\begin{align}
\Hs\h(\xv; \zero, \gv) = \Hi\h(\xv; \zero, \gv) = \D\h(\xv; \gv),
\end{align}
while if $\gv = \zero$, 
\begin{align}
\Hs\h(\xv; \dv) := \Hs\h(\xv; \dv, \zero), ~ \Hi\h(\xv; \dv) := \Hi\h(\xv; \dv, \zero), ~\He\h(\xv; \dv) := \He\h(\xv; \dv, \zero)
\end{align}
reduces to the second-order directional derivatives of \citet{Demyanov73}. 
The advantage of the definition of \citet{ben1982necessary} is evidenced in the following chain rule:
\begin{thm}[\citealt{ben1982necessary}]
\label{thm:chain}
Let $\h: \R^m \to \R$ be locally Lipschitz and $k: \R^n \to \R^m$ be (twice) directionally differentiable. Then, 
\begin{align}
\Ds (\h\circ k)(\xv; \dv) &= \Ds\h\big( k(\xv); \D k(\xv; \dv) \big),  \\
\Hs (\h \circ k)(\xv; \dv, \gv) &= \Hs\h\big(k(\xv); \D k(\xv; \dv), \He k(\xv; \dv, \gv) \big).
\end{align}
(The same result holds for the lower derivatives, and hence the derivatives when they exist.)
\end{thm}
\noindent In contrast, the definition of \citet{Demyanov73} fails to satisfy the chain rule above. Indeed, if $\h$ is differentiable, then 
\begin{align}
\D\h(\xv; \dv) = \inner{\nabla\h(\xv)}{\dv}
\end{align}
while if $\h$ is twice differentiable, then
\begin{align}
\He\h(\xv; \dv, \gv) = \D\h(\xv; \gv) + \He\h(\xv; \dv) = \inner{\nabla\h(\xv)}{\gv} + \inner{\dv}{\nabla^2 \h(\xv)\dv},
\end{align}
where $\nabla \h$ and $\nabla^2 \h$ are the gradient and Hessian of $\h$, respectively. (A slightly more general setting is discussed in \citealt[Proposition 1.1]{Seeger88}.)
The following properties of the directional derivatives are clear:
\begin{thm}
\label{thm:prop_dd}
For any $\lambda \geq 0$ we have 
\begin{align}
\D\h(\xv; \lambda\dv) &= \lambda \cdot\D\h(\xv; \dv), \\  
\He\h(\xv; \lambda\dv, \lambda^2\gv) &= \lambda^2 \cdot\He\h(\xv; \dv, \gv)
\end{align}
If $\h$ is locally Lipschitz around $\xv$, then $\D\h(\xv; \cdot)$ and $\He\h(\xv; \dv, \cdot)$ are Lipschitz continuous. (Similar results hold for the upper and lower derivatives.)  
\end{thm}

\subsection{Necessary conditions}\label{app:1st-necessary}
Consider the non-smooth optimization problem
\begin{align}
\min_{\xv \in \X \subseteq \R^m} ~ \h(\xv).
\end{align}
We define three tangent cones of the (closed) constraint set $\X$:
\begin{align}
\K[f](\X, \xv) &:= \{\dv: \forall \{t_k\} \to 0^+~ \exists \{t_{k_i}\} \to 0^+, \xv+t_{k_i} \dv \in \X \} \subseteq \mathrm{cone}(\X - \xv)\\
\K[d](\X, \xv) &:= \liminf_{t\to 0^+} \frac{\X - \xv}{t} := \{\dv: \forall \{t_k\} \to 0^+~ \exists \{t_{k_i}\} \to 0^+, \{\dv_{k_i}\} \to \dv, \xv+t_{k_i} \dv_{k_i} \in \X \} 
\\
\K[c](\X, \xv) &:= \limsup_{t\to 0^+} \frac{\X - \xv}{t} := \{\dv: \exists \{t_k\} \to 0^+, ~\{\dv_{k}\} \to \dv, \xv+t_{k} \dv_{k} \in \X \}.
\end{align}
Obviously, the (feasible) cone $\K[f]$ is contained in the (derivable) cone $\K[d]$, which is itself contained in the (contingent) cone $\K[c]$. 
$\K[d]$ and $\K[c]$ are always closed while $\K[f]$ may not be so (even when $\X$ is closed). 
On the other hand, if $\X$ is convex (and $\xv \in \X$), then all three tangent cones are convex, $\K[f] = \mathrm{cone}(\X - \xv)$ and $\K[d]=\K[c] = \overline{\K[f]}$. 
Note that for all tangent cones, we have
\begin{align}
\forall \xv \not\in \bar\X, ~ \K(\X, \xv) = \emptyset, \mbox{ and } \forall \xv \in \X^\circ, \K(\X, \xv) = \R^m,
\end{align}
where $\bar\X$ and $\X^\circ$ denote the closure and interior of $\X$, respectively. The following necessary condition is well-known:
\begin{thm}[first-order necessary condition, e.g.~\citet{Demyanov66}]
\label{thm:nec_1st}
Let $\xv^*$ be a local minimizer of $\h$ over $\X$. Then, 
\begin{align}
\forall \dv \in \K[f](\X, \xv^*), &~~ \Di\h(\xv^*; \dv) \geq 0.
\end{align}
The converse is also true if $\h$ and $\X$ are both convex around $\xv^*$. If $\h$ is locally Lipschitz, then 
\begin{align}
\forall \dv \in \K[d](\X, \xv^*), &~~ \Di\h(\xv^*; \dv) \geq 0.
\end{align}
\end{thm}
\begin{proof}
We first prove the converse part. Suppose to the contrary there exists $\xv$ around $\xv^*$ so that $\h(\xv) < \h(\xv^*)$. Then, $\dv = \xv - \xv^* \in \K[f](\X, \xv^*)$ and we have
\begin{align}
\Di\h(\xv^*; \dv) = \liminf_{t \to 0^+} \frac{\h((1-t)\xv^*+t\xv) - \h(\xv^*)}{t} \leq \h(\xv) - \h(\xv^*) < 0,
\end{align}
which is a contradiction.

To see the claim when $\h$ is locally Lipschitz, note that $\dv \in \K[d](\X, \xv^*)$ implies for any $\{t_k\} \to 0$ there exist $\{t_{k_i}\} \to 0^+$ and $\{\dv_{k_i}\} \to \dv$ such that $\xv^*+t_{k_i} \dv_{k_i} \in \X$. For sufficiently large $k_i$ we have $\h(\xv^*+t_{k_i} \dv_{k_i}) \geq \h(\xv^*)$ since $\xv^*$ by assumption is a local minimizer. Thus, 
\begin{align}
\liminf_{t \to 0^+} \frac{\h(\xv^*+t \dv) - \h(\xv^*)}{ t} &:=  \lim_{t_k \to 0^+} \frac{\h(\xv^*+t_k \dv) - \h(\xv^*)}{t_k} \\
\nonumber&\geq \limsup_{t_{k_i} \to 0^+} \frac{\h(\xv^*+ t_{k_i} \dv_{k_i}) - \h(\xv^*)}{ t_{k_i} } -\tr &- \limsup_{t_{k_i} \to 0^+} \frac{\h(\xv^*+ t_{k_i} \dv) - \h(\xv^*+ t_{k_i} \dv_{k_i})}{t_{k_i}} \\
& \geq 0 - 0 = 0.
\end{align}

The proof for a general function $\h$ is similar. 
\end{proof}

To derive second-order conditions, we define similarly the second-order tangent cones:
\begin{align}
\K[f](\X, \xv; \dv) &:= \{\gv: \forall \{t_k\} \downarrow 0~ \exists \{t_{k_i}\} \downarrow 0, \xv+t_{k_i} \dv +  t_{k_i}^2 \gv/2 \in \X \}, \\
\K[d](\X, \xv; \dv) &:= \liminf_{t\to 0^+} \frac{\X - \xv - t \dv}{t^2/2} \tr &:= \{\gv: \forall \{t_k\} \downarrow 0~ \exists \{t_{k_i}\} \downarrow 0, \{\gv_{k_i}\} \to \gv, \xv+t_{k_i} \dv + t_{k_i}^2 \gv_{k_i}/2 \in \X \}. 
\end{align}
The proof of the following result is completely similar to that of \Cref{thm:nec_1st}:
\begin{thm}[second-order necessary condition, e.g. \citealt{BenTalZowe85}]
\label{thm:necessary_2}
Let $\h$ be directionally differentiable and $\xv^*$ be a local minimizer of $\h$ over $\X$. Then, 
\begin{align}
\forall \dv\in \K[f](\X, \xv^*), \forall \gv \in \K[f](\X, \xv^*; \dv), ~~ \D\h(\xv^*; \dv) = 0 \implies
\Hi\h(\xv^*; \dv, \gv)\geq 0.
\end{align}
If $\h$ is locally Lipschitz, then 
\begin{align}
\forall \dv\in \K[d](\X, \xv^*), \forall \gv \in \K[d](\X, \xv^*; \dv), ~~ \D\h(\xv^*; \dv) = 0 \implies
\Hi\h(\xv^*; \dv, \gv)\geq 0.
\end{align}
\end{thm}

\subsection{Sufficient conditions}
We give sufficient conditions for a non-smooth function to attain an isolated minimum. 

\begin{thm}[first-order, e.g. \citealt{Demyanov70, BenTalZowe85}]
\label{thm:sufficient}
Let $\h$ be locally Lipschitz. If 
\begin{align}\label{eq:first}
\forall \zero \ne \dv \in \K[c](\X, \xv^*), ~ \Di\h(\xv^*; \dv) > 0,
\end{align}
then $\xv^*$ is an isolated local minimum of $\h$ over $\X$.
\end{thm}
\begin{proof}
Suppose to the contrary there exists a sequence $\xv_k\in \X$ converging to $\xv^*$ so that $\h(\xv_k) \leq \h(\xv^*)$. Let $t_k := \|\xv_k - \xv^*\|$ and $ \dv_k := (\xv_k - \xv^*) / \|\xv_k - \xv^*\|$. By passing to a subsequence we may assume $\dv_k \to \dv \ne \zero$, where clearly $\dv \in \K[c](\X, \xv^*)$ since $\xv^*+t_k \dv_k = \xv_k \in \X$. But then
\begin{align}
\Di\h(\xv^*; \dv) &\leq \liminf_{t_k \to 0^+} \frac{\h(\xv^*+t_k \dv) - \h(\xv^*)}{t_k} \\
&\leq  \liminf_{t_k \to 0^+} \frac{\h(\xv^*+t_k \dv_k) - \h(\xv^*)}{t_k} + \limsup_{t_k \to 0^+} \frac{\h(\xv^*+t_k \dv) - \h(\xv^*+t_k \dv_k)}{t_k} \\
&\leq 0 + 0 = 0,
\end{align}
arriving at a contradiction.
\end{proof}
\noindent Note that when $\X$ is convex, we may replace $\K[c]=\overline{\K[f]}$ with $\K[f]$ (recall the Lipschitz continuity in \Cref{thm:prop_dd}).

\begin{thm}[second-order, e.g. \citealt{Demyanov70}]
\label{thm:2nds}
Let $\h$ be locally Lipschitz and directional differentiable, and $\X$ be convex. If 
\begin{enumerate}
\item 
$\forall \dv \in \K[f](\X, \xv^*), ~ \D\h(\xv^*; \dv) \geq 0$,
\item 
$\exists \gamma > 0$ such that for all $\dv \in \K[f](\X, \xv^*), \|\dv \|=1,  \D\h(\xv^*; \dv) \in [0, \gamma]$ we have for all small $t$ and \emph{uniformly} on bounded sets in $\dv$:
\begin{align}
\label{eq:Ah}
\frac{\h(\xv^*+t\dv) - \h(\xv^*) - t \D\h(\xv^*;\dv)}{t^2/2} \geq \Af_\h(\xv^*; \dv) > 0,
\end{align}
\end{enumerate}
then $\xv^*$ is an isolated local minimum of $\h$ over $\X$.
\end{thm}
\begin{proof}
Let $\xv \in \X$ and $\xv \neq \xv^*$, then $\dv := (\xv - \xv^*)/\|\xv - \xv^*\| \in \K[f](\X, \xv^*)$ (since $\X$ is convex). Suppose $\D\h(\xv^*, \dv) \geq \gamma > 0$, then  
\begin{align}
\h(\xv^* + t \dv) = \h(\xv^*) + t \D\h(\xv^*; \dv) + o(t) \geq \h(\xv^*) + \gamma t + o(t) > \h(\xv^*) + \gamma t / 2,
\end{align}
for sufficiently small $t \leq t_{\dv}$. 
Since the function $\dv \mapsto \h(\xv^*+t \dv)$ is locally Lipschitz, we may choose a non-empty open subset from each set $\{\vv: \forall t \in(0, t_{\dv}], ~\h(\xv^* + t \vv)  > \h(\xv^*)\}$. Hence, using a standard compactness argument, we know for all small positive $t$, 
\begin{align}
\dv \in \K[f](\X, \xv^*), \|\dv\| = 1, \D\h(\xv^*, \dv) \geq \gamma \implies \h(\xv^* + t \dv) > \h(\xv^*).
\end{align}
Suppose instead $\D\h(\xv^*, \dv) \in [0, \gamma]$, then for all small positive $t$ and uniformly in $\dv$ we have
\begin{align}
\h(\xv^* + t \dv) & \geq \h(\xv^*) + t \D\h(\xv^*; \dv) + \tfrac12 t^2 \Af_\h(\xv^*;\dv)\\
& \geq \h(\xv^*) + \tfrac12 t^2 \Af_\h(\xv^*;\dv) \\
& > \h(\xv^*).
\end{align}
Finally, combining the above two cases completes the proof.
\end{proof}
\noindent We make a few remarks regarding \Cref{thm:2nds}:
\begin{itemize}
\item In general we cannot let $\gamma = 0$ (for an explicit counterexample, see \citealt{Demyanov70}). This is one of the subtleties to work with directional derivatives: even when $\D\h(\xv^*; \dv)$ vanishes for some direction $\dv$ we may still have $\D\h(\xv^*; \dv)$ approaching 0 for other directions, but with $\gamma = 0$ we will not know how $\Af_h(\xv^*;\dv)$ behaves (e.g. negative) along the latter directions.
\item It is clear that $\Hi\h \geq \Af_\h$. In some cases it is easier to verify the uniformity (along directions) in \eqref{eq:Ah} if we relax the lower 2nd-order directional derivative $\Hi\h$ to some convenient function $\Af_h$. See \Cref{thm:env_2nds_q} for an example.
\item If $\X = \R^m$ and $\h$ is Fr\'{e}chet differentiable with locally Lipschitz gradient $\nabla\h$ around $\xv^*$, then we can verify the uniformity in \eqref{eq:Ah} as follows. Note first that we have $\nabla \h(\xv^*) = \zero$ from the necessary condition. Second, for all small $t$ we have
\begin{align}
\frac{\h(\xv^*+t\bar\dv) - \h(\xv^*)}{t^2/2} &= \frac{\h(\xv^*+t\dv + t (\bar\dv - \dv)) - \h(\xv^*)}{t^2/2} \\
&= \frac{\h(\xv^*+t\dv) - \h(\xv^*) + t \inner{\nabla\h(\xv^*+\theta t\dv) - \nabla\h(\xv^*)}{\bar\dv - \dv} }{t^2/2} \\
&\geq \frac{\h(\xv^*+t\dv) - \h(\xv^*)}{t^2/2} - 2L\|\dv\|\|\bar\dv-\dv\|,
\end{align}
where $\theta\in[0,1]$ and $L$ is the local Lipschitz constant of $\nabla\h$. Thus, if $\frac{\h(\xv^*+t\dv) - \h(\xv^*)}{t^2/2} > 0$ then for all nearby $\bar\dv$ we also have $\frac{\h(\xv^*+t\bar\dv) - \h(\xv^*)}{t^2/2} > 0$. In this case we may let $\Af_\h = \Hi\h$ and recover \citep[Theorem 3.2]{BenTalZowe85}.
\end{itemize}

Another result that directly uses the second-order derivative is:
\begin{thm}[{second-order sufficient condition, e.g. \citealt{DemyanovMalozemov72}}]\label{thm:2nd_suff_gen}
Suppose $h$ is uniformly first-order and second-order directional differentiable (at $\xv^*$) and $\X$ is convex. If there exist $r, q > 0$ such that for all normalized feasible direction $\tv$, $\D h(\xv^*; \tv) \geq 0$, and
\begin{align}\label{eq:second}
0 \leq \D h(\xv^*; \tv)  < r \Longrightarrow \He {h}(\xv^*; \tv) \geq q > 0,
\end{align}
then $\xv^*$ is an isolated local minimum.
\end{thm}
\begin{proof}
If $\D h(\xv^*; \tv) \geq r$, it reduces to the proof of Thm.~\ref{thm:sufficient}. Otherwise, \eqref{eq:second} holds, and
\be
h(\xv^* + \a \tv) = h(\xv^*) + \a \D h(\xv^*; \tv) + \frac{\a^2}{2}\He {h}(\xv^*; \tv)+ o(\a^2; \tv).
\en
Since $h$ is uniformly second-order directional differentiable in any direction $\tv$, there exist $0 < \a_1 < \a_0$ such that for any $0 < \a < \a_1$ and for any $\|\tv\| = 1$, $o(\a^2; \tv)\geq -q\a^2/4$. Therefore, for any $\xv \in \Nc(\xv^*, \a_1)$ not equal to $\xv^*$, we can take $\tv = (\xv - \xv^*)/\|\xv - \xv^*\|$ (which is feasible from convexity of $\X$)
, $\a = \|\xv - \xv^*\|$ and obtain:
\be
h(\xv) = h(\xv^* + \a \tv) \geq h(\xv^*) + \a^2 q/4 > h(\xv^*).
\en
\end{proof}

In the theorem above, we are considering ``approximately'' critical directions, rather than only the second-order derivatives along the critical directions. The following example demonstrates this point, as inspired by \citet[Example 2.1]{BenTalZowe85}:
\begin{eg}
We cannot take $r = 0$ in \eqref{eq:2nd_seeger}. Consider $f((x_1, x_2), y) = (2x_1 + x_1^2 + x_2^2)y + x_1^3$ and $(\xv^*, y^*) = (\zero, 0)$. $\of_\e(x_1, x_2) = \e|2x_1 + x_1^2 + x_2^2| + x_1^3$ and it is uniformly twice directional differentiable. We can evaluate $\D\of_\e((0, 0); (t_1, t_2)) = 2\e|t_1|$ and $$\He \of_\e((0, 0); (t_1, t_2)) = \begin{cases}
2\e (t_1^2 + t_2^2) & t_1 > 0, \\
2\e t_2^2 & t_1 = 0, \\
-2\e (t_1^2 + t_2^2) & t_1 < 0.
\end{cases}$$ The critical directions are $(0, t_2)$ along which $\He \of_\e(\zero, \tv) = 2\e t_2^2 > 0$. However, $$\of_\e((0, 0), (x_1, \sqrt{-2x_1 - x_1^2})) = x_1^3 < 0$$ if $-2\leq x_1 \leq 0$. 
\end{eg}

\subsection{Envelope function}
Our main interest in this work is the envelope function:
\begin{align}
\of(\xv) := \max_{\yv \in \Y} ~ f(\xv, \yv)
\end{align}
where $\Y$ is some compact topological Hausdorff space\footnote{Results in this section can be extended to the more general case where the constraint set $\Y$ depends on $\xv$ (in some semicontinuous manner); see \citet{Seeger88} for an excellent treatment. For our purpose here it suffices to consider a constant $\Y$.}. It is easy to verify:
\begin{itemize}
\item If $f: \X \times \Y \to \R$ is (jointly) continuous, then so is $\of$ (in $\xv$).
\item If also $\partial_\xv f: \X \times \Y \to \R$ is (jointly) continuous, then $\of$ is locally Lipschitz.
\end{itemize}
The envelope function turns out to be directionally differentiable:
\begin{thm}[e.g. \citealt{Danskin66,Demyanov66}]
\label{thm:danskin}
Let $f$ and $\partial_\xv f$ be (jointly) continuous. Then, the envelope function $\of$ is directionally differentiable:
\begin{align}\label{eq:danskin}
\D\of(\xv; \dv) = \max_{\yv \in \Y_0(\xv)} ~ \inner{\partial_{\xv} f(\xv, \yv)}{\dv}, \mbox{ where }~ \Y_0(\xv) := \{\yv\in \Y: \of(\xv) = f(\xv; \yv)\}.
\end{align}
Clearly, $\D\of(\xv; \cdot)$ is Lipschitz continuous.
\end{thm}

The following theorem explains the necessity of the function $\Af_\h$ in \Cref{thm:2nds}:
\begin{thm}[\citealt{Seeger88, Demyanov70}]
\label{thm:env_2nds}
Let $f$ and $\partial_\xv f$ be continuous. Then,
\begin{align}
\D\of(\xv;\dv) &= \max_{\yv \in \Y_0(\xv)} \inner{\partial_\xv f(\xv, \yv)}{\dv},  ~~ \Y_0(\xv) := \{\yv\in \Y: \of(\xv) = f(\xv, \yv)\} \\
\Hi\of(\xv; \dv, \gv) &\geq \max_{\yv \in \Y_1(\xv; \dv)} \Hi f(\xv, \yv; \dv, \gv), ~~ \Y_1(\xv; \dv) := \{\yv \in \Y_0(\xv): \D\of(\xv;\dv) = \inner{\partial_\xv f(\xv, \yv)}{\dv}\}.
\end{align}
If $\partial_{\xv\xv}^2 f$ is also (jointly) continuous, then 
\begin{align}
\Af_{\of}(\xv; \dv) := \max_{\yv \in \Y_1(\xv;\dv)} ~ \inner{\partial_{\xv\xv}^2 f(\xv,\yv)\dv}{\dv}
\end{align}
satisfies the uniformity condition in \Cref{thm:2nds}.
\end{thm}
\begin{proof}
We need only prove the last claim. Indeed 
\begin{align}
\frac{\of(\xv+t\dv) - \of(\xv) - t\D\of(\xv;\dv)}{t^2/2} &\geq \max_{\yv \in \Y_1(\xv;\dv)}\frac{f(\xv+t\dv, \yv) - f(\xv, \yv) - t \inner{\partial_{\xv}f(\xv, \yv)}{\dv}}{t^2/2} \tr
&=\max_{\yv \in \Y_1(\xv;\dv)}  \inner{\partial_{\xv\xv}^2f(\xv + t\theta(\yv,\dv) \cdot \dv, \yv) \dv}{\dv}.
\end{align}
Since $\partial_{\xv\xv}^2 f$ is continuous (hence uniformly continuous over compact sets), the right-hand side converges to $\Af_{\of}(\xv;\dv)$ uniformly on bounded sets in $\dv$ as $t$ goes to 0.
\end{proof}
\noindent 
When $\Y$ has limit points, proving $\Af_{\of}(\xv; \dv) = \He\of(\xv; \dv)$ may be difficult (even with additional regularity conditions). Nevertheless, we can still apply the sufficient condition in \Cref{thm:2nds}.

\citet{Seeger88} pointed out the following equivalence:
\begin{align}
\label{eq:1st_equiv}
\D\of(\xv; \dv) = \max_{\yv\in\Y_0(\xv)} \D f(\xv, \yv; \dv) = \max_{\yv\in \Y_0(\xv)} \sup_{\vv\in \K[d](\Y, \yv)} \D f(\xv, \yv; (\dv, \vv)),
\end{align}
where the first two directional derivatives are taken wrt $\xv$ only while the last directional derivative is joint wrt $(\xv, \yv)$. Indeed, when $f$ is (jointly) continuously differentiable,  $\D f(\xv, \yv; (\dv, \vv)) = \inner{\partial_\xv f(\xv,\yv)}{\dv} + \inner{\partial_\yv f(\xv,\yv)}{\vv}$. However, since $\yv \in \Y_0(\xv)$, we know from the necessary condition in \Cref{thm:nec_1st} that $\inner{\partial_\yv f(\xv,\yv)}{\vv} \leq 0$ for all $\vv \in \K[d](\Y, \yv)$. Surprisingly, the second-order counterparts are no longer equivalent:
\begin{thm}[\citealt{Seeger88}]
\label{thm:env_2nds_q}
Let $f: \X \times \Y \to \R$ be continuously differentiable. Then,
\begin{align}
\Hi\of(\xv; \dv, \gv) &\geq \max_{\yv \in \Y_0(\xv)} \sup_{\vv\in \Vc(\xv, \yv; \dv)} \sup_{\wv\in \K[d](\Y, \yv; \vv)} \Hi f(\xv, \yv; (\dv,\vv), (\gv,\wv)), 
\end{align}
where $\Y_0(\xv) =\{ \yv\in \Y: \of(\xv) = f(\xv, \yv)\}$ and $\Vc(\xv, \yv; \dv) := \{\vv\in\K[d](\Y, \yv): \D\of(\xv;\dv) = \D f(\xv, \yv; (\dv, \vv))\}$. 

If \blue{the second-order derivative of $f$} is also (jointly) continuous, then 
\begin{align}
\Af_{\of}(\xv; \dv) &:= \max_{\yv \in \Y_0(\xv)} \sup_{\vv\in\Vc(\xv,\yv;\dv)} \sup_{\wv\in \K[d](\Y, \yv; \vv)}~ \inner{\begin{bmatrix} \partial_{\xv\xv}^2 f(\xv,\yv) & \partial_{\xv\yv}^2 f(\xv,\yv) \\ \partial_{\yv\xv}^2 f(\xv,\yv) & \partial_{\yv\yv}^2 f(\xv,\yv) \end{bmatrix}{\dv \choose \vv}}{{\dv \choose \vv}} +\tr
&+ \inner{\partial_{\yv} f(\xv,\yv)}{\wv}
\end{align}
satisfies the uniformity condition in \Cref{thm:2nds}, provided that the directions $\dv, \vv$ and $\wv$ are bounded.
\end{thm}
\begin{proof}
We assume $\K[d](\Y, \yv; \vv)$ is not empty for otherwise the theorem is vacuous. For any $\wv \in \K[d](\Y, \yv; \vv)$ we know for any sequence $t_k \downarrow 0$ there exist a subsequence $t_{k_i} \downarrow 0$ and $\wv_{k_i} \to \wv$ such that $\yv + t_{k_i} \vv + t_{k_i}^2 \wv_k\in \Y$. Thus, fix any $\yv \in \Y_0(\xv)$, $\vv\in\Vc(\xv,\yv;\dv)$ and $\wv \in \K[d](\Y, \yv; \vv)$, we know (after passing to a subsequence if necessary)
\begin{align}
&\frac{\of(\xv+t_k\dv +t_k^2\gv/2) - \of(\xv) - t_k\D\of(\xv;\dv)}{t_k^2/2} \\
&\qquad \geq \frac{f(\xv+t_k\dv +t_k^2\gv/2, \yv+t_k\vv+t_k^2\wv_k/2) - f(\xv, \yv) - t_k\D f(\xv, \yv; (\dv,\vv))}{t_k^2/2} \\
&\qquad \geq \frac{f(\xv+t_k\dv +t_k^2\gv/2, \yv+t_k\vv+t_k^2\wv/2) - f(\xv, \yv) - t_k\D f(\xv, \yv; (\dv,\vv))}{t_k^2/2} +\\
& \qquad\qquad + \frac{f(\xv+t_k\dv +t_k^2\gv/2, \yv+t_k\vv+t_k^2\wv_k/2) - f(\xv+t_k\dv +t_k^2\gv/2, \yv+t_k\vv+t_k^2\wv/2)}{t_k^2/2}\\
&\qquad = \Hi f(\xv, \yv; (\dv, \vv), (\gv, \wv)) + o(t_k),
\end{align}
where the small order term $o(t_k)$ is independent of $\dv$, $\vv$ and $\wv$ if they are bounded.
\end{proof}
\noindent By setting $\yv \in \Y_1(\xv; \dv), \vv= \wv = \zero$, we see that the lower bounds in \Cref{thm:env_2nds_q} are always shaper than the ones in \Cref{thm:env_2nds}. However, note that \Cref{thm:env_2nds} only requires $\Y$ to be any compact topological space while \Cref{thm:env_2nds_q} only applies when $\Y$ is a compact set of some finite dimensional vector space.

\begin{eg}[\citealt{Seeger88}]
Let $\Y = \R^m$ and $f(\xv, \yv) = {\xv \choose \yv}^\top \left\{\frac12\begin{bmatrix} \Av & \Bv \\ \Bv^\top & \Cv\end{bmatrix} {\xv \choose \yv} + {\pv \choose \qv} \right\}$. Assume $\Cv \prec \zero$. Then, $\Y_0(\xv)$ is a singleton, $\Y_1 = \R^m$, and WLOG $\wv = \zero$. Therefore, 
\begin{align}
\Af_{\of}(\xv; \dv) = \dv^\top (\Av - \Bv \Cv^{-1} \Bv^\top) \dv,
\end{align}
whence $(\xv,\yv) = \begin{bmatrix} \Av & \Bv \\ \Bv^\top & \Cv\end{bmatrix}^{-1} {\pv \choose \qv}$ is a (unique) global saddle point if $\Cv \prec \zero$ and $\Av - \Bv\Cv^{-1} \Bv^\top \succ \zero$.

However, if we apply \Cref{thm:env_2nds} we can only conclude that 
\begin{align}
\Af_{\of}(\xv; \dv) = \dv^\top A \dv,
\end{align}
which is clearly a looser lower bound (recall that $C \prec \zero$).
\end{eg}

In principle, one should use the lower second-order directional derivative)~$\He_+(\xv^*; \dv, \gv)\geq 0$ for a stronger necessary condition.  
\blue{However, to our knowledge, we do not have an appropriate formula for it. We therefore look into \emph{upper} second-order derivatives instead for which \cite{kawasaki1988upper} showed a result.} From this result, we are able to introduce the second-order necessary conditions for $\xv^*$ being a local minimizer of $\of(\xv)$:
\begin{thm}[\citealt{kawasaki1988upper}]\label{thm:upper_semi_Kaw}
Let $f$ be twice (jointly) continuously differentiable. Then,
\begin{align}
\label{eq:Kaw}
\Hs\of(\xv; \dv, \gv) = \max_{\yv \in \Y_1(\xv, \dv)} ~ \inner{\partial_\xv f(\xv, \yv)}{\gv} + \inner{\partial_{\xv\xv}^2 f(\xv,\yv) \dv}{\dv} + \limsup_{\zv \to \yv} \tfrac12 v_-^2(\zv; \dv) u^\dag(\zv),
\end{align}
where $(t)_- = \min\{t, 0\}$, $t^\dag = \begin{cases} 1/t, & t \ne 0\\ 0, & t = 0\end{cases}$, and 
\begin{align}
u(\yv) &:= \of(\xv) - f(\xv, \yv) \geq 0, \,
v(\yv; \dv) := \D\of(\xv; \dv) - \D f(\xv, \yv; \dv).
\end{align}
\end{thm}
\begin{proof}
We give a direct (and arguably simpler) proof of this result. Denote 
\begin{align}
\Delta(t) := \frac{\of(\xv+t \dv+ t^2\gv/2) - \of(\xv) - t \D\of(\xv; \dv)}{t^2/2}.
\end{align}
Using the definitions of $u$ and $v$ we have 
\begin{align}
\Delta(t) = \frac{\of(\xv+t\dv+t^2\gv) - f(\xv, \zv) - t \D f(\xv, \zv; \dv) - u(\zv) - t v(\zv; \dv)}{t^2/2},
\end{align}
which holds for any $\zv \in \Y$. Let us first choose $\zv = \zv_t \in \Y_0(\xv+t\dv+t^2\gv)$:
\begin{align}
\Delta(t) &= \frac{f(\xv+t\dv+t^2\gv, \zv_t) - f(\xv, \zv_t) - t \D f(\xv, \zv_t; \dv) }{t^2/2} - \frac{u(\zv_t) + t v(\zv_t; \dv)}{t^2/2}. 
\end{align}
Let $\yv \in \Y_0(\xv)$ be a limit point of $\zv_t$. Suppose $\yv \in \Y_0(\xv) \setminus \Y_1(\xv;\dv)$. Then, for small $t$ we have (in the corresponding subsequence) $v(\zv_t; \dv) \approx v(\yv; \dv) > 0$ hence $\liminf_t \Delta(t) = \Hi \of(\xv; \dv, \gv) = -\infty$, contradicting \Cref{thm:env_2nds}.
Thus, $\yv \in \Y_1(\xv; \dv)$. Optimizing $t$ for the second term we obtain
\begin{align}
\Delta(t) \leq \frac{f(\xv+t\dv+t^2\gv, \zv_t) - f(\xv, \zv_t) - t \D f(\xv, \zv_t; \dv) }{t^2/2} + \tfrac12 v_-^2(\zv_t; \dv) u^\dag(\zv_t), 
\end{align}
where we used the fact that if $u(\zv_t) = 0$ then $v(\zv_t; \dv) \geq 0$ (see \Cref{thm:danskin}). 
Taking limits on both sides proves the $\leq$ part in \eqref{eq:Kaw}.

For the converse, let $\yv \in \Y_1(\xv;\dv)$ and $\zv_k \to \yv$ attain the maximum and limsup in \eqref{eq:Kaw}, respectively. We need only consider $\lim\limits_{\zv_k \to \yv} \tfrac12 v_-^2(\zv_k; \dv) u^\dag(\zv_k) > 0$, for otherwise the $\geq$ part in \eqref{eq:Kaw} would already follow from \Cref{thm:env_2nds}. We obviously have $u(\zv_k) > 0$ and $v(\zv_k ;\dv) < 0$ for sufficiently large $t$. Since $u(\zv_k) \to u(\yv) = 0$ we also have $v(\zv_k; \dv)\to v(\yv; \dv) = 0$. We claim that (after passing to a subsequence if necessary) $\lim_k u(\zv_k) / v(\zv_k; \dv) = 0$, for otherwise $\lim v^2(\zv_k; \dv) / u(\zv_k) = 0$, contradicting to its strict positivity. Now, setting $t_k = -2u(\zv_k) / v(\zv_k; \dv)$ we have (for large $k$):
\begin{align}
\Delta(t_k) &\geq \frac{f(\xv+t_k\dv+t_k^2\gv, \zv_k) - f(\xv, \zv_k) - t_k \D f(\xv, \zv_k; \dv) - u(\zv_k) - t_k v(\zv_k; \dv)}{t_k^2/2}\\
&= \frac{f(\xv+t_k\dv+t_k^2\gv, \zv_k) - f(\xv, \zv_k) - t_k \D f(\xv, \zv_k; \dv) }{t_k^2/2} + \tfrac12 v_-^2(\zv_k; \dv) u^\dag(\zv_k).
\end{align}
Taking limits on both sides we obtain the $\geq$ part in \eqref{eq:Kaw}.
\end{proof}
\noindent For later convenience, we remind that 
\begin{align}
\Y_0(\xv) = \{\yv: u(\yv) = 0\}, \, \Y_1(\xv; \dv) = \{\yv: u(\yv) = v(\yv; \dv) = 0\}.
\end{align}
and denote $\bar{E}(\yv;\tv) = \limsup_{\zv \to \yv} \tfrac12 v_-^2(\zv; \dv) u^\dag(\zv)$. 

With Carath\'{e}dory's theorem for convex hulls, one can obtain from \eqref{eq:Kaw} the following necessary condition for envelope functions:

\begin{thm}[\textbf{\cite{kawasaki1990second}}]\label{thm:kawasaki}
Assume $f\in \Cc^2$ and $\X = \R^n$. If $\xv^*$ is a local minimum of $\of(\xv)$, then for each  $\dv\in \R^n$ satisfying $\D \of(\xv^*; \dv) = 0$, there exist at most $n + 1$ points $\yv_1, \dots, \yv_{n+1}\in \Y_1(\xv^*;\dv)$ and $\l_1, \dots, \l_n \geq 0$ not all zero, such that:
\be\label{eq:sec_order}
\sum_{i=1}^{a}\l_i \n_\xv f(\xv^*, \yv_i) = {\bf 0}, \, \sum_{i=1}^{a}\l_i \left(\dv^\top \n_{\xv \xv}^2 f(\xv^*, \yv_i)\dv + \bar{E}(\yv_i; \dv) \right) \geq 0.
\en
\end{thm}
\begin{proof}
We borrow the result from \cite{kawasaki1990second}. In order to write down the second-order derivative formula in \cite{kawasaki1988upper}, we define $$Y_{0}(\tv) := \{\yv \in \Y : \mbox{there exists a sequence } \{\zv_k\} \to \yv, \, u(\zv_k) > 0\mbox{ and } {v}(\zv_k; \tv)/u(\zv_k)\to -\infty\},$$ and the following upper semi-continuous function \citep{kawasaki1988upper}:
\be\label{eq:extra}
\bar{E}'(\yv; \tv) =\begin{cases}
\sup_{\{\zv_k\}\to \yv} \limsup_k {{v}(\zv_k; \tv)^2} /(2u(\zv_k)) & \yv \in Y_{0}(\tv)\mbox{ and }\{\zv_k\}\mbox{ is in }Y_{0}(\tv), \\
0 & u(\yv) \!= {v}(\yv; \tv) \!= 0\mbox{ \& } \yv \notin Y_{0}(\tv) \\
-\infty & \textrm{otherwise}.
\end{cases}
\en
As shown in \citet{kawasaki1990second}, $u(\yv) = {v}(\yv; \tv) = 0$ whenever $\yv \in Y_0(\tv)$. We simplify the definition above:
\begin{lem}
Denoting $x_- := \min\{x, 0\}$, $x^\dag = 1/x\mbox{ if }x\neq 0\mbox{ and } x^\dag = 0\mbox{ otherwise}$, then for any $u(\yv) = {v}(\yv; \tv) = 0$,
\be\label{eq:extra_sim}
\bar{E}(\yv; \tv) = \limsup_{\zv_k\to \yv} {{v}_-(\zv_k; \tv)^2}u^\dag(\zv_k)/2.\en
\end{lem}
\begin{proof}
It suffices to consider those sequences $\{\zv_k\}\subset \Y$ such that $u(\zv_k)\geq 0$. We want to prove that $\bar{E}(\yv; \tv) = \bar{E}'(\yv; \tv)$. We first prove $\bar{E}(\yv; \tv) \geq \bar{E}'(\yv; \tv)$. If $\yv \in Y_0(\tv)$, then for any $\d > 0$, there exists a sequence $\{\zv_k\}$ such that
$$\limsup_k {{v}(\zv_k; \tv)^2} /(2u(\zv_k)) \geq \bar{E}'(\yv; \tv) - \d,$$
$u(\zv_k) > 0\mbox{ and } {v}(\zv_k; \tv)/u(\zv_k)\to -\infty$.
For large enough $m$, ${v}(\zv_k;\tv) < 0$, and thus we take the same sequence in \eqref{eq:extra_sim} to obtain $\bar{E}(\yv; \tv) \geq \bar{E}'(\yv; \tv) -\d$. Since the above holds for any $\d > 0$, we have $\bar{E}(\yv; \tv) \geq \bar{E}'(\yv; \tv)$. If $\yv \notin Y_0(\tv)$, then $\bar{E}(\yv; \tv) \geq 0 = \bar{E}'(\yv; \tv)$. 

Now let us prove that $\bar{E}(\yv; \tv) \leq \bar{E}'(\yv; \tv)$. Assume for any $\d > 0$, $\{\zv_k\}$ is the sequence such that $$\limsup_k {{v}_-(\zv_k; \tv)^2}u^\dag(\zv_k)/2 \geq \bar{E}(\yv; \tv) - \d.$$ If $u(\zv_k) > 0$ or ${v}(\zv_k;\tv) < 0$ for finite number of $m$, then $\bar{E}(\yv; \tv) = 0\leq \bar{E}'(\yv; \tv)$. Assume WLOG now that for any $m$,  $u(\zv_k) > 0$ and ${v}(\zv_k;\tv) < 0$, if ${v}(\zv_k;\tv) /u(\zv_k)$ is bounded, then since ${v}(\yv;\tv) = 0$, $\bar{E}(\yv; \tv) = 0 \leq \bar{E}'(\yv; \tv)$. So we can assume further that ${v}(\zv_k;\tv) /u(\zv_k) \to -\infty$. Using the same sequence in \eqref{eq:extra}, we know $\bar{E}'(\yv; \tv)\geq \bar{E}(\yv; \tv) - \d$ for any $\d > 0$, and thus $\bar{E}'(\yv; \tv)\geq \bar{E}(\yv; \tv)$.
\end{proof}

\end{proof}
Moreover, the following assumption guarantees the existence of $\He \of(\xv; \dv,\gv)$ from which we can get second-order sufficient conditions:

\begin{assmp}[\textbf{\citet{kawasaki1992second}}]\label{assmp:suff}
For each $\yv\in \Y_1(\xv^*;\tv)$ with $\tv \neq \zero$ and $\D \of(\xv^*; \tv) = 0$, and for each non-zero $\dv \in \R^m$, there exist $\a, \b \neq 0$ and $p, q > 0$ such that the following approximation holds:
\be
{u}(\yv+\d \dv) = \a \d^p + o(\d^p), \, {v}(\yv + \d \dv;\tv) = \b \d^q + o(\d^q),
\en
whenever $\yv+\d \dv \in \Nc(\yv^*, \e)$ and $\d > 0$. Note that
$$
u(\yv) := \of(\xv^*) - f(\xv^*, \yv), \,
v(\yv; \dv) := \D\of(\xv^*; \dv) - \D f(\xv^*, \yv; \dv).
$$
\end{assmp}

\begin{thm}[\textbf{second-order sufficient condition, \textcite{kawasaki1992second}}]
\label{thm:suf_kawa}
Assume \Cref{assmp:suff} holds at $\xv^*$. Let $\X = \R^n$ and $\Y$ be convex. $\xv^*$ is an isolated local minimum of $\of(\xv)$ if for any $\dv\in \R^n$, $\D \of(\xv^*; \dv) > 0$, or $\D \of(\xv^*; \dv) = 0$, $\dv \neq \zero$ and there exist $a\geq 1$ points $\yv_1, \dots, \yv_a\in \Y_1(\xv^*;\dv)$ and $\l_1, \dots, \l_a > 0$ such that:
\be\label{eq:sec_order_suff}
\sum_{i=1}^{a}\l_i \n_\xv f(\xv^*, \yv_i) = {\bf 0}, \, \sum_{i=1}^{a}\l_i \left(\dv^\top \n_{\xv \xv}^2 f(\xv^*, \yv_i)\dv + \bar{E}(\yv_i; \dv) \right) > 0.
\en
\end{thm}


\section{Proofs in \Cref{sec:local_opt}}

\slmm*

\begin{proof}
Given that $\partial^2_{\yv\yv} f(\xv^*, \yv^*)$ is invertible, the first condition is clearly equivalent to $\yv^*$ being a local maximizer of $f(\xv^*, \cdot)$. Consider the non-linear equation \eqref{eq:im-tol}, whose solution is determined by the implicit function theorem as a continuously differentiable function $\yv(\xv)$ defined near $\xv^*$. Fix any $\e$. Since $\yv(\xv^*) = \yv^*$, shrinking the neighbourhood around $\xv^*$ if necessary we may assume $\yv(\xv) \in \Nc(\yv^*, \e)$ so that $\of_\e(\xv) = f(\xv, \yv(\xv))$. Thus, if $(\xv^*, \yv^*)$ is local minimax, then for $\xv$ near $\xv^*$:
\begin{align}
f(\xv^*, \yv(\xv^*)) = f(\xv^*, \yv^*) = \of_\e(\xv^*) \leq \of_\e(\xv) = f(\xv, \yv(\xv)),
\end{align}
so, $\xv^*$ is a local minimizer of the total function. Reversing the argument proves the converse.
\end{proof}

\inclusion* 
\begin{proof}
We first note that since $\yv^*$ maximizes $f(\xv^*, \yv)$ over $\Nc(\yv^*, \e_0)$, we clearly have for all $\yv^* \in \Nc \subseteq \Nc(\yv^*, \e_0)$:
\begin{align}
\of_{\Nc}(\xv^*) = f(\xv^*, \yv^*).
\end{align}
Moreover, for any $\Nc\supseteq \Nc(\yv^*, \e)$ and any $\xv \in \X$:
\be\label{eq:envelope}
\of_{\Nc}(\xv) \geq \of_{\e, \yv^*}(\xv) =: \of_{\e}(\xv). 
\en
Since $\xv^*$ is a local minimizer of $\of_{\e}$, say over the neighborhood $\Mc$, we have for all $\xv \in \Mc$ and $\Nc(\yv^*, \e) \subseteq \Nc \subseteq \Nc(\yv^*, \e_0)$:
\begin{align}
\of_{\Nc}(\xv) \geq \of_{\e}(\xv) \geq \of_{\e}(\xv^*) = f(\xv^*, \yv^*) = \of_{\Nc}(\xv^*),
\end{align}
i.e., $\xv^*$ is a local minimizer of $\of_{\Nc}(\xv)$ over the same local neighborhood $\Mc$.
\end{proof}

\propcon*

\begin{proof}
\blue{
We need only prove if $(\xv^*, \yv^*)$ is local minimax according to \Cref{def:jin}, then there exists some $\e_0 > 0$ such that $\xv^*$ is a local minimizer of $\of_{\e}(\xv)$ for all $\e \in (0, \e_0]$. 
Indeed, from \Cref{def:jin} we know  $f(\xv^*, \yv)$ is maximized at $\yv^*$ over some neighborhood $\Nc(\yv^*, \e_0)$ for some $\e_0 > 0$. For any $0 < \e \leq \e_0$, one can find $0 < \e_n < \e$ since the promised sequence $\e_n \to 0$. By definition $\xv^*$ is a local minimizer for $\of_{\e_n}$, hence by \Cref{thm:eqseq} it remains a local minimizer for $\of_{\e}$.
}
\end{proof}

\uniform*

\begin{proof}
Let $(\xv_\star, \yv_\star)$ be local saddle, i.e., $\yv_\star$ maximizes $f(\xv_\star, \cdot)$ over the neighborhood $\Nc(\yv_\star, \e)$ and $\xv_\star$ minimizes \blue{$\of_{0, \yv_\star} = f(\cdot, \yv_\star)$} over the neighborhood $\Nc(\xv_\star, \e)$. We fix the neighborhood  $\Nc(\xv_\star) = \Nc(\xv_\star, \e)$ and choose \blue{any sequence} $\{\e_n\}\subset (0, \e]$. Applying \Cref{thm:eqseq} we know $\xv_\star$ remains a minimum for all $\of_{\e_n}$ over the (fixed) neighborhood $\Nc(\xv_\star)$. Thus, $(\xv_\star, \yv_\star)$ is uniformly local minimax. 

Conversely, let $f$ be upper semi-continuous (in $\yv$ for any $\xv$) and $(\xv^*, \yv^*)$ uniformly local minimax over the fixed neighborhood $\Nc(\xv^*)$. By definition $\yv^*$ maximizes $f(\xv^*, \cdot)$ over some neighborhood $\Nc(\yv^*, \e_0)$, and $\xv^*$ minimizes all $\of_{\e_n}$ over the fixed neighborhood $\Nc(\xv^*)$, where the positive sequence $\e_n \to 0$. Fix any $\xv \in \Nc(\xv^*)$. Since $f(\xv, \cdot)$ is upper semi-continuous at $\yv^*$, we have for any $\d > 0$, there exists $\e_n \in(0, \e_0]$ such that:
\be
f(\xv^*, \yv^*) = \of_{\e_n}(\xv^*) \leq \of_{\e_n}(\xv) \leq f(\xv, \yv^*) + \d.
\en
Letting $\d \to 0$ we know $f(\xv, \yv^*)\geq f(\xv^*, \yv^*)$ for any  $\xv \in \Nc(\xv^*)$. 
\end{proof}

\equivalence*
\begin{proof}
($\Longleftarrow$) Suppose $(\xv^*, \yv^*)$ satisfies \eqref{eq:Jin}. Then clearly,  $\yv^*$ maximizes $f(\xv^*, \cdot)$ over the neighborhood $\Nc(\xv^*, \d_0)$. Take an arbitrary positive sequence $\{\d_n\}$ with $\d_n \to 0$ and let $\e_n = \sup_{m\geq n}h(\d_n)$. Since $h(\d) \to 0$ as $\d \to 0$, we may assume WLOG that $\e_n$ is well-defined and bounded from above. If $h(\d_n) = 0$ for some $n$ then $(\xv^*, \yv^*)$ is local saddle and hence local minimax thanks to \Cref{prop:ulmm}. Otherwise we have $\e_n > 0$ for all $\e_n$ and $\e_n \to 0$ since $\lim_{\d\to 0} h(\d) = 0$. WLOG we  assume $\e_1 \leq \d_0$ (for otherwise we may discard the head of the sequence $\{\e_n\}$). From \eqref{eq:Jin} we know for any $\xv \in \Nc(\xv^*, \d_n)$:
\be
\of_{h(\d_n)}(\xv) \geq f(\xv^*, \yv^*) = \of_{h(\d_n)}(\xv^*),
\en
since $h(\d_n) \leq \e_1 \leq \d_0$ and $\yv^*$ maximizes $f(\xv^*, \yv)$ over $\Nc(\xv^*, \d_0)$.
Therefore, $\xv^*$ is a local minimizer of $\of_{h(\d_n)}$ hence also of  $\of_{\e_n}$ thanks to \Cref{thm:eqseq}. 

($\Longrightarrow$) Suppose $(\xv^*, \yv^*)$ is local minimax (see \Cref{def:jin}). Then, $\yv^*$ maximizes $f(\xv^*, \cdot)$ over some neighborhood $\Nc(\yv^*,\e_0)$ where $\e_0 > 0$. Since $\xv^*$ is a local minimizer of $\of_{\e_n}$, it minimizes $\of_{\e_n}$ over some neighborhood $\Nc(\xv^*, \d'_n)$ with $\d'_n > 0$. From $\{\d'_n\}$ we construct another positive sequence $\{\d_n\}$ where $\d_0 = \min\{\d'_1, 1, \e_0\}$ > 0 and
\be
\d_{n} = \min\{\d'_{n}, \d_{n-1}, 1/n\}, ~~ n = 1, 2, \ldots,
\en
which is diminishing by construction. Define $h(\d) = \e_n$ if $\d_{n+1} < \d \leq \d_n$. Since $\e_n \to 0$, $\lim_{\d\to 0}h(\d) = 0$. WLOG we assume $\e_1 \leq \e_0$ and by definition $\d_0 \leq \e_0$. For any $\d \in (0, \d_0]$ there exists some $n$ such that $\d \in (\d_{n+1}, \d_n]$. Thus, for any $(\xv,\yv) \in \Nc(\xv^*, \d'_n) \times \Nc(\yv^*, \e_0)$:
\be\label{eq:jinh}
\of_{h(\d)}(\xv) = \of_{\e_n}(\xv) \geq \of_{\e_n}(\xv^*) = f(\xv^*, \yv^*)\geq f(\xv^*, \yv). 
\en
Since $\d \leq \d_n \leq \d_n'$ and $\d \leq \e_0$, the above still holds over the smaller neighborhood $\Nc(\xv^*, \d) \times \Nc(\yv^*, \d)$, which is exactly \eqref{eq:Jin}.
\end{proof}

\convexconcave*

\begin{proof}
Suppose $(\xv^*, \yv^*)$ is stationary. For any small $\e > 0$,
\begin{align}
\of_\e(\xv) = \max_{\yv \in \Nc(\yv^*, \e)} ~ f(\xv, \yv)
\end{align}
 is convex by assumption. To see that $\xv^*$ is a local (hence global) minimizer of $\of_\e$, we need only verify that $\zero \in \partial \of_\e(\xv^*)$. Since $\yv^*$ maximizes $f(\xv^*, \cdot)$ by assumption, we know from Danskin's theorem that $ \partial \of_\e(\xv^*) \supseteq \partial f(\xv^*, \yv^*) \ni \zero$ since $(\xv^*, \yv^*)$ is stationary.
 
Now suppose $(\xv^*, \yv^*)$ is local minimax. Then, $\yv^*$ is a local hence global maximizer of $f(\xv^*, \cdot)$.  Also, $\xv^*$ is a local hence global minimizer of $\of_\e$. Thus, 
\be
\of(\xv) \geq \of_\e(\xv) \geq \of_\e(\xv^*) = f(\xv^*, \yv^*) = \of(\xv^*),
\en
i.e., $\xv^*$ is a global minimizer of $\of$.
\end{proof}

\cclocal*

\begin{proof}
For convex-concave functions being local saddle is equivalent to satisfying \eqref{eq:first_order_necessary_lrp}. We also know from \Cref{prop:ulmm} that every local saddle point is local minimax (maximin) and from \Cref{def:glp} that every local minimax point is an {\glp}.
\end{proof}

\nested*

\begin{proof}
Clearly, $\of_\e(\xv^*) = f(\xv^*, \yv^*)$ for any $\e \in [0, \e_0]$ and $\yv \in \Nc(\yv^*, \e_1)$ implies $\yv\in \Nc(\yv^*, \e_2)$ for any $\e_1 \leq \e_2$, whence follows $\Y_0(\xv^*; \e_1) \subseteq \Y_0(\xv^*; \e_2)$.
Using Danskin's theorem in \Cref{thm:danskin} we thus have  $\D\of_{\e_2}(\xv^*;\tv) \geq \D \of_{\e_1}(\xv^*;\tv)$. 
\end{proof}

\secondminmax*

\begin{proof}
We know $\of_\e$ is locally Lipschitz since $\n_\xv f$ is continuous, and there exists $\e_0 > 0$ such that $\of_\e(\xv^*) = f(\xv^*, \yv^*)$ for any $0 < \e < \e_0$. The rest of the claim can be readily derived from \Cref{thm:necessary_2} and \Cref{thm:upper_semi_Kaw}, by taking $\e \to 0$ and noting that the upper directional derivative is by definition larger than the lower directional derivative. 
\end{proof}

\secondsuff*

\begin{proof}
It follows from \Cref{thm:suf_kawa}. From Danskin's theorem $\D \of_\e(\xv^*;\tv) \geq 0$ for any small $\e > 0$.
Besides, for any small enough $\e$, \eqref{eq:sec_order_suff} is satisfied since $\yv^*\in \Y_1(\xv^*;\e_0;\tv)$. Noting that $\of_{\e}(\xv^*) = f(\xv^*, \yv^*) = f(\xv^*, \yv)$ for any $0 \leq \e < \e_0$ and $\yv\in \Y_1(\xv^*;\e_0;\tv)$, \eqref{eq:taylor_exp} follows from \Cref{assmp:suff}.
\end{proof}

\secondsuffs*

\begin{proof}
Since $\yv^*\in \Y_0(\xv^*;\e)$, from Danskin's theorem (\blue{\Cref{thm:danskin}}) we know that $\D \of_\e(\xv^*;\tv)\geq 0$ for any $\e$ small enough. We then combine \Cref{thm:2nds} with \Cref{thm:env_2nds_q}. Note that all the directions $\tv, \vv, \wv$ are bounded.
\end{proof}


\section{Proofs in \Cref{sec:q2}}

\minimaxq*

\begin{proof}
The first claim follows directly from the definition of stationarity. 

To prove the second claim, we note that fixing $\xv$, $q(\xv, \cdot)$ is clearly quadratic in $\yv$. Thus, it admits a local (hence also global) maximizer $\yv$ iff 
\begin{align}
\Bv \preceq \zero, \\
\label{eq:quadstat1} \Cv^\top \xv + \Bv\yv = \zero.
\end{align}
Note that there exists some $\yv$ to satisfy \eqref{eq:quadstat1} iff $\Cv^\top \xv$ belongs to the range space of $\Bv$  iff 
\begin{align}
\p{\Bv} \Cv^\top \xv = \zero, ~i.e.~ \Lv^\top \xv =\zero, 
\end{align}
or equivalently $\xv = \p{\Lv} \zv$ for some $\zv \in \R^m$. 
Therefore, we have the envelope function:
\begin{align}\label{eq:barq}
\bar{q}(\xv) = \begin{cases}
\frac12 \xv^\top( \Av - \Cv \Bv^\dag \Cv^\top) \xv, & \Lv^\top\xv = \zero \\
\infty, & \mbox{ otherwise }
\end{cases}
.
\end{align}
Thus, the quadratic function $\bar{q}$ (when restricted to the null space of $\Lv^\top$) admits a local (hence also global) minimizer iff 
\begin{align}
\p{\Lv} (\Av - \Cv \Bv^\dag \Cv^\top) \p{\Lv} \succeq \zero,
\end{align}
in which case the minimizer $\xv$ satisfies
\begin{align}
\label{eq:xtmp}
\Lv^\top \xv = \zero = \p{\Lv}(\Av- \Cv\Bv^\dag \Cv^\top) \xv, 
\end{align}
whereas the maximizer $\yv$ satisfies \eqref{eq:quadstat1}. It is easy to verify that \eqref{eq:xtmp} and \eqref{eq:quadstat1} are equivalent to \eqref{eq:gLmM}. For the last claim, note first that we have proved in \Cref{prop:stable} that any local minimax point is stationary. Moreover, if $(\xv^*, \yv^*)$ is local minimax, then $\xv^*$ locally minimizes $\bar{q}_{\e, \yv^*}$ (for all small $\e$), i.e., for $\xv$ close to $\xv^*$, we have 
\begin{align}
\bar{q}(\xv) \geq \bar{q}_{\e, \yv^*}(\xv) \geq \bar{q}_{\e, \yv^*}(\xv^*) = q(\xv^*, \yv^*)= \bar{q}(\xv^*),
\end{align}
where the last equality follows since fixing $\xv^*$, $\yv^*$ is a local hence also global maximizer of the quadratic function $q(\xv^*, \cdot)$. We have shown above that any local minimizer of $\bar{q}(\xv)$ is necessarily global. Therefore, $(\xv^*, \yv^*)$ is global minimax.

Lastly, we prove the converse of the last claim. Let $\Bv \preceq \zero$, $\p{\Lv}(\Av - \Cv \Bv^\dag \Cv^\top) \p{\Lv} \cge \zero$, and $(\xv^*, \yv^*)$ be stationary, i.e. they satisfy \eqref{eq:stat}. Fixing $\yv^*$ we have for all small  $\e > 0$:
\begin{align}
2\bar{q}_\e(\xv) = 2\bar{q}_{\e, \yv^*}(\xv) = \max_{\|\yv - \yv^*\|\leq\e} 
\begin{bmatrix}
\xv \\ \yv
\end{bmatrix}^\top
\begin{bmatrix}
\Av & \Cv \\ \Cv^\top & \Bv
\end{bmatrix}
\begin{bmatrix}
\xv \\ \yv
\end{bmatrix}
.
\end{align}
We are left to prove $\xv^*$ is a local minimizer of $\bar{q}_\e$ for all small $\e$.\footnote{Unfortunately we cannot use the sufficient conditions in \Cref{sec:ns} since $\xv^*$ may not be an isolated local minimizer.} Let $c = \max\{\|\Bv^\dag \Cv^\top\|, \|\Av - \Cv \Bv^\dag \Cv^\top\|\}$. We assume first $c > 0$ and $\Lv \ne \zero$. Let $\sigma$ be the smallest positive singular value of $\Lv = \Cv\p{\Bv}$. Consider any $\xv$ such that $\|\xv - \xv^*\| \leq \e (\sigma \wedge 1)/ (3 c)$. We decompose 
\begin{align}
\xv - \xv^* = \deltav_{\parallel} + \deltav_\perp, \mbox{ where } \deltav_\perp = \p{\Lv} (\xv - \xv^*),
\end{align}
and define 
\begin{align}
\yv - \yv^* = -\Bv^\dag\Cv^\top (\xv - \xv^*) + \e \Lv^\top (\xv - \xv^*) / (2\|\Lv^\top (\xv - \xv^*)\|),
\end{align}
where by convention $0/0 := 0$.
Clearly, $\|\yv - \yv^*\| \leq \e / 3 + \e / 2 < \e$. Thus, using the stationarity of $(\xv^*,\yv^*)$:
\begin{align}
2\bar{q}_{\e}(\xv) \geq 2q(\xv, \yv) 
&= 
\begin{bmatrix}
\xv- \xv^* \\ \yv-\yv^*
\end{bmatrix}^\top
\begin{bmatrix}
\Av & \Cv \\ \Cv^\top & \Bv
\end{bmatrix}
\begin{bmatrix}
\xv-\xv^* \\ \yv-\yv^*
\end{bmatrix}
\\
\label{eq:trick}(\mbox{note } \Bv\Lv^\top = \zero) &=
(\xv- \xv^*)^\top (\Av - \Cv\Bv^\dag\Cv^\top)(\xv- \xv^*) + \e\|\Lv^\top (\xv - \xv^*)\|
\\
&= \deltav_\parallel^\top (\Av - \Cv\Bv^\dag\Cv^\top)\deltav_\parallel + 2\deltav_\parallel^\top (\Av - \Cv\Bv^\dag\Cv^\top)\deltav_\perp + \tr
& + \deltav_\perp^\top (\Av - \Cv\Bv^\dag\Cv^\top)\deltav_\perp +\e\|\Lv^\top \deltav_\parallel\|
\\
&\geq 
-\e\sigma  \|\deltav_\parallel\|/3 - 2\e\sigma  \|\deltav_\parallel\| / 3 + 0 + \e\sigma\|\deltav_\parallel\| = 0 = 2 \bar{q}_\e(\xv^*),
\end{align}
where we used the fact that $\|\deltav_\parallel\| \vee \|\deltav_\perp\| \leq \e\sigma/(3c)$ and $\p{\Lv} (\Av - \Cv\Bv^\dag\Cv^\top) \p{\Lv} \succeq \zero$. 
Finally, we note that if $c = 0$, then $\Av - \Cv\Bv^\dag \Cv^\top = \zero$ hence the proof still goes through (with $c$ replaced by 1 say). Similarly, if $\Lv = \zero$, then $\deltav_{\parallel} = \zero$ hence the proof again goes through (with $\sigma$ replaced by 1 say).
\end{proof}

\localglobalq*

\begin{proof}
If \eqref{eq:qc} holds, let 
\begin{align}
\label{eq:qctmp2}
\begin{bmatrix}
\Av & \Cv \\
\Cv^\top & \Bv
\end{bmatrix}
\begin{bmatrix}
\xv^* \\ \yv^*
\end{bmatrix}
= 
\begin{bmatrix}
\av \\ \bv
\end{bmatrix}
.
\end{align}
Then, performing the translation $(\xv, \yv) \gets (\xv - \xv^*, \yv - \yv^*)$ we reduce to the homogeneous case and applying \Cref{thm:local_global_q} we obtain the existence of a local (or global) minimax point. If a local minimax point exists, then stationarity yields the range condition. Performing translation and applying \Cref{thm:local_global_q} again establishes all conditions in \eqref{eq:qc}. 

All we are left to prove is when a global minimax point $(\xv^*, \yv^*)$ exists the range condition holds. Indeed, fixing $\xv^*$, $\yv^*$ maximizes the quadratic $q(\xv^*, \cdot)$ hence from stationarity:
\begin{align}
\label{eq:qctmp}
\Cv^\top \xv^* + \Bv\yv^* = \bv.
\end{align}
The above equation has a solution $\yv^*$ iff $\p{\Bv} \Cv^\top \xv^* = \p{\Bv}\bv$, i.e. $\Lv^\top \xv^* = \p{\Bv} \bv$ (recall that $\Lv := \Cv \p{\Bv}$).
Solving $\yv$ and plugging back in $q$ we obtain: for all $\xv$ such that $\Lv^\top \xv = \p{\Bv} \bv$, 
\begin{align}
\bar{q}(\xv) = \tfrac12 \xv^\top (\Av - \Cv\Bv^\dag\Cv^\top) \xv + \xv^\top \Cv\Bv^\dag\bv - \av^\top\xv.
\end{align}
Since $\xv^*$ is a global minimizer of $\bar{q}$, we obtain the stationarity condition:
\begin{align}
\p{\Lv} [(\Av - \Cv \Bv^\dag \Cv^\top) \xv^* + \Cv \Bv^\dag\bv- \av] = \zero.
\end{align}
Combined with \eqref{eq:qctmp} we obtain:
\begin{align}
\p{\Lv} [\Av\xv^* + \Cv \Bv^\dag \Bv \yv^* - \av] = \zero &\iff \Av\xv^* + \Cv \Bv^\dag \Bv \yv^* - \av = \Lv \zv = \Cv\p{\Bv}\zv \mbox{ for some } \zv \\
\label{eq:qctmp3}
&\iff \Av\xv^* + \Cv (\Bv^\dag \Bv \yv^*+ \p{\B} \zv) = \av
\end{align}
From \eqref{eq:qctmp} and \eqref{eq:qctmp3} we deduce $(\xv^*, \Bv^\dag \Bv \yv^*+ \p{\B} \zv)$ satisfies the range condition \eqref{eq:qctmp2}.
\end{proof}


\section{Momentum algorithms}\label{app:momentum_alg}

We study the effect of momentum for convergence to local saddle points, including heavy ball \citep{Polyak64} and Nesterov's momentum \citep{nesterov1983method}. They are similar to GDA and do not converge even for bilinear games, as proved in \cite{zhang2019convergence}. In the following two subsections, we study the effect of momentum for convergence to local saddle points. GDA is a special case if we take the momentum parameter $\b = 0$.

\blue{Many of the proofs in this appendix and \Cref{app:proof_stability} rely on Schur's theorem:}
\begin{restatable}[\textbf{\citet{schur1917potenzreihen}}]{thm}{}\label{schur}
The roots of a real polynomial $p(\l) = a_0 \l^n + a_1 \l^{n-1} + \dots + a_n$ are within the (open) unit disk of the complex plane iff $\forall k \in \{1, 2, \dots, n\}, ~ \det({\bf P}_k {\bf P}_k^{\sf H} - {\bf Q}_k^{\sf H} {\bf Q}_k) > 0$, where ${\bf P}_k, {\bf Q}_k$ are $k\times k$ matrices defined as: $[{\bf P}_k]_{i,j} = a_{i-j}{\bf 1}_{i\geq j}$, $[{\bf Q}_k]_{i,j} = a_{n-i+j}{\bf 1}_{i\leq j}$. 
\end{restatable}
In this theorem, we use $A^{\sf H}$ to denote the Hermitian conjugate of $A$, and 
\be
{\bf 1}_{\rm condition} = \begin{cases}
1 & \textrm{if condition is true,}\\
0 & \textrm{otherwise.}
\end{cases}
\en
Schur's theorem has been applied to analyze bilinear zero-sum games to give necessary and sufficient convergence conditions \citep{zhang2019convergence}. However, in that paper only real polynomials have been studied. Here we give a corollary for complex quadratic polynomials:
\begin{restatable}[\textbf{Schur}]{lem}{}\label{lem:2}
For complex quadratic polynomials $\l^2 + a \l + b$, the exact convergence condition is: 
\be
|b| < 1, \, (1 - |b|^2)^2 + 2 \Re(a^2 \bar{b}) > |a|^2 (1 + |b|^2).
\en
\end{restatable}
\begin{proof}
For quadratic polynomials, we compute 
\be
&&{\bf P}_1 = [1],\, {\bf Q}_1 = [b], \\
&&{\bf P}_2 = \begin{bmatrix}
1 & 0 \\
a & 1
\end{bmatrix}, \,
{\bf Q}_2 = \begin{bmatrix}
b & a \\
0 & b
\end{bmatrix}, 
\en
We require $\det({\bf P}_k {\bf P}_k^{\sf H} - {\bf Q}_k^{\sf H}{\bf Q}_k) =: \delta_k > 0$, for $k = 1, 2$. If $k = 1$, we have $1 - |b|^2 > 0$. If $k = 2$, we have:
\be
{\bf P}_k {\bf P}_k^{\sf H} - {\bf Q}_k^{\sf H} {\bf Q}_k =  \begin{bmatrix}
1 - |b|^2 & \bar{a} - a \bar{b} \\
a - \bar{a} b & 1 - |b|^2 
\end{bmatrix},
\en
where $\bar{a}$ means the complex conjugate. The determinant should be positive, so we have:
\be
 (1 - |b|^2)^2 + 2 \Re(a^2 \bar{b}) > |a|^2 (1 + |b|^2).
\en
\end{proof}

Some proofs in this section rely on Mathematica code, mostly with the built-in function \texttt{Reduce}. This function relies on cylindrical algebraic decomposition \citep{basu2005algorithms} and can be verified manually.

\subsection{Heavy ball (HB)}
We study the heavy ball method HB$(\a_1, \a_2, \b)$ \citep{Polyak64} in the context of minimax optimization, as also studied in \citet{GidelHPHLLM19, zhang2019convergence}: 
\be \label{eq:hb} \zv_{t+1} = \zv_{t} + {\vv}(\zv_t) + \b(\zv_t - \zv_{t-1}), \vv(\zv) = (-\a_1 \n_\xv f(\zv), \a_2 \n_\yv f(\zv)). \en
\begin{thm}[\textbf{HB}]\label{thm:hb}
${\rm HB}(\a_1, \a_2, \b)$ is exponentially stable iff $\forall \, \l \in \sp(\haa)$, $ |\b| < 1$,
$$
2 \b \Re(\l^2) - 2 (1 - \b)^2 (1+\b) \Re(\l) >  (1 + \b^2)|\l|^2.$$
\end{thm}

\begin{proof} With state augmentation $\zv_t \to (\zv_{t+1}, \zv_t)$, the Jacobian for HB$(\a_1, \a_2, \b)$ is:
\be\label{eq:hb_j}
{\bf J}_{\rm HB}(f) = \begin{bmatrix}
(1+\b) {\bf I}_{n+m} + \haa & -\b {\bf I}_{n+m} \\
{\bf I}_{n+m} & \zero
\end{bmatrix},
\en
The spectrum can be computed as:
\be
&&\sp({\bf J}_{\rm HB}(f)) =\{w : p(w):=(w-1)(w-\b) - w \l = 0, \l \in \haa\}.
\en
This quadratic equation can be further expanded as:
\be\label{eq:hb_expand}
w^2 - (\b + 1 + \l) w + \b = 0.
\en
With \Cref{lem:2}, we obtain the necessary and sufficient conditions for which all the roots are within a unit disk:
\be\label{eq:n_mom}
&&|\b| < 1,  2\b \Re(\l^2) - 2(1 - \b)^2 (1+\b) \Re(\l) > (1 + \b^2)|\l|^2.
\en
\end{proof}

This theorem can also be derived from Euler transform as in \cite[Section 6]{niethammer1983analysis} which is used in analyzing methods for solving linear equations. The first inequality $|\b| < 1$ can be easily used to guide hyper-parameter tuning in practice. The second condition in fact describes an ellipsoid centered at $(-\b - 1, 0)$. If we define $\l = u + i v$ and $(u, v)\in \R^2$, then this condition can be simplified as:
\be\label{eq:elips}
\frac{(u + \b + 1)^2}{(\b + 1)^2} + \frac{v^2}{(\b - 1)^2} < 1.
\en
As shown on the left of \Cref{fig:hb_nag_conv}, if the momentum factor $\b$ is positive, the ellipsoid is elongated in the horizontal direction; otherwise, it is elongated in the vertical direction. This agrees with existing results on negative momentum \citep{GidelHPHLLM19, zhang2019convergence}, where they studied bilinear games. 

\begin{cor}[\textbf{HB}]\label{cor:hb_stable}
For any $|\b| < 1$, HB$(\a, \a,\b)$ is exponentially stable for small enough $\a$ at a local saddle point iff at such a point $\Re(\l) \neq 0$ for all $\l \in \sp(\Hv)$.
\end{cor}

\begin{proof} From \Cref{lem:loc_sadl}, for any $\l \in \sp(\Hv)$, $\Re(\l) \leq 0$. If $\Re(\l) \neq 0$ for all $\l \in \sp(\Hv)$, then \eqref{eq:elips} holds for small enough $\a$. If $\Re(\l) = 0$ for some $\l \in \sp(\Hv)$, we cannot have \eqref{eq:elips}. 
\end{proof}

\begin{figure}
    \centering
    \includegraphics[width=0.8\textwidth]{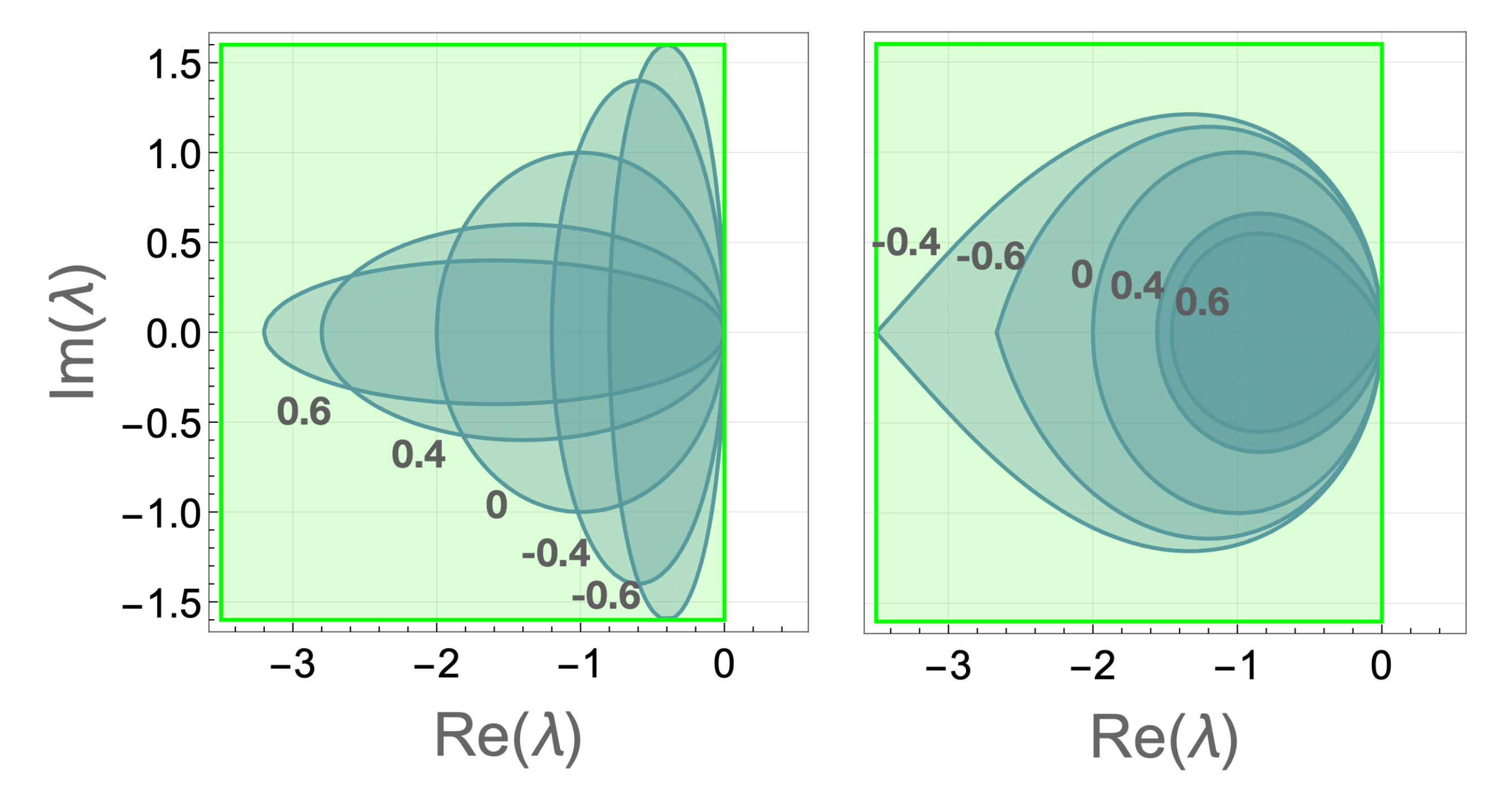}
    \caption{Convergence regions of momentum methods with different momentum parameter $\b$: ({\bf left}) HB$(\a, \b)$; ({\bf right}) NAG$(\a, \b)$. We take $\b = 0, \pm 0.4, \pm 0.6$ (as shown in the figure).  The green region represents the one where the eigenvalues of $\sp({\bf H}_{\a_1, \a_2})$ at local saddle points may occur. }
    \label{fig:hb_nag_conv}
\end{figure}

\subsection{Nesterov's accelerated gradient (NAG)}
Nesterov's accelerated gradient \citep{nesterov1983method} is a variant of Polyak's heavy ball, which achieves the optimal convergence rate for convex functions. It has been widely applied in deep learning \citep{sutskever2013importance}. In \cite{bollapragada2019nonlinear}, the authors analyzed the spectrum of NAG using numerical range in the context of linear regression, which is equivalent to the case when $\sp(\Hv)\subset \R$ (cf.~\citet[p.~11]{bollapragada2019nonlinear}). 

The key difference between HB and NAG is  the order of momentum update and the gradient update. We study Nesterov's momentum for minimax optimization:
\be\label{eq:nag_alg}
\zv_{t+1} = \zv'_t + \a \vv(\zv'_t), \, \zv'_t = \zv_t + \b (\zv_t - \zv_{t-1}),
\en
which we denote as NAG$(\a_1, \a_2, \b)$. We have the following stability result for NAG:
\begin{thm}[\textbf{NAG}]\label{thm:nag}
${\rm NAG}(\a_1, \a_2, \b)$ is exponentially stable iff for any $\l \in \sp(\haa)$: 
\begin{align}
\label{eq:nag_1} 
 |1 + \l|^{-2} > 1 + 2 \b (\b^2 - \b - 1)\Re(\l) + \b^2 |\l|^2 (1+2\b), \, |\b|\cdot |1+\l| < 1.
\end{align}
\end{thm}

\begin{proof} With state augmentation $\zv_t \to (\zv_{t+1}, \zv_t)$, the Jacobian for NAG is:
\be
\begin{bmatrix}
(1+\b) ({\bf I}_{n+m} + \haa) & -\b ({\bf I}_{n+m} + \haa) \\
{\bf I}_{n+m} & \zero
\end{bmatrix}.\nonumber
\en
The spectrum can be computed as:
\be
\sp({\bf J}(f)) &=&\{w : p(w):=w^2 - w(1+\b)(1+\l) + \b(1+\l) = 0, \l \in \haa\}.\nonumber
\en
Comparing with \eqref{eq:hb_expand}, we find that the two characteristic polynomials are different only by $O(\a \b)$. With Lemma \ref{lem:2}, the condition for local linear convergence is:
\be
\label{eq:nag_schur_1}&&|1 + \l|^{-2} >   1 + 2 \b (\b^2 - \b - 1)\Re(\l) + \b^2 |\l|^2 (1+2\b),\\
\label{eq:nag_schur_2}&&|\b|\cdot |1+ \l| < 1.
\en
\end{proof}

 From \Cref{fig:hb_nag_conv}, the convergence region of NAG is better conditioned than HB. However, NAG is still similar to HB and GDA in terms of the local convergence behavior:

\begin{cor}[\textbf{NAG}]\label{cor:nag}
If $\Re(\l) \geq 0$ for some $\l \in \haa$, then ${\rm NAG}(\a_1, \a_2, \b)$ is not exponentially stable.
\end{cor}

\begin{proof}
Take $\l \in \haa $ and assume $\l = u + i v$ with $u, v\in \R$. \eqref{eq:nag_1} can be translated to the following Mathematica code:
\begin{verbatim}
Reduce[b^2 ((1 + u)^2 + v^2) < 1 && ((1 + u)^2 + v^2) (1 + 
2 b (b^2 - b - 1) u + b^2 (u^2 + v^2) (1 + 2 b)) < 1 && u >= 0],
\end{verbatim}
and the result is \texttt{False}.
\end{proof}

According to \Cref{lem:loc_sadl},  NAG$(\a_1, \a_2, \b)$ never converges on bilinear games. Summarizing the previous subsections, we conclude that adding momentum does not help in converging to local saddle points.


\section{Proofs in \Cref{sec:local_stab}}\label{app:proof_stability}

\peg*

\begin{proof}
\blue{
From the second equation of \eqref{eq:peg} we obtain
\begin{align}
\zv_{t+3/2} &= \zv_{t+1} + \vv(\zv_{t+1/2}) \tr
&= \zv_{t} + \left(1 + \frac{1}{\beta}\right)\vv(\zv_{t+1/2}) + \vv(\zv_{t-1/2}) - \vv(\zv_{t-1/2}) \tr
&=\zv_{t+1/2} + \left(1 + \frac{1}{\beta}\right)\vv(\zv_{t+1/2})  - \vv(\zv_{t-1/2}).
\end{align}
In the second line we used the first equation of \eqref{eq:peg} and in the third line we used the second equation of \eqref{eq:peg}. 
}
\end{proof}

\thmeg*

\begin{proof}
From \eqref{eq:eg_alg} the update of EG can be rewritten as $\zv_{t+1} = \zv_t + \vv(\zv_t + \vv(\zv_t))/\beta$. \blue{We compute the Jacobian matrix of this update:
$$
\Jv = {\bf J}(f) = {\I} + \haa /\b + \haa^2/\b.
$$
It then follows that
$
\sp({\bf J}) = 1 + \sp(\haa)/\b +  \sp(\haa)^2/\b,
$
}
where the operation is element-wise. Therefore, $\rho({\bf J}(f)) < 1$ iff 
$$
\max_{\l \in \haa} |1 + \l/\b + \l^2/\b| < 1.
$$
Similarly for OGD, the spectrum can be computed as:
\be
\sp({\bf J}_{\rm OGD}) &=&\{x : p(x):=x^2 - (1 + k \l) x +\l = 0,\,\l \in \haa\}.
\en
With \Cref{lem:2}, we obtain the necessary and sufficient conditions when the roots of $p(x)$ are in the unit circle:
\be
&&|\l| < 1, \, (k - 1)|\l|^2 (k - 3 + (k + 1) |\l|^2) <2(k - 1)\Re(\l)(k|\l|^2 - 1),\,\forall \l \in \haa.  \nonumber
\en
\end{proof}

\egratio*

\begin{proof}
Rewriting $\l = x + i y$ with $x, y\in \R$ for $\l \in \haa$ and using \Cref{thm:eg_stable}, we run the following Mathematica code ($b_1 \equiv \b_1, \, b_2\equiv \b_2$):
\begin{verbatim}
    Reduce[ForAll[{x, y, b1, b2}, ((y + 2 x y)/b2)^2 + 
      (1 + (x + x^2 - y^2)/b2)^2 < 1 && b1 > b2 > 1, 
      ((y + 2 x y)/b1)^2 + (1 + (x + x^2 - y^2)/b1)^2 < 1]]
\end{verbatim}
The answer is \texttt{True}. For the second part, we rewrite the stability condition for OGD as:
\be
k|\l|^2(1 + |\l|^2 - 2\Re(\l)) < 3|\l|^2 - |\l|^4 - 2 \Re(\l).
\en
Since $\Re(\l) \leq |\l|$, $1 + |\l|^2 - 2\Re(\l)\geq 0$. The left hand side increases with $k$. 
\end{proof}


From \Cref{thm:hb} and \Cref{thm:eg_stable} we can easily infer the relation among the stable sets of gradient algorithms:

\begin{restatable}[]{cor}{corgdeg}\label{thm:gd_eg_2}
Given $|\l| < 1$ with $\l \in {\bf H}_{\a_1, \a_2}$, whenever GDA$(\a_1, \a_2)$ converges, EG$(\a_1, \a_2, 1)$ converges as well. Given $|\l| < 1/\sqrt{3}$ with $\l \in {\bf H}_{\a_1, \a_2}$, whenever GDA$(\a_1, \a_2)$ converges, OGD$(2, \a_1, \a_2)$ converges.
\end{restatable}

\begin{proof}
When $\b = 0$, \eqref{eq:n_mom} becomes $|1 + \l| < 1$. The first part follows from:
\be
|1 + \l| < 1 \mbox{ and }|\l | < 1 \Longrightarrow |1 + \l + \l^2| < 1.
\en
Taking $k = 2$, from \Cref{thm:eg_stable}, the stability condition for OGD is:
\be\label{eq:ogd2}
|\l|^2 (-1 + 3|\l|^2) < 2\Re(\l)(2|\l|^2 - 1).
\en
We want to show that for all $|1+\l| < 1$ and $|\l| < 1/\sqrt{3}$, \eqref{eq:ogd2} holds, and thus we define $\l = u + i v$ ($u, v\in \R$) and use the following Mathematica code:
\begin{verbatim}
Reduce[ForAll[{u, v},  (1 + u)^2 + v^2 < 1 &&  u^2 + v^2 < 1/3, 
(u^2 + v^2) (-1 + 3 (u^2 + v^2)) < 2 u (-1 + 2 (u^2 + v^2))]]
\end{verbatim}
This result is \texttt{True}.
\end{proof}

\lemlocsadl*

\begin{proof}
The convergence analysis reduces to the spectral study of ${\bf H}_{1,\g}$. With the similarity transformation:
\be
&&{\bf H'} ={\bf U}^{-1}{\bf H}_{1,\g} {\bf U} = \begin{bmatrix}
- \n_{\xv \xv}^2 f & -\sqrt{\g} \n_{\xv \yv}^2 f \\
\sqrt{\g} \n_{\yv \xv}^2 f & \g \n_{\yv \yv}^2 f
\end{bmatrix}, \, {\bf U} = \begin{bmatrix}
{\bf I} & \zero \\
\zero & \sqrt{\g}{\bf I} 
\end{bmatrix},
\en
It suffices to study the spectrum of ${\bf H'}$. For any local saddle point $(\xv^*, \yv^*)$, we have:
\be
\n_{\xv \xv}^2 f(\xv^*, \yv^*) \succeq {\bf 0}, \, \n_{\yv \yv}^2 f(\xv^*, \yv^*) \preceq {\bf 0}.
\en
\blue{From this necessary condition, $\Re({\bf H'}) := ({\bf H'} + {\bf H'}^{\top})/2$ is negative semi-definite, and with the Ky Fan inequality (\cite{fan1950theorem}) we have $\Re (\sp({\bf H'})) \prec \sp(\Re({\bf H'})) \prec {\bf 0}$, with ``$\prec$'' meaning majorization \citep{marshall1979inequalities}.} The second part can be proved by assuming $z = -u + i v$ with $u\geq 0$ and $v\in \R$. The quadratic function can be $$q = \frac{ux^2}{2} - \frac{uy^2}{2\g} + \frac{v}{\sqrt{\g}} xy,$$
since one can verify that $(0, 0)$ is a local saddle point where:
\be
{\bf H}_{1, \g} =  \begin{bmatrix}
-u & -v/\sqrt{\g} \\
v\sqrt{\g} & -u
\end{bmatrix},
\en
whose two eigenvalues are $z$ and $\bar{z}$. For bilinear games $f = \xv^\top {\bf C} \yv + \av^\top \xv + \bv^\top \yv$, at any local saddle point, \blue{the Jacobian matrix of the vector field} is:
\be
{\bf H}_{1, \g} = \begin{bmatrix}
\zero & -{\bf C} \\
\g {\bf C}^\top & \zero
\end{bmatrix}.
\en
The eigenvalues are $\l = \pm i\sqrt{\g}\s$, with $\s$ a singular value of $\Cv$. 
\end{proof}

\coregsadl*

\begin{proof}
At a local saddle point, from \Cref{lem:loc_sadl}, for any $\l \in \sp({\bf H})$, $\Re(\l) \leq 0$. \blue{The corollary follows with $0 < |\l| < 1/\a$ for every $\l \in \sp({\bf H})$} and \Cref{thm:eg_stable}, since if $\b = 1$, we can show:
\be
\Re(\l)\leq 0 \mbox{ and }0 < |\l| < 1 \Longrightarrow |1 + \l + \l^2| < 1,
\en
with the following Mathematica code (rewrite $\l = u + i v$ with $u, v \in \R$):
\begin{verbatim}
    Reduce[ForAll[{u, v}, u <= 0 && 0 < u^2 + v^2 < 1, (v + 2 u v)^2
    + (1 + u + u^2 - v^2)^2 < 1]],
\end{verbatim}
and the result is \texttt{True}. For OGD, if $1 < k \leq 2$, we use \Cref{thm:eg_stable}, \Cref{lem:loc_sadl}, and the following Mathematica code (rewrite $\l = u + i v$ with $u, v \in \R$):

\begin{verbatim}
    Reduce[ForAll[{u,v,k}, 0 < u^2+v^2<1/k^2 && u<=0 && 1<k<=2,
    (u^2+v^2)(-3+k+(1+k)(u^2+v^2)) <2u(-1+k(u^2+v^2))]].
\end{verbatim}
The result is \texttt{True}. If $k\geq 3$ and the game is bilinear, from \Cref{thm:eg_stable}, \Cref{thm:egratio} and \Cref{lem:loc_sadl} we must have $4|\l|^4 < 0$ to obtain local convergence, which is obviously false.
\end{proof}

\lemloc*

\begin{proof}
\blue{Let us assume} $z = u + i v$ with $(u, v)\in \R^2$. We first construct a real polynomial:
\be\label{eq:poly_1}
(\l - z)(\l - \bar{z}) = \l^2 - 2u \l + u^2 + v^2 = 0.
\en
On the other hand, the characteristic polynomial of $\Hv_{\a_1, \a_2}(q)$ with $q(x, y) = ax^2/2 + by^2/2 + cxy$ is:
\be\label{eq:poly_2}
\l^2 + (\a_1 a - \a_2 b)\l + \a_1 \a_2 (c^2 - a b) = 0.
\en
Comparing \eqref{eq:poly_1} and \eqref{eq:poly_2}, it suffices to require that:
\be
\a_1 a - \a_2 b = -2u, \, \a_1 \a_2 (c^2 - ab) = u^2 + v^2,
\en
which always has real solutions given $(\a_1 > 0, \a_2 > 0, u, v)$.
\end{proof}

\thmnegy*

\begin{proof}
Assume $\xv\in \R^n$ and Using Lemma 36 of \citet{jin2019minmax}, for any $\d > 0$, there exists $\g_0 > 0$, when $\g > \g_0$, the eigenvalues of ${\bf H}(1/\g, 1)$, $\l_1, \dots, \l_n, \l_{n+1}, \dots, \l_{m+n}$, are:
\be\label{eq:close}
|\l_i + \mu_i/\g| < \d/\g,\, \forall i = 1, \dots, n,|\l_{i + n} - \nu_i| < \d, \, \forall i = 1, \dots, m,
\en
where $\mu_i \in \sp(\n_{\xv \xv}^2 f - \n_{\xv \yv}^2 f (\n_{\yv \yv}^2 f)^{-1}\n_{\yv \xv}^2 f)$ and $\nu_i \in \sp(\n_{\yv \yv}^2 f)$. From our assumption, $\mu_i > 0$ and $\nu_i < 0$. With \eqref{eq:close}, there exists $\g_0$ such that for every $\g > \g_0$, $\Re(\l_i) < 0$ for all $\l_i \in H(1/\g, 1)$. From \Cref{cor:eg_sadl}, EG ($\beta = 1$) and OGD ($1 < k \leq 2$) are exponentially stable if $\a_2$ is small enough.
\end{proof}

\thmall*

\begin{proof}
We consider $q(x, y):= -x^2 + xy$ as the example, with $\X = \Y = \R$. From \eqref{thm:local_global_q} we know that $(0, 0)$ is a global minimax point. $(0, 0)$ is also local minimax since it is stationary (see \Cref{thm:qc}). $\Hv_{1, \g}$ at $(0, 0)$ is:
\be
\Hv_{1, \g} = \begin{bmatrix}
2 & -1 \\
\gamma & 0
\end{bmatrix}.
\en
If $0 < \g\leq 1$, the two eigenvalues are $1\pm \sqrt{1 - \g}$ which are both real and positive. One can read from \Cref{thm:hb} (or \Cref{fig:hb_nag_conv}) and \Cref{thm:eg_stable} (or \Cref{fig:eg_ogd_saddle}) that GDA (with momentum) and EG/OGD do not converge to $(0, 0)$, locally and globally. Specifically, when $\g = 1$, $\a_1 = \a_2$.

If $\gamma > 1$, the eigenvalues are $\l_{1, 2} = 1\pm i\sqrt{\gamma - 1}$, which have positive real parts. From \Cref{thm:hb} (or \Cref{fig:hb_nag_conv}), GDA (with momentum) do not converge to $(0, 0)$. Now let us study 2TS-EG and 2TS-OGD, \blue{which corresponds to the second point of \Cref{thm:all}.}

\paragraph{2TS-EG} Taking $\b\to \infty$ we require that $\Re(\l + \l^2) < 0$, which simplifies to:
\be
\a_1 + \a_1^2 - \a_1^2 (\g - 1) < 0,
\en
and thus
\be
\a_2 > 1 + 2\a_1 > 1.
\en
We cannot take $\a_2$ to be arbitrarily small. 

\paragraph{2TS-OGD} For 2TS-OGD, we need $\a_2$ to be $\Omega(1)$ as well. From \Cref{thm:eg_stable}, we take $k \to 1_+$ so that the convergence region is the largest:
\be\label{eq:ogd_large}
|\l| < 1, \, |\l -1/2| > 1/2.
\en
Bringing in the eigenvalues $\a_1(1\pm i\sqrt{\g -1})$, we obtain:
\be\label{eq:simple_condition_ogd}
\a_1 < 1, \, 1/\a_1 < \g < 1/\a_1^2.
\en
In other words, $1 < \a_2 < 1/\a_1$. We could take $\a_1$ infinitesimal but not $\a_2$. 

\paragraph{\blue{Alternating updates}} \blue{Now let us study alternating updates on this example. We use the same framework as \citet{zhang2019convergence}. If a simultaneous algorithm takes the form of: 
\begin{align}
\xv_{t} = T_1(\xv_{t-1}, \yv_{t-1}, \dots, \xv_{t-k}, \yv_{t-k}), \, 
\yv_t = T_2(\xv_{t-1}, \yv_{t-1}, \dots, \xv_{t-k}, \yv_{t-k}),
\end{align}
then the corresponding alternating algorithm is:
\begin{align}
\xv_{t} = T_1(\xv_{t-1}, \yv_{t-1}, \dots, \xv_{t-k}, \yv_{t-k}), \, 
\yv_t = T_2(\xv_{t}, \yv_{t-1}, \dots, \xv_{t-k + 1}, \yv_{t-k}),
\end{align}
by replacing all the $\xv_{t-i}$ in the update function for $\yv_t$ to $\xv_{t+1-i}$, for $i = 1, \dots, k$. We only study GDA and OGD in this paper for illustration purpose and other gradient algorithms follow similarly. The alternating GDA can be written as ($\alpha_1 > 0$, $\alpha_2 > 0$):
\begin{align}
\xv_{t+1} = \xv_t - \a_1 \partial_\xv f(\xv_t, \yv_t), \, \yv_{t+1} = \yv_t + \a_2 \partial_\yv f(\xv_{t+1}, \yv_t),
\end{align}
and the alternating OGD can be written as (see \eqref{eq:ogd})($\alpha_1 > 0$, $\alpha_2 > 0$, $k > 1$):
\begin{align}
&\xv_{t+1} = \xv_t - k\a_1 \partial_\xv f(\xv_t, \yv_t) + \a_1 \partial_\xv f(\xv_{t-1}, \yv_{t-1}), \\
&\yv_{t+1} = \yv_t + k \a_2 \partial_\yv f(\xv_{t+1}, \yv_t) - \a_2 \partial_\yv f(\xv_{t}, \yv_{t-1}).
\end{align}
Let us denote $\Av = \xxv f(\xv^*, \yv^*)$, $\Bv = \yyv f(\xv^*, \yv^*)$ and $\Cv = \xyv f(\xv^*, \yv^*)$. Locally, we can treat the gradient algorithms as a linear dynamical system. For instance, the linear dynamical system of simultaneous GDA and simultaneous OGD can be written as:
\begin{align}
\textrm{GDA: } \begin{pmatrix} \xv_{t+1} - \xv^* \\ \yv_{t+1} - \yv^* \end{pmatrix} &= \begin{pmatrix} \xv_{t} - \xv^* \\ \yv_{t} - \yv^* \end{pmatrix} + \begin{pmatrix} -\alpha_1 \Av & -\alpha_1 \Cv \\ \alpha_2 \Cv^\top & \alpha_2 \Bv \end{pmatrix} \begin{pmatrix} \xv_{t} - \xv^* \\ \yv_{t} - \yv^* \end{pmatrix}, \\
\textrm{OGD: } \begin{pmatrix} \xv_{t+1} - \xv^*  \\ \yv_{t+1} - \yv^* \end{pmatrix} &= \begin{pmatrix} \xv_{t} - \xv^* \\ \yv_{t} - \yv^* \end{pmatrix} + k\begin{pmatrix} -\alpha_1 \Av & -\alpha_1 \Cv \\ \alpha_2 \Cv^\top & \alpha_2 \Bv \end{pmatrix} \begin{pmatrix} \xv_{t} - \xv^* \\ \yv_{t} - \yv^* \end{pmatrix} - \tr
& -  \begin{pmatrix} -\alpha_1 \Av & -\alpha_1 \Cv \\ \alpha_2 \Cv^\top & \alpha_2 \Bv \end{pmatrix} \begin{pmatrix} \xv_{t - 1} - \xv^* \\ \yv_{t - 1} - \yv^* \end{pmatrix}.
\end{align}
With Theorem 2.3 from \cite{zhang2019convergence}, the characteristic equations for alternating GDA and alternating OGD are:
\begin{align}\label{eq:char_poly_gda_ogd}
&\textrm{GDA: }\det\left( (\lambda - 1)\Iv -  \begin{pmatrix} -\alpha_1 \Av & -\alpha_1 \Cv \\ \alpha_2 \lambda \Cv^\top & \alpha_2 \Bv \end{pmatrix} \right) = 0, \\
&\textrm{OGD: }\det\left( (\lambda - 1)\lambda \Iv - (k\lambda - 1) \begin{pmatrix} -\alpha_1 \Av & -\alpha_1 \Cv \\ \alpha_2 \lambda \Cv^\top & \alpha_2 \Bv \end{pmatrix} \right) = 0.
\end{align}
For the quadratic example $q(x, y) = -x^2 + x y$ we are considering, we have $\Av = -2, \Bv = 0, \Cv = 1$. Bringing it to \eqref{eq:char_poly_gda_ogd}, we obtain: 
\begin{align}\label{eq:convergence_gda_ogd}
&\textrm{GDA: } \lambda ^2 + (\alpha_1 \alpha_2-2 \alpha_1 - 2)\lambda  + 2 \alpha_1+ 1 = 0, \\
&\textrm{OGD: }\lambda ^4 + \left(\alpha_1 \alpha_2 k^2-2 \alpha_1 k-2\right) \lambda^3  + (2 \alpha_1-2 \alpha_1 \alpha_2 k+2 \alpha_1 k+1)\lambda^2 +  (\alpha_1 \alpha_2-2 \alpha_1) \lambda = 0.
\end{align}
From Corollary 2.1 of \cite{zhang2019convergence}, alternating GDA is stable iff:
\be
2\a_1 + 1 < 1, \, |\a_1 \a_2 - 2\a_1 - 2| < 2\a_1 + 2.
\en
Note that the first condition can never hold since $\a_1 > 0$. Hence, alternating GDA cannot converge to the local minimax point $(0, 0)$ if the initialization is not at $(0, 0)$. For alternating OGD, the second equation of \eqref{eq:convergence_gda_ogd} can be simplified as $\lambda = 0$ or:
\begin{align}
\lambda^3 + \left(\alpha_1 \alpha_2 k^2-2 \alpha_1 k-2\right) \lambda^2  + (2 \alpha_1-2 \alpha_1 \alpha_2 k+2 \alpha_1 k+1)\lambda + \a_1(\alpha_2 - 2) = 0.
\end{align}
Using Corollary 2.1 of \cite{zhang2019convergence} again we know that alternating OGD is stable iff:
\begin{align}
|c| < 1, \, |a + c| < 1 + b, \, b - a c < 1 - c^2,
\end{align}
where $a = \a_1 \a_2 k^2 - 2 \a_1 k - 2$, $b = 2 \a_1 - 2\a_1 \a_2 k + 2\a_1 k + 1$, $c = \a_1 (\a_2 - 2)$. We simplify it on Mathematica:} 
{\color{blue}
\begin{verbatim}
Reduce[Abs[c] < 1 && Abs[a+c] < 1 + b && b - a c < 1 - c^2 && k > 1 
&& \alpha_1 > 0 && \alpha_2 > 0, {\alpha_1, \alpha_2}]
\end{verbatim}
}
\noindent \blue{and obtain that:
\begin{align}
& k>1\mbox{ and } 0<\alpha_1<\frac{4}{k^2-1}\mbox{ and } \tr
& \sqrt{\frac{-2 \alpha_1+\alpha_1^2 k^2+1}{\alpha_1^2 (k+1)^2}}+\frac{2 \alpha_1+\alpha_1 k-1}{\alpha_1 (k+1)}<\alpha_2<\frac{4 \alpha_1+4 \alpha_1 k+4}{\alpha_1+\alpha_1 k^2+2 \alpha_1 k}.
\end{align}
Since $k > 1$ and 
\begin{align}
\sqrt{\frac{-2 \alpha_1+\alpha_1^2 k^2+1}{\alpha_1^2 (k+1)^2}}+\frac{2 \alpha_1+\alpha_1 k-1}{\alpha_1 (k+1)} &\geq \sqrt{\frac{-2 \alpha_1+\alpha_1^2+1}{\alpha_1^2 (k+1)^2}}+\frac{2 \alpha_1+\alpha_1 k-1}{\alpha_1 (k+1)} \tr
&= \frac{\a_1 k + 2\a_1 - 1 + |\a_1 - 1|}{\a_1 (k + 1)} \tr
&\geq \frac{\a_1 k + 2\a_1 - 1 + 1 - \a_1}{\a_1 (k + 1)} \tr
&= 1,
\end{align}
we have $\a_2 > 1$ for alternating updates of OGD.}
\end{proof}


\section{Local robust points}\label{app:local_robust_point}

\blue{
In this section, we summarize results about local robust points, which naturally extend local minimax points to a symmetric version. They are stationary points (\Cref{prop:stable_new}), but they may not correspond to solution concepts in sequential games (\Cref{eg:glp}). In one-dimensional case they are equivalent to the stable sets of Optimistic Gradient Descent (\Cref{prop:lrp_alg}). However, in general cases all common coordinate-independent gradient algorithms would fail to converge to some local robust point (\Cref{prop:failure_lrp}).
The main results are summarized in \Cref{table:lrp}. 
}

\begin{table*}
\renewcommand{\arraystretch}{1.3}
\centering
\begin{tabular}{|c|c|c|} 
\hline
&\textbf{Statement} & \textbf{Reference} \\ \hline
\hline
& \cellcolor[gray]{.9} non-trivial examples & Prop.~\ref{prop:lrp}, Eg.~\ref{eg:glp} \\
\cline{2-3} & \cellcolor{lightgray} nuances in the definition & Examples \ref{eg:glp_nbhr}, \ref{eg:lrp_eps_0} \\
\cline{2-3} & \cellcolor[gray]{.9} LRPs are stationary points & \Cref{prop:stable_new} \\
\cline{2-3} LRP & \cellcolor{lightgray} optimality conditions & Theorems \ref{prop:stable_new}, \ref{thm:suff_lrp}, \ref{thm:2nd_nece_lrp}, \ref{thm:lrp_2nd_suff} \\
\cline{2-3} & \cellcolor[gray]{.9} LRP in quadratic games & \Cref{thm:lrp_quadratic} \\
\cline{2-3} & \cellcolor{lightgray} equivalence with the stable set of OGD in 1D & Prop.~\ref{prop:lrp_alg} \\
\cline{2-3} &  \cellcolor[gray]{.9} failure of gradient algorithms at LRP & \Cref{prop:failure_lrp} \\
\hline
\end{tabular}
\caption{Results of local robust points.}\label{table:lrp}
\end{table*}

\subsection{Definition of local robust points}\label{sec:def_lrp}


In the definition of local minimax points, $\xv$ and $\yv$ are \emph{asymmetric}: $\yv$ is the follower who knows the strategy of $\xv$, but $\xv$ only knows a ``rough'' set of the strategies of $\yv$ and hence aims to optimize the worst-case scenario. One natural (and perhaps more realistic) generalization is to allow robust optimization for $\yv$ as well, so as to restore equal position for both players:

\begin{restatable}[{\glp}]{defn}{}\label{def:glp}
We call $(\xv^\star, \yv^\star) \in  \X\times \Y$ a local robust point ({\glp}) if
\begin{itemize}
    \item fixing $\xv^\star$, there exists some sequence $0 \leq \varepsilon_n \to 0$ such that for each $\varepsilon_n$ in the sequence, there exists an envelope function $\uf_{\varepsilon_n, \xv^\star}(\yv)$ such that $\yv^\star$ is a local maximizer; 
    
    \item fixing $\yv^\star$, there exists some sequence $0 \leq \e_n \to 0$ such that for each $\e_n$ in the sequence, there exists an envelope functions \blue{$\of_{\e_n, \yv^\star}(\xv)$} such that $\xv^\star$ is a local minimizer.
\end{itemize}
\end{restatable}

In the above definition, both $\xv$ and $\yv$ are doing robust optimization: $\of_\e(\xv)$ and $-\uf_\ve(\yv)$ can be treated as the worst-case cost for each player, assuming that each one only knows an approximate strategy of the opponent ($\xv^\star$ or $\yv^\star$), up to some estimation error ($\e$ or $\ve$). Since each player does not know the exact amount of perturbation, it will try to minimize a sequence of envelope functions \blue{with a series of neighborhoods that can be arbitrarily small.} 

\begin{figure}
    \centering
    \includegraphics[width=10cm]{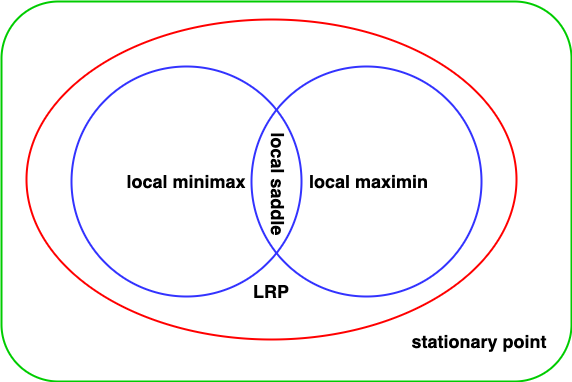}
    \caption{The relation among the sets of local saddle, local minimax and local maximin points, as well as {\glp}s. In the unconstrained case, they are all stationary (\Cref{prop:stable}). }
    \label{fig:glp}
\end{figure}

LRPs are a subclass of stationary points, as we will see in \Cref{prop:stable_new}. The definition of {\glp}s includes local saddle, local minimax and local maximin points, as visualized in \Cref{fig:glp}. For example, if $\{\ve_n\} = \{0\}$ and $0 < \e_n \to 0$, then {\glp} reduces to local minimax points. The simplest non-trivial example for LRPs might be quadratic games. In general for one-dimensional quadratic games, it can be shown that:

\begin{prop}[\blue{\textbf{characterization of LRPs in one-dimensional quadratic games}}]\label{prop:lrp}
$f(x, y) = ax^2/2 + c x y + b y^2/2$ has an {\glp} at $(0, 0)$ iff \be
\{c = 0, \, a\geq 0 \geq b\}\mbox{ or }\{c\neq 0,\, c^2 \geq a b\}.
\en
\end{prop}
\begin{proof}
If $c = 0$, $f$ is separable, we obtain $a\geq 0$ because $x^\star$ locally minimizes $\of_\e(x)$, and $b\leq 0$ since $y^\star$ locally maximizes $\of_\varepsilon(y)$. If $c\neq 0$, then for small enough $x$, $y$,
\be\label{eq:glp_quadratic}
\of_\e(x) = \begin{cases}
|c x|\e + b \e^2/2 + a x^2/2 & \mbox{if }b\geq 0\\
(c^2 - ab)x^2/(-2b) & \mbox{if }b<0
\end{cases},\;
\uf_\varepsilon(y) = \begin{cases}
-|c y|\varepsilon + b y^2/2 + a \varepsilon^2/2 & \mbox{if }a\leq 0\\
-(c^2 - ab)y^2/(2a) & \mbox{if }a>0
\end{cases}.
\en
From the above, we can show that it is necessary and sufficient to have $c^2 \geq ab$: if $c^2 \geq ab$, then $\of_\e(x)$ is locally minimized at $x = 0$ and $\uf_\ve(y)$ is locally maximized at $y = 0$; if $c^2 < a b$, then $a > 0, b > 0$, when $\uf_\ve(y)$ is not locally maximized at $y = 0$, or $a < 0, b < 0$, when $\of_\e(x)$ is not locally minimized at $x = 0$. 
\end{proof}

If $c = 0$ and $a = -2, \, b = 2$, then this quadratic function clearly does not have an {\glp} (but has a stationary point), which implies the non-triviality of our definition. Another interesting case is when $a = -2, c = 1$ and $b = 2$:
\begin{eg}[\blue{\textbf{LRPs may not be either local minimax or maximin}}]\label{eg:glp}
Consider $f(x, y) = -x^2 + x y + y^2$ and $(x^\star, y^\star) = (0,0)$ \blue{with the domain $|x|\leq D, \, |y|\leq D$.} Straightforward calculation gives (\blue{assuming $0< \e \leq D$, $0< \ve \leq D$}):
\begin{align}
\of_\e(x) = -x^2 + \e |x| + \e^2, \quad \uf_\varepsilon(y) = -\varepsilon^2 - \varepsilon |y| + y^2.
\end{align}
Thus, $f$ has an {\glp} at $(0, 0)$, which is neither local minimax or local maximin: $f(0, y) = y^2$ is not locally maximized at $y = 0$ and $f(x, 0) = -x^2$ is not locally minimized at $x = 0$. \blue{Note that $(0, 0)$ is not a global minimax/maximin point either. However, we have:
\be
&&\of_D(x) = \max_{|y|\leq D} f(x, y)  = -x^2 + D|x| + D^2 \geq \of_D(0), \mbox{ for all } |x|\leq D\tr
&&\uf_D(y)  = \min_{|x|\leq D} f(x, y) = -D^2 - D|y| + y^2 \leq \uf_D(0), \mbox{ for all } |y|\leq D.
\en 
So $(0, 0)$ can be treated as some type of ``global robust point'', defined as 
\be
&&\sup_{\yv\in \Yc}f(\xv, \yv) \geq \sup_{\yv\in \Yc}f(\xv^\star, \yv),\mbox{ for any }\xv \in \Xc \\ 
&&\inf_{\xv\in \Xc}f(\xv, \yv) \leq \inf_{\xv\in \Xc}f(\xv, \yv^\star), \mbox{ for any }\yv \in \Yc.
\en
In such a game, each player is agnostic of the opponent's strategy and only optimizing the worst case. There is no follower or leader. Such study goes beyond the regime of sequential games and we leave it to future research.  
}
\end{eg}

However, for {\glp}s, some results we derived in \Cref{sec:existing} for local minimax points cease to hold anymore. For example, for local minimax points the norm we choose in the neighborhood definition is immaterial (see \Cref{thm:eqseq}), but for {\glp}s, that choice of the neighborhoods does matter, as can be seen from the following example:
\begin{eg}[\textbf{effect of the neighborhood}]\label{eg:glp_nbhr}
Consider the function
\be
f(\xv, \yv) = -\xv^\top \begin{bmatrix}
0 & 0 \\
0 & 1
\end{bmatrix}\xv + \xv^\top \begin{bmatrix}
1 & 0 \\
0 & 1
\end{bmatrix}\yv + \yv^\top \begin{bmatrix}
1 & 0 \\
0 & 0
\end{bmatrix}\yv,
\en
with $\X = \Y = \R^2$ and $(\xv^\star, \yv^\star) = (\zero, \zero)$. For the $\ell_\infty$ normed ball $\Nc_\infty(\yv^\star, \e) = \{\yv \in \R^2: \|\yv - \yv\|_\infty\leq \e\}$, $\of_\e(\xv) = \e^2 + \e |x_1| + \e |x_2| - x_2^2$ which is locally minimized at $\xv^\star$. However, for the Euclidean ball $\Nc_2(\yv^\star, \e) = \{\yv \in \R^2: \|\yv - \yv\|_2\leq \e\}$, $$\of_\e(0, x_2) = \max_{\yv\in \Nc_2(\yv^\star, \e)}x_2 y_2 + y_1^2 - x_2^2\leq \max_{|y_2|\leq \e} \e^2 - y_2^2 + x_2 y_2 - x_2^2 \leq \e^2 - 3x_2^2/4 < \of_\e(0, 0) = \e^2,$$ for any $0 < |x_2| < 2\e$. One can show that $(\xv^\star, \yv^\star) = (\zero, \zero)$ is an {\glp} by choosing the neighborhoods of $\xv^\star$ and $\yv^\star$ to be $\ell_\infty$ balls, since 
$$\of_\e(\xv) = \e^2 + \e |x_1| + \e |x_2| - x_2^2\geq \of_\e(\zero)\mbox{ locally and }\uf_\ve(\yv) = -\ve^2 - \ve |y_1| - \ve |y_2| + y_1^2\leq \uf_\ve(\zero)$$
locally. In \Cref{sec:quad_lrp} we will show a ``meaningful'' neighborhood choice for {\glp}s in quadratic games using the eigenspace. 
\end{eg}

\noindent In order for the class of LRPs to include the class of local minimax points, we may no longer take $\{\e_n\}$ and $\{\ve_n\}$ to be \blue{strictly {positive}} sequences as in Def.~\ref{def:jin}:

\begin{eg}[\blue{\textbf{The definition of LRPs need to include $\e = 0$ and $\varepsilon = 0$}}]\label{eg:lrp_eps_0}
Take $$f(x, y) = xy^3 - x^2/(1+y^2)$$ and $(x^\star, y^\star) = (0, 0)$. This point is a local minimax point, since $\uf_0(y) = f(x^\star, y) = 0$, and $\of_\e(x) \geq \e^3 |x| - x^2/(1+\e^2) \geq 0 = \of_\e(x^\star)$, given small enough $x$. However, for any $\ve > 0$, 
$$\uf_\varepsilon(y) = -\varepsilon |y|^3 - \varepsilon^2/(1 + y^2)\mbox{ and }\uf_\varepsilon(y) - \uf_\ve(y^\star) = \ve y^2(\ve/(1+y^2) -  |y|) > 0$$ for small enough $y$. \blue{Therefore, in Definition \ref{def:glp} the case of $\varepsilon = 0$ needs to be included, as otherwise $(x^\star, y^\star) = (0, 0)$ does not satisfy the definition of LRPs, since for any $\varepsilon > 0$, the variable $y^\star$ cannot be a local maximizer of $\uf_\varepsilon$. }
\end{eg}

\subsection{Optimality conditions for LRPs}\label{app:lrps}

\blue{
Let us define the \emph{active sets} of the \emph{zeroth} order (by ``zeroth'' we mean that only the function values are involved):
\be
&&\,{\Y}_0(\xv^*;\e) = \{\yv\in\Nc(\yv^*, \e): \of_\e(\xv^*) = f(\xv^*, \yv)\}, \\
&&\,{\X}_0(\yv^*;\ve) = \{\xv\in\Nc(\xv^*, \ve): \uf_\ve(\yv^*) = f(\xv, \yv^*)\}.
\en
We derive the first-order optimality conditions for LRPs.
}

\begin{restatable}[first-order necessary, {\glp}]{thm}{firstlrp}\label{prop:stable_new}
Let $f\in \Cc^1$. At an {\glp} $(\xv^\star, \yv^\star)$, we have:
\be\label{eq:first_order_necessary_lrp_new}
\partial_{\xv} f(\xv^\star, \yv^\star)^\top \bar{\tv} \geq 0 \geq \partial_{\yv} f(\xv^\star, \yv^\star)^\top \ubar{\tv},
\en
for any directions $\bar{\tv} \in \K[d](\X, \xv^\star)$, $\ubar{\tv} \in \K[d](\Y, \yv^\star)$, where the cone
\begin{align}\K[d](\X, \xv) := \liminf_{\a\to 0^+} \frac{\X - \xv}{\a} := \{\tv: \forall \{\a_k\} \to 0^+~ \exists \{\a_{k_i}\} \to 0^+, \{\tv_{k_i}\} \to \tv, \tr
\blue{\mbox{ such that } \xv+\a_{k_i} \tv_{k_i} \in \X \}}\nonumber
\end{align}
and $\K[d](\Y, \yv)$ is defined similarly.
\end{restatable}

\begin{proof}
Use \Cref{thm:nec_1st}, \Cref{thm:danskin} and the assumption that $f\in \C^1$.
\end{proof}

\begin{thm}[\textbf{first-order sufficient condition, {\glp}}]\label{thm:suff_lrp}
If $f$ is continuously differentiable and there exist \blue{two sequences $\e_n \to 0$, $\ve_n\to 0$}, such that for any $n\in \mathbb{N}^+$:
\be
&&\zero \neq \bar{\tv}\in \K[c](\X, \xv^\star)  \, \Longrightarrow 
\D \of_{\e_n}(\xv^\star; \bar{\tv}) = \max_{\yv \in {\Y}_0(\xv^\star;\e_n)} \partial_{\xv} f(\xv, \yv)^\top \bar{\tv} > 0, \\
&&\zero \neq \ubar{\tv}\in \K[c](\Y, \yv^\star)  \, \Longrightarrow 
\D \uf_{\ve_n}(\yv^\star; \ubar{\tv}) =
\min_{\xv \in {\X}_0(\yv^\star;\ve_n)} \partial_{\yv} f(\xv, \yv)^\top \ubar{\tv} < 0.
\en
then $(\xv^\star, \yv^\star)$ is an isolated {\glp} of $f$.
\end{thm}

\blue{
We next discuss how to obtain second-order conditions for {\glp}s.}
Recalling \Cref{def:glp}, for the second-order optimality conditions of the local maximality of min-type envelope functions $\uf_\e(\yv)$, 
we can simply take $f\to -f$, $\of_\e(\xv)\to -\uf_\e(\yv)$ and switch the roles of $\xv$ and $\yv$. Let us define that:
\be
&&\bar{u}_\e(\yv) := \of_\e(\xv^\star) - f(\xv^\star, \yv),
\, \bar{v}(\yv; \tv) = -\n_\xv f(\xv^\star, \yv)^\top \tv, \tr
&&
{\Y}_1(\e; \tv) = \{\yv\in\Nc(\yv^\star, \e): \bar{u}_\e(\yv) = \bar{v}(\yv; \tv) = 0\},\\
&& \ubar{u}_\ve(\xv) := f(\xv, \yv^\star) - \uf_\varepsilon(\yv^\star),\, \ubar{v}(\xv; \tv) = \n_\yv f(\xv, \yv^\star)^\top \tv,
\tr&&
{\X}_1(\varepsilon; \tv) = \{\xv\in\Nc(\xv^\star, \e): \ubar{u}_\ve(\xv) = \ubar{v}(\xv; \tv) = 0\}, 
\en
and 
\be
\bar{E}_\e(\yv;\tv) = \limsup_{\zv \to \yv} \tfrac12 \bar{v}_-^2(\zv; \dv) \bar{u}_\e^\dag(\zv), \, \ubar{E}_\ve(\xv; \tv) = \limsup_{\zv\to \xv} {\ubar{v}_-(\zv; \tv)^2}\ubar{u}_\e^\dag(\zv)/2.
\en
We obtain the second-order necessary conditions for {\glp}s from \Cref{thm:kawasaki}:

\begin{thm}[\textbf{second-order necessary condition, {\glp}}]\label{thm:2nd_nece_lrp}
If $(\xv^\star, \yv^\star)$ is an {\glp} with sequence $\{\e_k\}, \{\ve_k\}$, then for any $\e_k$, for each direction $\bar{\tv} \in \R^n$, $\D\of_{\e_k}(\xv^\star; \bar{\tv}) > 0$, or  $\D \of_{\e_k}(\xv^\star; \bar{\tv}) = 0$ and there exist at most $n + 1$ points $\yv_1, \dots, \yv_{n+1}\in \Y_1(\e_k;\bar{\tv})$ and $\l_1, \dots, \l_n \geq 0$ not all zero, such that:
\be\label{eq:sec_order_1}
\sum_{i=1}^{n+1}\l_i \n_\xv f(\xv^\star, \yv_i) = {\bf 0}, \,
\sum_{i=1}^{n+1}\l_i \left(\bar{\tv}^\top \n_{\xv \xv}^2 f(\xv^\star, \yv_i)\bar{\tv} + \bar{E}_{\e_k}(\yv_i, \bar{\tv})\right)\geq 0.
\en
For each feasible direction $\ubar{\tv} \in \R^m$, $\D \uf_{\ve_k}(\yv^\star; \ubar{\tv}) < 0$, or $\D \uf_{\ve_k}(\yv^\star; \ubar{\tv}) = 0$ and there exist at most $m + 1$ points $\xv_1, \dots, \xv_{n+1}\in \X_1(\ve_k;\ubar{\tv})$ and $\mu_1, \dots, \mu_m \geq 0$ not all zero, such that:
\be\label{eq:sec_order_2}
\sum_{i=1}^{m+1}\mu_i \n_\xv f(\xv_i, \yv^\star) = {\bf 0}, \,
\sum_{i=1}^{m+1}\mu_i \left(\ubar{\tv}^\top \n_{\yv \yv}^2 f(\xv_i, \yv^\star)\ubar{\tv} - \ubar{E}_{\ve_k}(\xv_i, \ubar{\tv})\right)\leq 0.
\en
\end{thm}

\begin{rem}\label{rem:nec_true}
For LRPs we do not have the simplification as local minimax points in \Cref{thm:nec_localmm} since \Cref{lem:dec_f_p} does not necessarily hold. In fact, $\yv^\star$ may not even be in the active set $\Y_0(\xv^\star)$ (e.g.~\Cref{eg:glp}). Comparably, for a local minimax point $(\xv^\star, \yv^\star)$, $\yv^\star\in \Y_0(\xv^\star)$ and $\bar{u}_\e(\yv^\star)$ is a constant for small enough $\e$.  
\end{rem}

It is also possible to construct second-order sufficient conditions for {\glp}s from \Cref{thm:suf_kawa} and \Cref{thm:2nds}. 
We only construct one from \Cref{thm:suf_kawa} as the other construction is analogous. Similar to \Cref{assmp:suff}, we need the following assumption:
\begin{assmp}\label{assmp:suff_2}
For each $\xv\in {\X}_1(\varepsilon; \tv)$ with $\tv \neq \zero$ and $\D \uf_\varepsilon(\xv^\star; \tv) = 0$, and for each non-zero $\dv \in \R^m$, there exist $\a, \b \neq 0$ and $p, q > 0$ such that the following approximation holds:
\be
\ubar{u}_\ve(\xv+\d \dv) = \a \d^p + o(\d^p), \, \ubar{v}(\xv + \d \dv;\tv) = \b \d^q + o(\d^q),
\en
whenever $\xv+\d \dv \in \Nc(\xv^\star, \e)$ and $\d > 0$.
\end{assmp}
With this assumption and \Cref{assmp:suff} (with a slight change of notations) we can write down the second-order sufficient condition for {\glp}s, similar to \Cref{thm:2nd_nece_lrp}:

\begin{thm}[\textbf{second-order sufficient condition, {\glp}}] \label{thm:lrp_2nd_suff} 
Assume that \Cref{assmp:suff} and \Cref{assmp:suff_2} hold, and let $\X = \R^n$ and $\Y = \R^m$. Suppose there exists a sequence $\{\e_k\}$ such that for any $\e_k$, for each direction $\bar{\tv} \in \R^n$, $\D\of_{\e_k}(\xv^\star; \bar{\tv}) > 0$, or  $\D\of_{\e_k}(\xv^\star; \bar{\tv}) = 0$ and there exist $a\geq 1$ points $\yv_1, \dots, \yv_{a}\in \Y_1(\e_k; \bar{\tv})$ and $\l_1, \dots, \l_a \geq 0$ not all zero, such that:
\be\label{eq:sec_order_suff_1}
\sum_{i=1}^{a}\l_i \n_\xv f(\xv^\star, \yv_i) = {\bf 0}, \,
\sum_{i=1}^{a}\l_i \left(\bar{\tv}^\top \n_{\xv \xv}^2 f(\xv^\star, \yv_i)\bar{\tv} + \bar{E}_{\e_k}(\yv_i, \bar{\tv})\right) > 0.
\en
If moreover there exists a sequence $\{\ve_k\}$ such that for any $\ve_k$, along each $\ubar{\tv} \in \R^m$, $\D\uf_{\ve_k}(\yv^\star; \ubar{\tv}) < 0$, or $\D\uf_{\ve_k}(\yv^\star; \ubar{\tv}) = 0$ and there exist $b\geq 1$ points $\xv_1, \dots, \xv_{b}\in {\X}_1(\varepsilon_k; \tv)$ and $\mu_1, \dots, \mu_m \geq 0$ not all zero, such that:
\be\label{eq:sec_order_suff_2}
\sum_{i=1}^{b} \mu_i \n_\yv f(\xv_i, \yv^\star) = {\bf 0},\, \sum_{i=1}^{b}\mu_i \left(\ubar{\tv}^\top \n_{\yv \yv}^2 f(\xv_i, \yv^\star)\ubar{\tv} - \ubar{E}_{\ve_k}(\xv_i, \ubar{\tv})\right) < 0,
\en
then $(\xv^\star, \yv^\star)$ is an \glp.
\end{thm}

\subsection{Local robust points in quadratic games}\label{sec:quad_lrp}
\blue{In this subsection, we discuss the existence conditions for {\glp}s in quadratic games.} Since {\glp}s are also stationary, we can translate the origin such that the quadratic game is homogeneous. 

\begin{defn}[\blue{\textbf{positive/negative part of a symmetric matrix}}]\label{defn:pos_neg_part_matrix}
For an $n$-dimensional symmetric matrix $\Av\in \mathds{S}^n$, given its spectral decomposition $\Av = \Uv \Dv \Uv^\top$, we define the positive part $\Av_p = \Uv \Dv_p \Uv^\top$, and the negative part is $\Av_n = \Uv \Dv_n \Uv^\top$, where $[\Dv_p]_{i, j}=  d_{ii}\d_{i, j}{\bf 1}_{d_{ii} > 0}$ (resp.~$[\Dv_n]_{i, j} =  d_{ii}\d_{i, j}{\bf 1}_{d_{ii} < 0}$) is a diagonal matrix that takes the positive part (resp.~the negative part) of $\Dv$.
\end{defn}

\begin{defn}[\textbf{eigenspace neighborhood}]
Given the spectral decomposition of a symmetric matrix $\Av = \sum_i \l_i \vv_i \vv_i^\top$, we define the eigenspace neighborhood w.r.t.~$\Av$ as:
\be
\Nc_\Av(\xv, \e):= \{\xv + \sum_{i} c_i \vv_i: |c_i| \leq \e\}.
\en
\end{defn}
With the decomposition of symmetric matrices and the eigenspace neighborhoods, we can derive the condition for LRPs in unconstrained quadratic games:
\begin{restatable}[\blue{\textbf{necessary and sufficient conditions of {\glp}s in quadratic games}}]{thm}{lrpq}\label{thm:lrp_quadratic}
Let us choose $\Nc(\yv^\star,\e) = \Nc_\Bv(\yv^\star, \e)$ and $\Nc(\xv^\star, \ve) = \Nc_\Av(\xv^\star, \ve)$ for envelope functions $\of_\e(\xv)$ and $\uf_\ve(\yv)$ respectively. In order for $(\xv^\star, \yv^\star) = (\zero, \zero)$ to be an {\glp} for the homogeneous quadratic game, it is necessary and sufficient that:
\be
\label{eq:lrp_neg}&&\p{\Lv}(\Av - \Cv \Bv_n^{\dag} \Cv^\top)\p{\Lv} \cge \zero, \, ~~ \Lv = \Cv \p{\Bv_{n}},\\
\label{eq:lrp_pos}&&\p{\Mv}(\Bv - \Cv^\top \Av_p^{\dag}\Cv)\p{\Mv}\cle \zero, \, ~~\Mv = \Cv^\top \p{\Av_p}.
\en
\end{restatable}

\begin{proof}
Given the spectral decomposition $\Bv = \sum_i b_i \vv_i \vv_i^\top$ and $\yv = \sum_i y_i \vv_i$, the quadratic function can be written as:
\be
q(\xv, \yv) = \xv^\top \Av \xv/2 + \sum_i b_i y_i^2/2 + \sum_i y_i \xv^\top \Cv \vv_i.
\en
Maximizing over the eigenspace neighborhood of $\Nc(\yv^\star, \e)$ we obtain:
\be
\oq_\e(\xv) = \xv^\top (\Av - \Cv \Bv_n^{\dag} \Cv^\top)\xv/2 + \sum_{i\in \Ic_+} (b_i \e^2/2 + \e|\xv^\top \Cv \vv_i|), \, \Ic_+ :=\{i\in [m]: b_i \geq 0\}.
\en
In order for $\oq_\e(\xv) \geq \oq_\e(\xv^\star)$, it is necessary that for all $\xv$ such that $\vv_i^\top \Cv^\top \xv = 0$ for $i \in \Ic_+$, $\xv^\top (\Av - \Cv \Bv_n^{\dag} \Cv^\top)\xv/2\geq 0$. That is, for all $\Lv^\top \xv = \zero$ with $\Lv := \Cv \p{\Bv_n}$, $\xv^\top (\Av - \Cv \Bv_n^{\dag} \Cv^\top)\xv/2\geq 0$, which yields \eqref{eq:lrp_neg}. Symmetrically we obtain \eqref{eq:lrp_pos} for maximizing $\ubar{q}_\ve(\yv)$. The sufficient part is analogous to the proof of \Cref{thm:local_global_q}. Denote $\etav$ as an $|\Ic_+|$-dimensional vector with $\eta_i = \vv_i^\top \Cv^\top \xv$ and $i\in \Ic_+$, then
\be
\sum_{i\in \Ic_+} |\xv^\top \Cv \vv_i| = \|\etav\|_1 \geq \|\etav\|_2 = \|\sum_{i\in \Ic_+} (\vv_i^\top \Cv^\top \xv)\vv_i \|_2 = \|\Lv^\top \xv\|_2.
\en
The rest follows after \eqref{eq:trick}.
\end{proof}

In the special case of local minimax when $\Bv\cle \zero$, \eqref{eq:lrp_neg} and \eqref{eq:lrp_pos} reduces to \eqref{eq:qc}. 

\subsection{Stability at local robust points}

Finally, we discuss the convergence of first-order algorithms near {\glp}s. In \Cref{prop:lrp}, we gave full characterization for {\glp}s in one-dimensional quadratic games. In fact, from our spectral analysis in \Cref{sec:local_stab} one can draw the following conclusion:
\begin{restatable}[\textbf{\blue{local stability at LRP}}]{prop}{proplrp}\label{prop:lrp_alg}
Suppose $c^2 \neq ab$. For one-dimensional homogeneous quadratic games $q(x, y) = ax^2/2 + c x y + b y^2/2$, the stable sets of GDA (with momentum) and EG/OGD are within the set of {\glp}s. Moreover:
\begin{itemize}[topsep=2pt, itemsep=0pt]
    \item There exists a quadratic game and an {\glp}, $\zv^\star$, such that no hyper-parameter choice can allow 2TS-EG to converge to $\zv^\star$.
    \item Whenever a LRP exists, there always exists a hyper-parameter choice $(\a_1, \a_2, k)$ such that 2TS-OGD converges to the LRP.  
\end{itemize}
\end{restatable}

\begin{proof} \textbf{Part I} From stationarity the set of LRPs is $\{(0, 0)\}$ if $c^2 > ab$ and empty if $c^2 < ab$. The stable sets of gradient algorithms can only be empty or $\{(0, 0)\}$. We note that for $q(x, y) = ax^2/2 + c x y + b y^2/2$, the characteristic polynomial of $\haa$ is:
\be\label{eq:quadr_solution}
\l^2 + (\a_1 a - \a_2 b) \l + \a_1 \a_2 (c^2 - ab) = 0.
\en
It is necessary that $c^2 - a b \geq 0$ since from our spectral characterization, the two roots are either 1) both complex and are conjugate to each other; 2) both real and negative. If $c = 0$, we must have $a \geq 0 \geq b$ since the two roots are both real and must be non-positive. Comparing with \Cref{prop:lrp} we have the first conclusion.

\paragraph{Part II} Let us show the claim for EG. Take $q(x, y) =  -x^2 + x y + y^2/2$. From \eqref{eq:quadr_solution} and \Cref{thm:egratio}, it suffices to show that:
\be
p(\l) := \l^2 - (2\a_1 + \a_2) \l + 3 \a_1 \a_2 = 0
\en
has no solution in the region $\{\l \in \mathbb{C}: \Re(\l + \l^2) < 0\}$. If $(2\a_1 + \a_2)^2 \geq 12 \a_1 \a_2$, it suffices to show that $p(\l)$ has no root between $-1$ and $0$. Otherwise, the condition $\Re(\l + \l^2) < 0$ becomes 
$$
2\a_1 + \a_2 + (2\a_1 + \a_2)^2 < 6\a_1 \a_2,
$$
which cannot be true since $(2\a_1 + \a_2)^2 \geq 8\a_1 \a_2$ and $\a_1 > 0$, $\a_2 > 0$. 

\paragraph{Part III} For the claim of OGD, if $c = 0$ then $a > 0 > b$ and it is easy. If $c \neq 0$, combining \eqref{eq:quadr_solution} and \eqref{eq:ogd_large}, it suffices to show the existence of $(\a_1, \a_2) \in \R_{++}$ such that
\be
(\a_1 a - \a_2 b)^2 < 4 \a_1 \a_2 (c^2 - ab) < 4, \, \a_1 a - \a_2 b > -2\a_1 \a_2 (c^2 - ab),
\en
which, with $\g = \a_2/\a_1$, reduces to the existence of $(\a_2, \g) \in \R_{++}$ such that
\be
\frac{\g b - a}{2(c^2 - a b)} < \a_2,\, \a_2^2 < \frac{\g}{c^2 - ab}, \, (a - \g b)^2 < 4\g(c^2 - a b),
\en
which reduces to the existence of $\g \in \R_{++}$ such that
\be
(a - \g b)^2 < 4\g(c^2 - ab).
\en
this is always true no matter whether $b = 0$ or $b\neq 0$.
\end{proof}

This proposition shows the essential difference between EG and OGD in the convergence to {\glp}s. The last claim shares the same spirit with \citet[Theorem 28]{jin2019minmax}, since we can similarly write:
\be
\mathcal{LRP} = 2\mathcal{TS}\textrm{-}\mathcal{OGD},
\en
where $\mathcal{LRP}$ is the set of LRPs and $2\mathcal{TS}\textrm{-}\mathcal{OGD}$ is the set of all possible stable points of 2TS-OGD given some parameters $(\a_1 > 0, \a_2 > 0, k > 1)$. 

\blue{
However, this result does not hold in higher dimensions. We can prove the following:
\begin{prop}[\textbf{failure of gradient algorithms at LRP}]\label{prop:failure_lrp}
There exists a two-dimensional quadratic function $q(\xv, \yv)$ with its LRP at $(\zero, \zero)$, in the same setting as \Cref{thm:lrp_quadratic}, such that GD (with momentum), EG or OGD cannot converge to the LRP for any hyper-parameter choice. 
\end{prop}
\begin{proof}
Combined with what we have in \Cref{prop:lrp_alg} and \Cref{thm:all}, it suffices to prove the negative result for OGD. Since local robust points include both local minimax points and local maximin points, we construct a two-dimensional quadratic function  that include both cases:
\be\label{eq:lrp_failure_eg}
q(\xv, \yv) = -x_1^2 + x_1 y_1 + x_2 y_2 + y_2^2.
\en
Note that $(\zero, \zero)$ is the only stationary point. We now prove that it is also a local robust point. Writing the quadratic function in the same form as \eqref{eq:quadr}, we have:
\be
\Av = \begin{bmatrix}
-2 & 0 \\
0 & 0 
\end{bmatrix}, \, \Bv = \begin{bmatrix}
0 & 0 \\
0 & 2 
\end{bmatrix}, \, \Cv = \begin{bmatrix}
1 & 0 \\
0 & 1 
\end{bmatrix}. 
\en
From \Cref{defn:pos_neg_part_matrix}, we obtain the positive and the negative parts of $\Av$ and $\Bv$:
\be
\Av_p = \zero, \, \Av_n = \Av, \, \Bv_p = \Bv, \, \Bv_n = \zero,
\en
and thus $\p{\Bv_n} = \p{\Av_p} = \Iv$. In \eqref{eq:lrp_neg} and \eqref{eq:lrp_pos}, one can write $\Lv = \Mv = \Iv$ and $\p{\Lv} = \p{\Mv} = \zero$. It thus follows that \eqref{eq:lrp_neg} and \eqref{eq:lrp_pos} hold and $(\zero, \zero)$ is a LRP.  \\
We now analyze the local convergence of OGD. The Jacobian of $\vv(\zv)$ is a constant:
\be
\Hv_{\a_1, \a_2}(q) = \begin{bmatrix}
-\a_1 \Av & -\a_1 \Cv \\
\a_2 \Cv^\top & \a_2 \Bv
\end{bmatrix}.
\en
Note that $\Cv^\top$ and $\Bv$ are diagonal matrices and thus they commute. So, we can compute the characteristic equation of $\Hv_{\a_1, \a_2}(q)$ as:
\be
\det((\lambda \Iv + \a_1 \Av) (\lambda \Iv - \a_2 \Bv) + \a_1 \a_2 \Cv \Cv^\top) = \zero,
\en
from which we obtain:
\be
&&\label{eq:first_char_poly}\lambda(\lambda -  2\a_1) + \a_1 \a_2 = 0, \\
&&\label{eq:second_char_poly}\lambda(\lambda -  2\a_2) + \a_1 \a_2 = 0.
\en
For 2TS-OGD, when $k\to 1_+$ the algorithm is the most stable (\Cref{thm:egratio}), where the condition should be (\Cref{thm:eg_stable}, \eqref{eq:ogd_large}):
\be\label{eq:ogd_large_new}
|\lambda| < 1, \, |\lambda - 1/2| > 1/2.
\en
Now we separate the discussion into two cases: if $\a_1 \geq \a_2 > 0$, then \eqref{eq:first_char_poly} gives:
\be
\lambda_{1,2} = \a_1 \pm \sqrt{\a_1^2 - \a_1 \a_2},
\en
and there exists a real and positive root. Similarly, if $\a_2 \geq \a_1 > 0$, \eqref{eq:second_char_poly} has a real and positive root. In either case \eqref{eq:ogd_large_new} would be violated. 
\end{proof}
From the proof, we can see that the problem lies in the coordinate-independent step sizes. In fact, \eqref{eq:lrp_failure_eg} could be rewritten as:
\be
q(\xv, \yv) = q_1(x_1, y_1) + q_2(x_2, y_2), \, q_1(x, y) := -x^2 + x y, \, q_2(x, y) := x y + y^2.
\en
For the function $q_1$, $(0, 0)$ is a local minimax point, and the stability constraint for 2TS-OGD is (with $k\to 1_+$, see \eqref{eq:simple_condition_ogd}):
\be\label{eq:simple_lrp_one}
\a_1 < 1, \, 1 < \a_2 < 1/\a_1.
\en
While for the function $q_2$, $(0, 0)$ is a local maximin point, and the stability constraint for 2TS-OGD is (in a similar way):
\be\label{eq:simple_lrp_two}
\a_2 < 1, \, 1 < \a_1 < 1/\a_2.
\en
\eqref{eq:simple_lrp_one} and \eqref{eq:simple_lrp_two} are conflicting each other. Therefore, it tells us that coordinate-dependent step sizes might be necessary in order for stability near a LRP, such as those in Adam \citep{kingma2014adam}, which is widely used in GAN training. \\
We finally mention that LRPs are a wider class that could include the stable points of gradient algorithms. For example, in the proof of Prop.~27 of \citet{jin2019minmax}, there is a two-dimensional quadratic function that has $(0, 0)$ as a stable solution of simultaneous GDA, but it is neither local maximin or minimax. It can be shown that it is in fact a local robust point.
}


\bibstyle{abbrvnat}


\begin{thebibliography}{73}
\providecommand{\natexlab}[1]{#1}
\providecommand{\url}[1]{\texttt{#1}}
\expandafter\ifx\csname urlstyle\endcsname\relax
  \providecommand{\doi}[1]{doi: #1}\else
  \providecommand{\doi}{doi: \begingroup \urlstyle{rm}\Url}\fi

\bibitem[Arrow et~al.(1958)Arrow, Hurwicz, and Uzawa]{arrow1958studies}
K.~Arrow, L.~Hurwicz, and H.~Uzawa.
\newblock \emph{Studies in linear and non-linear programming}.
\newblock Stanford University Press, 1958.

\bibitem[Azizian et~al.(2020{\natexlab{a}})Azizian, Mitliagkas, Lacoste-Julien,
  and Gidel]{Azizian2019ATA}
W.~Azizian, I.~Mitliagkas, S.~Lacoste-Julien, and G.~Gidel.
\newblock A tight and unified analysis of extragradient for a whole spectrum of
  differentiable games.
\newblock In \emph{the 23rd International Conference on Artificial Intelligence
  and Statistics}, 2020{\natexlab{a}}.

\bibitem[Azizian et~al.(2020{\natexlab{b}})Azizian, Scieur, Mitliagkas,
  Lacoste-Julien, and Gidel]{azizian2020accelerating}
W.~Azizian, D.~Scieur, I.~Mitliagkas, S.~Lacoste-Julien, and G.~Gidel.
\newblock Accelerating smooth games by manipulating spectral shapes.
\newblock In \emph{the 23rd International Conference on Artificial Intelligence
  and Statistics}, 2020{\natexlab{b}}.

\bibitem[Barazandeh and Razaviyayn(2020)]{baraz2020solving}
B.~Barazandeh and M.~Razaviyayn.
\newblock Solving non-convex non-differentiable min-max games using proximal
  gradient method.
\newblock In \emph{ICASSP 2020-2020 IEEE International Conference on Acoustics,
  Speech and Signal Processing (ICASSP)}, pages 3162--3166. IEEE, 2020.

\bibitem[Basu et~al.(2005)Basu, Pollack, and Roy]{basu2005algorithms}
S.~Basu, R.~Pollack, and M.-F. Roy.
\newblock \emph{Algorithms in real algebraic geometry}.
\newblock Springer, 2005.

\bibitem[Ben-Tal and Zowe(1982)]{ben1982necessary}
A.~Ben-Tal and J.~Zowe.
\newblock Necessary and sufficient optimality conditions for a class of
  nonsmooth minimization problems.
\newblock \emph{Mathematical Programming}, 24\penalty0 (1):\penalty0 70--91,
  1982.

\bibitem[Ben-Tal and Zowe(1985)]{BenTalZowe85}
A.~Ben-Tal and J.~Zowe.
\newblock \href{https://doi.org/10.1007/BF00942193}{Directional derivatives in
  nonsmooth optimization}.
\newblock \emph{Journal of Optimization Theory and Applications}, 47\penalty0
  (4):\penalty0 483--490, 1985.

\bibitem[Berard et~al.(2020)Berard, Gidel, Almahairi, Vincent, and
  Lacoste-Julien]{Berard2020A}
H.~Berard, G.~Gidel, A.~Almahairi, P.~Vincent, and S.~Lacoste-Julien.
\newblock A closer look at the optimization landscapes of generative
  adversarial networks.
\newblock In \emph{International Conference on Learning Representations}, 2020.
\newblock URL \url{https://openreview.net/forum?id=HJeVnCEKwH}.

\bibitem[Bertsekas(1997)]{bertsekas1997nonlinear}
D.~P. Bertsekas.
\newblock Nonlinear programming.
\newblock \emph{Journal of the Operational Research Society}, 48\penalty0
  (3):\penalty0 334--334, 1997.

\bibitem[Bollapragada et~al.(2019)Bollapragada, Scieur, and
  d’Aspremont]{bollapragada2019nonlinear}
R.~Bollapragada, D.~Scieur, and A.~d’Aspremont.
\newblock Nonlinear acceleration of primal-dual algorithms.
\newblock In \emph{the 22nd International Conference on Artificial Intelligence
  and Statistics}, pages 739--747, 2019.

\bibitem[Clarke(1990)]{Clarke90}
F.~H. Clarke.
\newblock \emph{Optimization and Nonsmooth Analysis}.
\newblock SIAM, 1990.

\bibitem[Cominetti and Correa(1990)]{CominettiCorrea90}
R.~Cominetti and R.~Correa.
\newblock \href{https://doi.org/10.1137/0328045}{A Generalized Second-Order
  Derivative in Nonsmooth Optimization}.
\newblock \emph{{SIAM} Journal on Control and Optimization}, 28\penalty0
  (4):\penalty0 789--809, 1990.

\bibitem[Danskin(1966)]{Danskin66}
J.~M. Danskin.
\newblock \href{https://doi.org/10.1137/0114053}{The Theory of Max-Min, with
  Applications}.
\newblock \emph{{SIAM} Journal on Applied Mathematics}, 14\penalty0
  (4):\penalty0 641--664, 1966.

\bibitem[Daskalakis and Panageas(2018)]{daskalakis2018limit}
C.~Daskalakis and I.~Panageas.
\newblock The limit points of (optimistic) gradient descent in min-max
  optimization.
\newblock In \emph{Advances in Neural Information Processing Systems}, pages
  9236--9246, 2018.

\bibitem[Daskalakis et~al.(2018)Daskalakis, Ilyas, Syrgkanis, and
  Zeng]{daskalakis2018training}
C.~Daskalakis, A.~Ilyas, V.~Syrgkanis, and H.~Zeng.
\newblock \href{https://openreview.net/forum?id=SJJySbbAZ}{Training GANs with
  optimism}.
\newblock In \emph{the 6th International Conference on Learning
  Representations}, 2018.

\bibitem[Dem'yanov(1966)]{Demyanov66}
V.~F. Dem'yanov.
\newblock \href{https://link.springer.com/article/10.1007/BF01074499}{On the
  solution of several minimax problems. I}.
\newblock \emph{Cybernetics}, 2:\penalty0 47--53, 1966.

\bibitem[Dem'yanov(1970)]{Demyanov70}
V.~F. Dem'yanov.
\newblock \href{https://doi.org/10.1016/0041-5553(70)90037-6}{Sufficient
  conditions for a local minimax}.
\newblock \emph{{USSR} Computational Mathematics and Mathematical Physics},
  10\penalty0 (5):\penalty0 53--63, 1970.

\bibitem[Dem'yanov(1973)]{Demyanov73}
V.~F. Dem'yanov.
\newblock \href{https://doi.org/10.1016/0041-5553(72)90053-5}{Second-order
  directional derivatives of a function of the maximum}.
\newblock \emph{Cybernetics}, 9:\penalty0 797--–800, 1973.

\bibitem[Dem’yanov and Malozemov(1974)]{DemyanovMalozemov72}
V.~F. Dem’yanov and V.~N. Malozemov.
\newblock \emph{Introduction to Minimax}.
\newblock Wiley, 1974.

\bibitem[Facchinei and Pang(2007)]{facchinei2007finite}
F.~Facchinei and J.-S. Pang.
\newblock \emph{Finite-dimensional variational inequalities and complementarity
  problems}.
\newblock Springer Science \& Business Media, 2007.

\bibitem[Fan(1950)]{fan1950theorem}
K.~Fan.
\newblock On a theorem of weyl concerning eigenvalues of linear
  transformations: {II}.
\newblock \emph{Proceedings of the National Academy of Sciences of the United
  States of America}, 36\penalty0 (1):\penalty0 31, 1950.

\bibitem[Farnia and Ozdaglar(2020)]{farnia2020gans}
F.~Farnia and A.~Ozdaglar.
\newblock Do {GANs} always have {Nash} equilibria?
\newblock In \emph{International Conference on Machine Learning}, pages
  3029--3039. PMLR, 2020.

\bibitem[Fiez et~al.(2019)Fiez, Chasnov, and Ratliff]{fiez2019convergence}
T.~Fiez, B.~Chasnov, and L.~J. Ratliff.
\newblock \href{https://arxiv.org/abs/1906.01217}{Convergence of learning
  dynamics in {S}tackelberg games}.
\newblock \emph{arXiv}, 2019.
\newblock arXiv:1906.01217.

\bibitem[Gidel et~al.(2019)Gidel, Hemmat, Pezeshki, Huang, Lepriol,
  Lacoste-Julien, and Mitliagkas]{GidelHPHLLM19}
G.~Gidel, R.~A. Hemmat, M.~Pezeshki, G.~Huang, R.~Lepriol, S.~Lacoste-Julien,
  and I.~Mitliagkas.
\newblock {Negative momentum for improved game dynamics}.
\newblock In \emph{the 22nd International Conference on Artificial Intelligence
  and Statistics}, 2019.

\bibitem[Golshtein(1972)]{Golshtein72}
E.~G. Golshtein.
\newblock A generalized gradient method for finding saddlepoints.
\newblock \emph{Ekonomika i matematicheskie}, 8\penalty0 (4):\penalty0 36--52,
  1972.

\bibitem[Goodfellow et~al.(2014)Goodfellow, Pouget-Abadie, Mirza, Xu,
  Warde-Farley, Ozair, Courville, and Bengio]{goodfellow2014generative}
I.~Goodfellow, J.~Pouget-Abadie, M.~Mirza, B.~Xu, D.~Warde-Farley, S.~Ozair,
  A.~Courville, and Y.~Bengio.
\newblock Generative adversarial nets.
\newblock In \emph{Advances in neural information processing systems}, pages
  2672--2680, 2014.

\bibitem[Hampel(1974)]{hampel1974influence}
F.~R. Hampel.
\newblock The influence curve and its role in robust estimation.
\newblock \emph{Journal of the american statistical association}, 69\penalty0
  (346):\penalty0 383--393, 1974.

\bibitem[Heusel et~al.(2017)Heusel, Ramsauer, Unterthiner, Nessler, and
  Hochreiter]{heusel2017gans}
M.~Heusel, H.~Ramsauer, T.~Unterthiner, B.~Nessler, and S.~Hochreiter.
\newblock {GAN}s trained by a two time-scale update rule converge to a local
  {N}ash equilibrium.
\newblock In \emph{Advances in neural information processing systems}, pages
  6626--6637, 2017.

\bibitem[Hiriart-Urruty and Lemar{\'e}chal(2004)]{hiriart2004fundamentals}
J.-B. Hiriart-Urruty and C.~Lemar{\'e}chal.
\newblock \emph{Fundamentals of convex analysis}.
\newblock Springer Science \& Business Media, 2004.

\bibitem[Hiriart-Urruty and Lemar{\'e}chal(2013)]{hiriart2013convex}
J.-B. Hiriart-Urruty and C.~Lemar{\'e}chal.
\newblock \emph{Convex analysis and minimization algorithms I: Fundamentals},
  volume 305.
\newblock Springer, 2013.

\bibitem[Hsieh et~al.(2019)Hsieh, Iutzeler, Malick, and
  Mertikopoulos]{hsieh2019convergence}
Y.-G. Hsieh, F.~Iutzeler, J.~Malick, and P.~Mertikopoulos.
\newblock On the convergence of single-call stochastic extra-gradient methods.
\newblock In \emph{NeurIPS}, pages 6936--6946, 2019.

\bibitem[Hsieh et~al.(2020)Hsieh, Iutzeler, Malick, and
  Mertikopoulos]{hsieh2020explore}
Y.-G. Hsieh, F.~Iutzeler, J.~Malick, and P.~Mertikopoulos.
\newblock Explore aggressively, update conservatively: Stochastic extragradient
  methods with variable stepsize scaling.
\newblock In \emph{NeurIPS 2020-34th Conference on Neural Information
  Processing Systems}, 2020.

\bibitem[Ibrahim et~al.(2020)Ibrahim, Azizian, Gidel, and
  Mitliagkas]{IbrahimAGM19}
A.~Ibrahim, W.~Azizian, G.~Gidel, and I.~Mitliagkas.
\newblock Linear lower bounds and conditioning of differentiable games.
\newblock In \emph{International conference on machine learning}, pages
  6356--6366, 2020.

\bibitem[Jin et~al.(2020)Jin, Netrapalli, and Jordan]{jin2019minmax}
C.~Jin, P.~Netrapalli, and M.~Jordan.
\newblock What is local optimality in nonconvex-nonconcave minimax
  optimization?
\newblock In \emph{International conference on machine learning}, pages
  5735--5744, 2020.

\bibitem[Katok and Hasselblatt(1995)]{katok1995introduction}
A.~Katok and B.~Hasselblatt.
\newblock \emph{Introduction to the modern theory of dynamical systems},
  volume~54.
\newblock Cambridge university press, 1995.

\bibitem[Kawasaki(1988)]{kawasaki1988upper}
H.~Kawasaki.
\newblock The upper and lower second order directional derivatives of a
  sup-type function.
\newblock \emph{Mathematical Programming}, 41\penalty0 (1-3):\penalty0
  327--339, 1988.

\bibitem[Kawasaki(1991)]{kawasaki1990second}
H.~Kawasaki.
\newblock Second order necessary optimality conditions for minimizing a
  sup-type function.
\newblock \emph{Mathematical programming}, 49\penalty0 (1-3):\penalty0
  213--229, 1991.

\bibitem[Kawasaki(1992)]{kawasaki1992second}
H.~Kawasaki.
\newblock Second-order necessary and sufficient optimality conditions for
  minimizing a sup-type function.
\newblock \emph{Applied Mathematics and Optimization}, 26\penalty0
  (2):\penalty0 195--220, 1992.

\bibitem[Kingma and Ba(2015)]{kingma2014adam}
D.~P. Kingma and J.~Ba.
\newblock Adam: A method for stochastic optimization.
\newblock In \emph{International Conference on Learning Representations}, 2015.

\bibitem[Korpelevich(1976)]{korpelevich1976extragradient}
G.~Korpelevich.
\newblock The extragradient method for finding saddle points and other
  problems.
\newblock \emph{Matecon}, 12:\penalty0 747--756, 1976.

\bibitem[Lin et~al.(2020)Lin, Jin, and Jordan]{lin2020nearoptimal}
T.~Lin, C.~Jin, and M.~I. Jordan.
\newblock Near-optimal algorithms for minimax optimization.
\newblock In \emph{the 33rd Conference on Learning Theory}, 2020.

\bibitem[Liu et~al.(2020)Liu, Lu, Chen, Feng, Xu, Al-Dujaili, Hong, and
  O'Reilly]{liu2019min}
S.~Liu, S.~Lu, X.~Chen, Y.~Feng, K.~Xu, A.~Al-Dujaili, M.~Hong, and U.-M.
  O'Reilly.
\newblock Min-max optimization without gradients: Convergence and applications
  to black-box evasion and poisoning attacks.
\newblock In \emph{International conference on machine learning}, pages
  2307--2318, 2020.

\bibitem[Madry et~al.(2018)Madry, Makelov, Schmidt, Tsipras, and
  Vladu]{madry2017towards}
A.~Madry, A.~Makelov, L.~Schmidt, D.~Tsipras, and A.~Vladu.
\newblock Towards deep learning models resistant to adversarial attacks.
\newblock In \emph{the 6th International Conference on Learning
  Representations}, 2018.

\bibitem[Marshall et~al.(1979)Marshall, Olkin, and
  Arnold]{marshall1979inequalities}
A.~W. Marshall, I.~Olkin, and B.~C. Arnold.
\newblock \emph{Inequalities: theory of majorization and its applications},
  volume 143.
\newblock Springer, 1979.

\bibitem[Mertikopoulos et~al.(2018)Mertikopoulos, Papadimitriou, and
  Piliouras]{MertikopoulosPP18}
P.~Mertikopoulos, C.~Papadimitriou, and G.~Piliouras.
\newblock
  \href{https://epubs.siam.org/doi/pdf/10.1137/1.9781611975031.172}{Cycles in
  adversarial regularized learning}.
\newblock In \emph{Proceedings of the Twenty-Ninth Annual ACM-SIAM Symposium on
  Discrete Algorithms}, pages 2703--2717, 2018.

\bibitem[Mertikopoulos et~al.(2019)Mertikopoulos, Lecouat, Zenati, Foo,
  Chandrasekhar, and Piliouras]{mertikopoulos2019optimistic}
P.~Mertikopoulos, B.~Lecouat, H.~Zenati, C.-S. Foo, V.~Chandrasekhar, and
  G.~Piliouras.
\newblock Optimistic mirror descent in saddle-point problems: Going the extra
  (gradient) mile.
\newblock In \emph{the 7th International Conference on Learning
  Representations}, 2019.

\bibitem[Mescheder et~al.(2017)Mescheder, Nowozin, and
  Geiger]{mescheder2017numerics}
L.~Mescheder, S.~Nowozin, and A.~Geiger.
\newblock The numerics of {GAN}s.
\newblock In \emph{Advances in Neural Information Processing Systems}, pages
  1825--1835, 2017.

\bibitem[Mokhtari et~al.(2019)Mokhtari, Ozdaglar, and
  Pattathil]{mokhtari2019proximal}
A.~Mokhtari, A.~Ozdaglar, and S.~Pattathil.
\newblock Proximal point approximations achieving a convergence rate of
  $o(1/k)$ for smooth convex-concave saddle point problems: Optimistic gradient
  and extra-gradient methods.
\newblock \emph{arXiv:1906.01115}, 2019.

\bibitem[Murty and Kabadi(1987)]{murty1987some}
K.~G. Murty and S.~N. Kabadi.
\newblock Some np-complete problems in quadratic and nonlinear programming.
\newblock \emph{Mathematical programming}, 39\penalty0 (2):\penalty0 117--129,
  1987.

\bibitem[Nash(1950)]{nash1950equilibrium}
J.~F. Nash.
\newblock Equilibrium points in $n$-person games.
\newblock \emph{Proceedings of the national academy of sciences}, 36\penalty0
  (1):\penalty0 48--49, 1950.

\bibitem[Nemirovsky and Yudin(1983)]{nemirovsky1983problem}
A.~S. Nemirovsky and D.~B. Yudin.
\newblock \emph{Problem complexity and method efficiency in optimization.}
\newblock Wiley, 1983.

\bibitem[Nesterov(1983)]{nesterov1983method}
Y.~Nesterov.
\newblock A method for unconstrained convex minimization problem with the rate
  of convergence $o(1/k^{2})$.
\newblock \emph{Doklady AN USSR}, 269:\penalty0 543--547, 1983.

\bibitem[Niethammer and Varga(1983)]{niethammer1983analysis}
W.~Niethammer and R.~S. Varga.
\newblock The analysis of $k$-step iterative methods for linear systems from
  summability theory.
\newblock \emph{Numerische Mathematik}, 41\penalty0 (2):\penalty0 177--206,
  1983.

\bibitem[Peng et~al.(2020)Peng, Dai, Zhang, and Cheng]{peng2019training}
W.~Peng, Y.-H. Dai, H.~Zhang, and L.~Cheng.
\newblock Training {GANs} with centripetal acceleration.
\newblock \emph{Optimization Methods and Software}, 35\penalty0 (5):\penalty0
  955--973, 2020.

\bibitem[Polyak(1987)]{polyak87}
B.~Polyak.
\newblock \emph{Introduction to Optimization}.
\newblock Optimization Software Inc., 1987.

\bibitem[Polyak(1964)]{Polyak64}
B.~T. Polyak.
\newblock \href{https://doi.org/10.1016/0041-5553(64)90137-5}{Some methods of
  speeding up the convergence of iteration methods}.
\newblock \emph{USSR Computational Mathematics and Mathematical Physics},
  4\penalty0 (5):\penalty0 1--17, 1964.

\bibitem[Popov(1980)]{popov1980modification}
L.~D. Popov.
\newblock A modification of the {A}rrow--{H}urwicz method for search of saddle
  points.
\newblock \emph{Mathematical Notes}, 28\penalty0 (5):\penalty0 845--848, 1980.

\bibitem[Razaviyayn et~al.(2020)Razaviyayn, Huang, Lu, Nouiehed, Sanjabi, and
  Hong]{razaviyayn2020non}
M.~Razaviyayn, T.~Huang, S.~Lu, M.~Nouiehed, M.~Sanjabi, and M.~Hong.
\newblock Nonconvex min-max optimization: Applications, challenges, and recent
  theoretical advances.
\newblock \emph{IEEE Signal Processing Magazine}, 37\penalty0 (5):\penalty0
  55--66, 2020.

\bibitem[Rockafellar and Wets(2009)]{rockafellar2009variational}
R.~T. Rockafellar and R.~J.-B. Wets.
\newblock \emph{Variational analysis}, volume 317.
\newblock Springer Science \& Business Media, 2009.

\bibitem[Schaefer et~al.(2020)Schaefer, Zheng, and
  Anandkumar]{schaefer2020implicit}
F.~Schaefer, H.~Zheng, and A.~Anandkumar.
\newblock Implicit competitive regularization in {GANs}.
\newblock In \emph{International Conference on Machine Learning}, pages
  8533--8544. PMLR, 2020.

\bibitem[Schur(1917)]{schur1917potenzreihen}
I.~Schur.
\newblock {\"U}ber potenzreihen, die im innern des einheitskreises
  beschr{\"a}nkt sind.
\newblock \emph{Journal f{\"u}r die reine und angewandte Mathematik},
  147:\penalty0 205--232, 1917.

\bibitem[Seeger(1988)]{Seeger88}
A.~Seeger.
\newblock Second order directional derivatives in parametric optimization
  problems.
\newblock \emph{Mathematics of Operations Research}, 13\penalty0 (1):\penalty0
  124--139, 1988.

\bibitem[Sinha et~al.(2018)Sinha, Namkoong, and Duchi]{sinha2018certifying}
A.~Sinha, H.~Namkoong, and J.~Duchi.
\newblock Certifying some distributional robustness with principled adversarial
  training.
\newblock In \emph{International Conference on Learning Representations}, 2018.

\bibitem[Sion et~al.(1958)]{sion1958general}
M.~Sion et~al.
\newblock On general minimax theorems.
\newblock \emph{Pacific Journal of mathematics}, 8\penalty0 (1):\penalty0
  171--176, 1958.

\bibitem[Sutskever et~al.(2013)Sutskever, Martens, Dahl, and
  Hinton]{sutskever2013importance}
I.~Sutskever, J.~Martens, G.~Dahl, and G.~Hinton.
\newblock On the importance of initialization and momentum in deep learning.
\newblock In \emph{International conference on machine learning}, pages
  1139--1147, 2013.

\bibitem[Sutton et~al.(1998)Sutton, Barto, et~al.]{sutton1998introduction}
R.~S. Sutton, A.~G. Barto, et~al.
\newblock \emph{Introduction to reinforcement learning}, volume 135.
\newblock MIT press Cambridge, 1998.

\bibitem[von Neumann(1928)]{neumann1928theorie}
J.~von Neumann.
\newblock Zur theorie der gesellschaftsspiele.
\newblock \emph{Mathematische annalen}, 100\penalty0 (1):\penalty0 295--320,
  1928.

\bibitem[von Stackelberg(1934)]{von2010market}
H.~von Stackelberg.
\newblock \emph{Market structure and equilibrium}.
\newblock Springer, 1934.

\bibitem[Wang et~al.(2020)Wang, Zhang, and Ba]{wang2019solving}
Y.~Wang, G.~Zhang, and J.~Ba.
\newblock On solving minimax optimization locally: A follow-the-ridge approach.
\newblock In \emph{the 8th International Conference on Learning
  Representations}, 2020.

\bibitem[Zhang and Yu(2020)]{zhang2019convergence}
G.~Zhang and Y.~Yu.
\newblock Convergence of gradient methods on bilinear zero-sum games.
\newblock In \emph{the 8th International Conference on Learning
  Representations}, 2020.

\bibitem[Zhang et~al.(2020)Zhang, Poupart, and Yu]{zhang2020optimality}
G.~Zhang, P.~Poupart, and Y.~Yu.
\newblock Optimality and stability in non-convex smooth games.
\newblock arXiv:2002.11875, 2020.

\bibitem[Zhang et~al.(2021)Zhang, Wu, Poupart, and Yu]{zhang2020newton}
G.~Zhang, K.~Wu, P.~Poupart, and Y.~Yu.
\newblock Newton-type methods for minimax optimization.
\newblock In \emph{{ICML} workshop on Beyond First-Order Methods in ML
  Systems}, 2021.
\newblock arXiv:2006.14592.

\bibitem[Zhang et~al.(2019)Zhang, Hong, and Zhang]{zhang2019lower}
J.~Zhang, M.~Hong, and S.~Zhang.
\newblock \href{https://arxiv.org/abs/1912.07481}{On Lower Iteration Complexity
  Bounds for the Saddle Point Problems}.
\newblock arXiv:1912.07481, 2019.

\end{thebibliography}
\end{document}